%% file: iclr2027_conference.tex
\documentclass{article} 
\usepackage{iclr2027_conference,times}

\input{math_commands.tex}

\usepackage{hyperref}
\usepackage{url}
\usepackage{xcolor}

\usepackage{algorithm}
\usepackage{algorithmic}
\usepackage{graphicx}
\usepackage{amsmath}
\usepackage{amssymb}
\usepackage{amsthm}
\usepackage{import}
\usepackage{booktabs}
\usepackage{natbib}
\usepackage{wrapfig}
\usepackage{cleveref}
\usepackage{multirow}

\newtheorem{theorem}{Theorem}[section]
\newtheorem{proposition}[theorem]{Proposition}
\newtheorem{lemma}[theorem]{Lemma}
\newtheorem{corollary}[theorem]{Corollary}
\newtheorem{definition}[theorem]{Definition}

\newtheorem{remark}[theorem]{Remark}
\usepackage{subcaption}

\title{Lifted Bellman Linear Programming for Offline Reinforcement Learning}

\author{
Hyukjun Yang\textsuperscript{1} \quad
Jongchan Park\textsuperscript{2}\thanks{Equal contribution.} \quad
Narim Jeong\textsuperscript{1}\footnotemark[1] \quad
Donghwan Lee\textsuperscript{1}\thanks{Corresponding author.} \\
\textsuperscript{1}Korea Advanced Institute of Science and Technology (KAIST) \quad
\textsuperscript{2}NAVER LABS \\
\texttt{\{jundol32, donghwan\}@kaist.ac.kr}
}

\iclrfinalcopy 
\begin{document}

\maketitle
\lhead{Preprint. Under review.}

\import{contents}{00_abstr}

\import{contents}{01_intro}

\import{contents}{02_relat}

\import{contents}{03_preli}

\import{contents}{04_motiv}

\import{contents}{05_algor}

\import{contents}{06_exper}

\import{contents}{07_concl}

\bibliographystyle{iclr2027_conference}
\bibliography{iclr2027_conference}

\newpage
\import{contents}{XX_appen}

\end{document}

%% file: math_commands.tex
\usepackage{amsmath,amsfonts,bm}

\def\eqref#1{equation~\ref{#1}}

\def\1{\bm{1}}

\DeclareMathAlphabet{\mathsfit}{\encodingdefault}{\sfdefault}{m}{sl}
\SetMathAlphabet{\mathsfit}{bold}{\encodingdefault}{\sfdefault}{bx}{n}



%% file: contents/00_abstr.tex
\begin{abstract}
Offline reinforcement learning (RL) typically trains a critic by minimizing a regression loss against bootstrapped value targets, which are stabilized by target networks updated through exponential moving averages (EMA). Furthermore, multi-step targets incorporate behavior-policy actions and therefore require off-policy correction. We instead impose in-sample Bellman optimality on the critic through inequality constraints. We formulate the Lifted Bellman Linear Program (LBLP), which lifts the linear programming characterization of Bellman optimality to the joint $(Q,V)$ space so that every constraint involves only state-action pairs in the dataset. Its unique minimizer is the in-sample optimal pair, and constraints along $K$-step segments of dataset trajectories leave this minimizer unchanged for any rollout policy and any horizon. Under deterministic dynamics, this minimizer lies between the best return of the dataset and the optimal value. Relaxing the constraints into hinge penalties recovers the same solution above a finite penalty coefficient in the tabular case. Approximate Lifted Bellman Unconstrained Minimization (ALBUM) implements this relaxation with neural networks and detaches the $K$-step rollout targets by stop gradient. 
Its objective contains no squared regression onto bootstrapped targets, so it trains without target networks or EMA updates.
Under deterministic dynamics, the $K$-step constraints allow discounted returns along dataset trajectories to serve as lower bounds without off-policy correction or action chunking, providing a direct mechanism for long-horizon value propagation. On OGBench, ALBUM uses a single critic with a Gaussian policy, matches the average performance of FQL, and substantially outperforms ReBRAC, while using the fewest parameters and the least peak GPU memory among all measured methods.
\end{abstract}

%% file: contents/01_intro.tex
\section{Introduction}

Offline reinforcement learning (RL) learns a policy solely from a pre-collected dataset~\citep{levine2020offline}, and its central obstacle is value overestimation, since the Bellman maximization selects values extrapolated to actions outside the dataset and each backup propagates them to preceding states~\citep{mao2024doubly}. Prior work either keeps policy improvement close to the behavior policy~\citep{fujimoto2019off, kumar2019stabilizing, fujimoto2021minimalist}, or modifies the Bellman update by penalizing the values assigned to actions the dataset does not contain~\citep{kumar2020conservative}, or by replacing the maximization with an in-sample regression~\citep{kostrikov2021offline, garg2023extreme, xu2023offline}. This regression brings two costs. The critic is regressed onto a target computed from its own estimate, so a target network with EMA updates is needed to keep the target stable~\citep{mnih2015human, lillicrap2015continuous}. The same loss also limits how far a single backup reaches. When a multi-step return enters the
target, the intermediate actions come from the behavior policy, and the resulting target is biased for the learned policy unless it is corrected~\citep{li2026reinforcement}. This raises a more basic question.
\emph{Can in-sample Bellman optimality be imposed on the critic entirely through inequality constraints over state-action pairs in the dataset, so that the critic is the solution of a constrained program rather than a regression onto bootstrapped targets?}

We approach this question starting from the linear program (LP) characterization of Bellman optimality, in which the optimal value function is the smallest function satisfying a family of Bellman inequalities~\citep{de2003linear, puterman2014markov}. However, neither of the two standard LP formulations carries over to offline RL directly. The action-value formulation imposes a Bellman inequality for every action at the next state, requiring evaluation at actions that the dataset need not contain~\citep{lee2025analysis}. The state-value formulation avoids this issue by eliminating the next action, but identifies only the optimal state-value function, so a policy must be recovered from the dual occupancy measure~\citep{nachum2020reinforcement}. Offline work building on this formulation, therefore, turns to the dual, either by learning a density ratio~\citep{lee2021optidice} or by solving a saddle-point problem~\citep{zhan2022offline, kamoutsi2021efficient, ozdaglar2023revisiting, gabbianelli2024offline}. 

In this work, we resolve both difficulties with the Lifted Bellman Linear Program (LBLP), which stays on the primal side and lifts the program to the joint $(Q,V)$ space. An epigraph variable $V$ expresses $V(s) \ge \max_a Q(s,a)$ as the family of inequalities $Q(s,a) \le V(s)$, one for each action, so the Bellman inequality relates $Q$ to $V$ alone and involves no next action. LBLP places positive weight on both $Q$ and $V$, so its unique minimizer directly determines both $Q$ and $V$ without requiring a dual step. Unlike prior formulations that also use this lifting~\citep{jeong2025stochastic, mehta2020convex, bas2021logistic, lu2023convex}, in which either $V$ is not a decision variable or only one of $Q$ and $V$ carries weight in the objective (Section~\ref{sec:related works}), LBLP uniquely determines the pair through its objective.

In the offline setting, we instantiate LBLP on the empirical Markov decision process (MDP) induced by the dataset. Every constraint is indexed by a dataset state-action pair, including the $K$-step constraints along dataset trajectories. These additional constraints leave the feasible set and the unique minimizer unchanged. Under deterministic dynamics, this minimizer lies between the best dataset return and the optimal value, and both bounds hold without the coefficient of a value penalty \citep{kumar2020conservative}. For neural function approximation, we relax the constraints into hinge penalties \citep{ryu2019caql, sikchi2024dual}, and the relaxation is exact above a finite penalty coefficient in the tabular case. Approximate Lifted Bellman Unconstrained Minimization (ALBUM) implements this relaxation with neural networks and detaches the $K$-step targets by stop gradient. Since the ALBUM objective contains no squared regression onto bootstrapped targets, it requires no target networks or EMA updates, and because it imposes maximization only through constraints, it trains with a single critic. Under deterministic dynamics, its multi-step penalties provide lower bounds along trajectories without off-policy correction or action chunking, which gives a direct mechanism for long-horizon value propagation. Using a single $Q$, a single $V$, and a Gaussian actor, ALBUM matches the average performance of FQL~\citep{park2025flow} on OGBench~\citep{ogbench} and substantially outperforms ReBRAC~\citep{tarasov2023revisiting} while using half as many parameters.


Our contributions are as follows.
\begin{itemize}
\item We formulate the Lifted Bellman Linear Program (LBLP), a primal LP over the joint $(Q,V)$ space whose constraints involve only state-action pairs in the dataset. With positive objective weights on both $Q$ and $V$, its unique minimizer is the Bellman-optimal pair of the empirical MDP. Constraints along $K$-step rollouts are implied by the one-step constraints, so they can be added for any rollout policy and horizon without off-policy correction or action chunking.

\item Under deterministic dynamics, we show that this minimizer lies between the best dataset return and the optimal value, and both bounds hold without a penalty coefficient. In an idealized coordinate-wise iteration, the $K$-step term never slows and can accelerate value propagation (Theorem~\ref{thm:speedup}).

\item We relax LBLP through hinge penalties and show that the relaxation is exact above a finite penalty coefficient in the tabular case. ALBUM implements this relaxation with neural networks. Since its objective contains no squared regression onto bootstrapped targets and imposes maximization only through constraints, ALBUM trains without target networks or EMA updates and with a single critic (Table~\ref{tab:targetensemble}).

\item We analyze the update obtained by detaching the $K$-step targets with stop gradient. The detachment removes $\gamma$ and $K$ from the rollout term of the necessary condition for bounded updates (Proposition~\ref{prop:bounded}), and under additional conditions on the dataset, the LBLP solution is a stationary point of this update (Appendix~\ref{app:stationary}).

\item On OGBench~\citep{ogbench}, ALBUM with single $Q$ and $V$ networks and a Gaussian actor matches the average performance of FQL~\citep{park2025flow} and substantially outperforms ReBRAC~\citep{tarasov2023revisiting}, while using the fewest parameters and the least peak GPU memory among all measured methods. 
\end{itemize}

%% file: contents/02_relat.tex
\section{Related Works}
\label{sec:related works}
\paragraph{Offline RL.}
Offline RL aims to learn a policy from a pre-collected dataset, where the maximization can amplify erroneous value estimates at
out-of-distribution actions~\citep{fujimoto2019off, kumar2019stabilizing,
kostrikov2021offline, mao2024doubly}. Existing methods mitigate this problem
through policy constraints~\citep{fujimoto2019off, kumar2019stabilizing,
fujimoto2021minimalist}, pessimistic value
regularization~\citep{kumar2020conservative, lyu2022mildly, mao2024doubly},
in-sample regression~\citep{kostrikov2021offline, garg2023extreme,
xu2023offline}, hybrid regularization~\citep{tarasov2023revisiting}, and
generative actor-critic methods~\citep{wang2022diffusion, park2025flow}.

\paragraph{Multi-step Values.}
Multi-step and Monte Carlo returns have been used as lower bounds for critic
values to facilitate value propagation, as auxiliary terms in a regression
objective~\citep{s.he2017learning, oh2018self}. Iterated Bellman inequalities
have been used in LP-based value-function approximation to obtain tighter
underestimators~\citep{wang2015approximate}. In offline RL, multi-step targets
are biased for the learned policy unless corrected~\citep{li2026reinforcement},
and a prevailing approach is action chunking, which conditions the critic on the action sequence~\citep{li2026reinforcement, kim2026deas, li2026decoupled,
song2026chunkguided}. ALBUM instead imposes $K$-step returns as inequality constraints of the program that defines a single-action critic, rather than as auxiliary terms alongside a regression loss. Under deterministic dynamics,
these constraints require no off-policy correction and leave the solution of the program unchanged, while enabling long-horizon value propagation
without action chunking.

\paragraph{Linear Programs for Bellman Optimality.}
The LP characterization replaces the maximization with Bellman
inequalities~\citep{de2003linear, puterman2014markov,
lakshminarayanan2017linearly}. Off-policy and offline work building on this
formulation mostly passes to the dual, either estimating density ratios for
evaluation or policy optimization~\citep{nachum2019dualdice, zhang2020gendice,
nachum2019algaedice, lee2021optidice, nachum2020reinforcement} or solving
saddle-point problems~\citep{zhan2022offline, kamoutsi2021efficient,
ozdaglar2023revisiting, gabbianelli2024offline}. Other work introduces variables alongside $Q$~\citep{mehta2020convex, lu2023convex, jeong2025stochastic}, but in these programs $V$ is not a decision variable or carries no weight in the objective. \citet{lee2019stochastic, bas2021logistic} carry $Q$ and $V$ together, but only $V$ is weighted, and $Q$ is fixed by the Bellman equality rather than determined by the objective. \citet{lee2025analysis} relax the inequalities with a log
barrier, but their program involves $Q$ alone and therefore evaluates $Q$ at
next actions. Hinge relaxations appear in \citet{ryu2019caql}, whose hinge
retains a maximization in its argument, and in \citet{sikchi2024dual}, whose
$Q$ is fit by regression outside the program. To our knowledge, LBLP is the first primal program for Bellman optimality whose constraints involve only dataset pairs and whose objective determines both $Q$ and $V$. Two properties follow from this structure. The critic is the solution of a program rather than a regression onto bootstrapped targets, and multi-step constraints along dataset trajectories are implied by the one-step constraints, so they leave the solution unchanged without off-policy correction.

%% file: contents/03_preli.tex
\section{Preliminaries}

\subsection{Markov Decision Processes and Offline Reinforcement Learning}

We consider discounted MDP $\mathcal{M}=(\mathcal{S},\mathcal{A},P,R,\gamma)$ with finite
state space $\mathcal{S}$, finite action space
$\mathcal{A}$, transition kernel $P(\cdot\mid s,a)$,
expected reward $R(s,a)$, and discount factor $0\le\gamma<1$. We
identify an action-value function $Q$ with a vector in
$\mathbb{R}^{|\mathcal{S}||\mathcal{A}|}$ indexed by $(s,a)$ and a state-value
function $V$ with a vector in $\mathbb{R}^{|\mathcal{S}|}$, and we write
$\langle\cdot,\cdot\rangle$ for the Euclidean inner product on either space. The
same symbol $P$ denotes the operator this kernel induces, $(PV)(s,a) :=
\sum_{s'\in\mathcal{S}}P(s'\mid s,a)V(s')$. The realized reward $r(s,a,s')$,
whose average over successors is $R(s,a)$, satisfies $r_{\min}\le r(s,a,s')\le
r_{\max}$. For a scalar $x$ we write $[x]_{+} := \max\{x,0\}$, and both this
operation and inequalities between vectors are applied componentwise. Let $\mathbb{E}_{\mathcal{M},\pi}$ denote the expectation over
trajectories $(s_0,a_0,s_1,a_1,\ldots)$ generated by $a_k\sim\pi(\cdot\mid
s_k)$ and $s_{k+1}\sim P(\cdot\mid s_k,a_k)$, with
$r_{k}:=r(s_k,a_k,s_{k+1})$. The value functions of $\pi$ are
$V^{\pi}(s)=\mathbb{E}_{\mathcal{M},\pi}[\sum_{k=0}^{\infty}\gamma^{k}r_{k}\mid
s]$ and
$Q^{\pi}(s,a)=\mathbb{E}_{\mathcal{M},\pi}[\sum_{k=0}^{\infty}\gamma^{k}r_{k}
\mid s,a]$. Their optima $Q^{*}$ and $V^{*}$ are linked by
$V^{*}(s)=\max_{a\in\mathcal{A}}Q^{*}(s,a)$. In the offline setting the agent is given a fixed dataset
$\mathcal{D}=\{(s_i,a_i,r_i,s_i')\}_{i=1}^{m}$ of $m$ transitions collected by a
behavior policy $\beta$, where $r_i=r(s_i,a_i,s_i')$. Since
our algorithm draws contiguous subsequences rather than isolated transitions, we
assume $\mathcal{D}$ retains the trajectory structure of the collected data. The
dataset also induces an empirical Markov decision
process~\citep{fujimoto2019off}, whose definition is given in
Appendix~\ref{app:empirical-mdp}.

\subsection{Linear Programming Characterization of Bellman Optimality}
\label{subsec:lp}

For any $\rho$ with strictly positive entries, $V^*$ is the unique solution of
the linear program~\citep{de2003linear, puterman2014markov}
\begin{equation}
\min_{V}\ \langle\rho,V\rangle
\quad\text{subject to}\quad
V(s)\ge R(s,a)+\gamma(PV)(s,a),\qquad (s,a)\in\mathcal{S}\times\mathcal{A}.
\label{eq:lp-v}
\end{equation}
The LP stated over $Q$ involves $Q(s',a')$ for every action $a'$ at the next state $s'$~\citep{puterman2014markov, lee2025analysis}.

%% file: contents/04_motiv.tex
\section{Lifted Bellman Linear Programming}
\label{sec:LBLP}

We introduce the \emph{Lifted Bellman Linear Program} (LBLP), which lifts the
LP of Section~\ref{subsec:lp} from $V$ or $Q$ alone to the pair \((Q,V)\). The
lifting is the classical epigraph reformulation~\citep{boyd2004convex}, under
which \(V(s)\ge\max_{a\in\mathcal A}Q(s,a)\) is equivalent to \(Q(s,a) -  V(s) \le 0\)
for all \(a\in\mathcal A\). We state the formulation for an arbitrary finite
MDP, since its exactness does not rely on the offline setting, and
Section~\ref{sec:offline-lblp} instantiates it on the empirical MDP induced by
the dataset.

\begin{definition}[Lifted Bellman Linear Program (LBLP)]
\label{def:lblp}
Let $\mathcal M=(\mathcal S,\mathcal A,P,R,\gamma)$ be a finite discounted MDP with
$0\le\gamma<1$, let $\beta$ be a behavior policy on $\mathcal M$, and let $K>1$ be a finite
integer. The decision variables are $Q$ and $V$, the weights satisfy $c_Q(s,a)>0$
for all $(s,a)\in\mathcal S\times\mathcal A$ and $c_V(s)>0$ for all $s\in\mathcal S$,
and the objective is $F(Q,V):=\langle c_Q,Q\rangle+\langle c_V,V\rangle$. Let us define
\begin{equation*}
y_K(s,a):=\mathbb E_{\mathcal M,\beta}\!\left[
\sum_{t=0}^{K-1}\gamma^{t}R(s_t,a_t)+\gamma^{K}V(s_K)\;\middle|\;s,a\right],
\end{equation*}
where the rollout follows $a_t\sim\beta(\cdot\mid s_t)$ for $t\ge1$. The constraint
functions are
\begin{equation*}
\begin{aligned}
g_{\mathrm B}(Q,V)(s,a)
&:=R(s,a)+\gamma (PV)(s,a)-Q(s,a),
&
g_{\mathrm E}(Q,V)(s,a)
&:=Q(s,a)-V(s),\\
g_{\mathrm{KQ}}(Q,V)(s,a)
&:=y_K(s,a)-Q(s,a),
&
g_{\mathrm{KV}}(Q,V)(s,a)
&:=y_K(s,a)-V(s).
\end{aligned}
\end{equation*}
The LBLP associated with $(\mathcal M,\beta,K)$ is defined as
\begin{equation}\label{eq:lblp}
\min_{Q,V}\ F(Q,V)\quad\text{s.t.}\quad
g_{\mathrm B}\le0,\quad g_{\mathrm E}\le0,\quad
g_{\mathrm{KQ}}\le0,\quad g_{\mathrm{KV}}\le0 .
\end{equation}

\end{definition}

The four constraints play two distinct roles. The one-step constraints
$g_{\mathrm B}$ and $g_{\mathrm E}$ together determine the Bellman-optimal pair,
whereas $g_{\mathrm{KQ}}$ and $g_{\mathrm{KV}}$ leave both the feasible set and
the solution unchanged, as Section~\ref{sec:lblpexactness} shows. We retain them because each
directly relates the value at $(s,a)$ to a value $K$ transitions later through a
single inequality, and Section~\ref{subsec:acc} shows that this can accelerate
convergence toward the solution.

\subsection{Exactness of LBLP}
\label{sec:lblpexactness}

\begin{lemma}[Rollout constraints are implied by the one-step constraint]
\label{lem:kstep-redundant}
Let $(Q,V)$ satisfy $g_{\mathrm{B}}(Q,V) \le 0$ and $g_{\mathrm{E}}(Q,V) \le 0$.
Then $g_{\mathrm{KQ}}(Q,V) \le 0$ and $g_{\mathrm{KV}}(Q,V) \le 0$.
\end{lemma}

The proof is in Appendix~\ref{app:kstep-redundant}. Note that the implication holds in one direction only. A point satisfying the $K$-step constraint ($g_{\mathrm{KQ}}(Q,V), g_{\mathrm{KV}}(Q,V)$) need not satisfy the one-step constraint ($g_B(Q,V), g_E(Q,V)$), since the latter demands an inequality at every transition of the rollout while the former demands only one on their aggregate. By Lemma~\ref{lem:kstep-redundant}, the rollout constraints never exclude the Bellman-optimal pair. It remains to show that the one-step constraint forces $Q \ge Q^*$ and $V \ge V^*$ at every feasible point, so that minimizing the objective returns the pair itself.

\begin{proposition}[Exactness of the LBLP] \label{prop:lblp-exactness}
Let $\mathcal{M}$ be a finite discounted MDP, $\beta$ be a policy on
$\mathcal{M}$, and $K>1$ be a finite integer. Then the LBLP of
Definition~\ref{def:lblp} has the unique minimizer $(Q,V)=(Q^*,V^*)$.
\end{proposition}

The proof is in Appendix~\ref{app:lblp-exactness}. Both results hold for any
finite MDP and any rollout policy, so the rollout constraints require no
off-policy correction, and they carry over to the offline setting, where
$\mathcal{M}$ is unavailable and only the dataset $\mathcal{D}$ is given.


\subsection{LBLP on the Empirical MDP of an Offline Dataset}
\label{sec:offline-lblp}

The dataset \(\mathcal D\) induces an empirical MDP \({\mathcal M}_{\mathcal D}\) with pair set \(\mathcal X_{\mathcal D}=\{(s,a):(s,a,r,s')\in\mathcal D\}\) and action sets
\(\mathcal A_{\mathcal D}(s)=\{a:(s,a)\in\mathcal X_{\mathcal D}\}\). Appendix~\ref{app:empirical-mdp} gives the full construction. Since \({\mathcal M}_{\mathcal D}\) is a finite discounted MDP, Definition~\ref{def:lblp} applies directly. We write
\((Q^{*}_{\mathcal D},V^{*}_{\mathcal D})\) for its minimizer, which
Proposition~\ref{prop:lblp-exactness} identifies as the Bellman-optimal pair of
\({\mathcal M}_{\mathcal D}\), so that
\(V^{*}_{\mathcal D}(s)=\max_{a\in\mathcal A_{\mathcal D}(s)}Q^{*}_{\mathcal D}(s,a)\)
with the maximum taken over the actions the data contains at \(s\). Here $\beta$ is a policy on $\mathcal M_{\mathcal D}$, since it selects only
actions the data contains, and the program never evaluates an action outside the
data, as its constraints live on $\mathcal X_{\mathcal D}$ and the maximization
ranges over $\mathcal A_{\mathcal D}$. For \((s,a)\in\mathcal X_{\mathcal D}\), let \(\mathcal H_{\mathcal D}(s,a)\)
be the set of paths obtained by taking a trajectory of \(\mathcal D\) in which
\((s,a)\) appears and following it from that point to its end. Writing \(G(\tau)=\sum_{t\ge 0}\gamma^{t}R_{\mathcal D}(s_t,a_t)\) for the
discounted return along \(\tau=\big((s_0,a_0),(s_1,a_1),\dots\big)\), we define
\(G_{\mathcal D}(s,a):=\max_{\tau\in\mathcal H_{\mathcal D}(s,a)} G(\tau)\).

\begin{proposition}[Sandwich property of the LBLP solution]
\label{prop:sandwich}
Suppose the transition dynamics of $\mathcal{M}$ are deterministic, so that $P_{\mathcal{D}}(\cdot \mid s,a) = P(\cdot \mid s,a)$ and $R_{\mathcal{D}}(s,a) = R(s,a)$ for every $(s,a) \in \mathcal{X}_{\mathcal{D}}$, and suppose every trajectory of $\mathcal{D}$ ends at a state $s$ with $\mathcal{A}_\mathcal{D}(s)=\emptyset$ that is terminal in $\mathcal{M}$. Then for every \((s,a)\in\mathcal X_{\mathcal D}\), \(G_{\mathcal D}(s,a)
\;\le\;
Q^{*}_{\mathcal D}(s,a)
\;\le\;
Q^{*}(s,a),\) where $Q^{*}$ denotes the optimal value in $\mathcal{M}.$
\end{proposition}

\begin{wrapfigure}{r}{0.35\textwidth}
\vspace{-\baselineskip}
\centering
\includegraphics[width=0.34\textwidth]{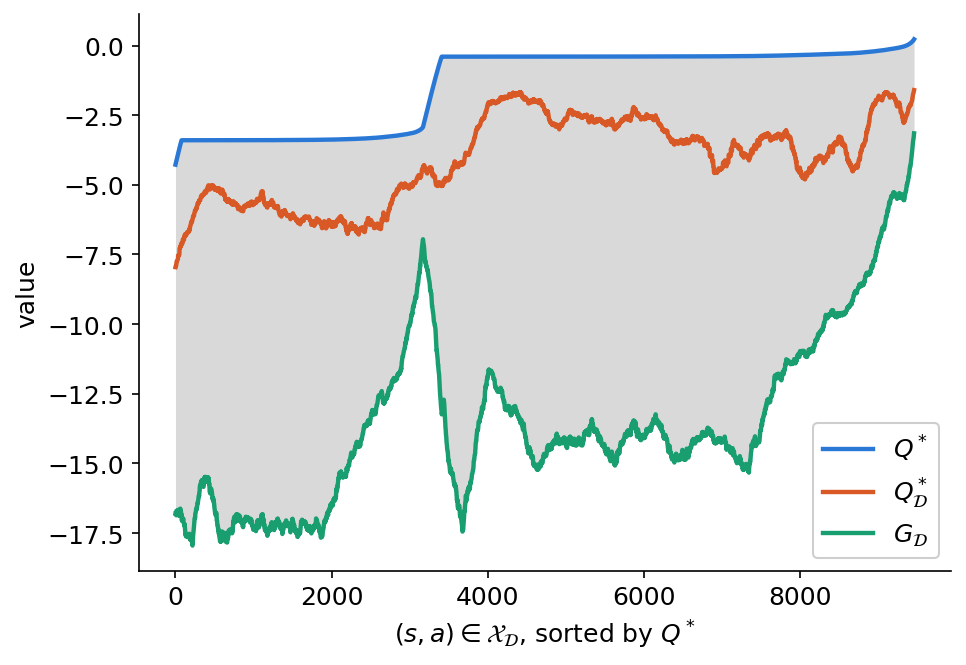}
\caption{Proposition~\ref{prop:sandwich} on a tabular gridworld. Each point is
a sorted pair in $\mathcal X_{\mathcal D}$.}
\label{fig:sandwich}
\vspace{-\baselineskip}
\end{wrapfigure}

The proof is provided in Appendix~\ref{app:proof-sandwich}. Under full behavioral coverage the left inequality also holds against
\(Q^{\beta}\), and both bounds carry over to the state-value function
(Appendix~\ref{app:sandwich-variants}). Conservative methods such as CQL \citep{kumar2020conservative} obtain this upper bound on their fixed point through a penalty coefficient that controls the degree of conservatism, and excessive conservatism can degrade performance \citep{lyu2022mildly, cen2024learning}. For the LBLP minimizer, both bounds of Proposition~\ref{prop:sandwich} hold without a coefficient, and the lower bound $G_{\mathcal{D}}$ depends on $\mathcal{D}$ alone. Figure~\ref{fig:sandwich} verifies this on a tabular gridworld where $Q^{*}$ is
available in closed form, solving the LBLP of Definition~\ref{def:lblp} as an
exact LP. The minimizer stays inside the admissible interval at
every pair of $\mathcal X_{\mathcal D}$.
Appendix~\ref{app:sandwich} gives the construction and repeats the measurement
across eight seeds.

\subsection{Acceleration through Rollout constraints}
\label{subsec:acc}

By Lemma~\ref{lem:kstep-redundant} and Proposition~\ref{prop:lblp-exactness},
the rollout constraints and their horizon can be varied without changing the solution, and we show that such a choice can accelerate convergence toward it. To characterize this effect, we consider an idealized scheme that
approaches $(Q^*_{\mathcal D},V^*_{\mathcal D})$ from below. Assume deterministic
dynamics and fix, for each $(s,a)\in\mathcal X_{\mathcal D}$, a dataset
trajectory starting from $(s,a)$, along which $s^{(k)}(s,a)$ is the state reached
after $k$ transitions and $G_k(s,a)$ collects the discounted rewards on the way.
For $k\in\{1,K\}$, let $L_k[V](s,a):=G_k(s,a)+\gamma^kV(s^{(k)}(s,a))$ denote the
$k$-step target. Starting from an underestimate $V_0\le V^*_{\mathcal D}$, the
scheme updates
\begin{equation}
Q_{n+1}(s,a)=\max_{k\in\{1,K\}}L_k[V_n](s,a),\qquad
V_{n+1}(s)=\max_{a\in\mathcal A_{\mathcal D}(s)}Q_{n+1}(s,a),
\label{eq:repair}
\end{equation}
so that each state settles at the largest $Q_{n+1}(s,a)$ the data offers there.
With $K=1$, \eqref{eq:repair} is value iteration on $\mathcal M_{\mathcal D}$,
so comparing it with $K>1$ isolates how the rollout horizon affects value
propagation. We define the suboptimality of the realized $k$-step path and the
current value gap as
$\delta_k(s,a) := Q^*_{\mathcal D}(s,a) - L_k[V^*_{\mathcal D}](s,a)$ and
$e_n := V^*_{\mathcal D} - V_n$, respectively.

\begin{theorem}[Multi-horizon speedup] \label{thm:speedup}
Under the hypotheses of Lemma~\ref{lem:propagation}, the following hold.
\begin{enumerate}
    \item $\|e_{n+1}\|_\infty \le \gamma\|e_n\|_\infty$ for any rollout length $K$,
    so the error decays geometrically at rate $\gamma$ and the iterates converge to
    $(Q^*_{\mathcal D}, V^*_{\mathcal D})$. This is the minimizer of LBLP on
    $\mathcal M_{\mathcal D}$ identified in Proposition~\ref{prop:lblp-exactness}.

    \item The error $\big( Q^*_{\mathcal D} - Q_{n+1}\big)(s,a)$ is a minimum
over $k \in \{1,K\}$ and is therefore no larger than its $k=1$ term, so
adding rollout horizon never slows the iteration. At coordinates where the
$K$-step term attains the minimum, the error is at most
$\delta_K(s,a)+\gamma^K\|e_n\|_\infty$ rather than $\gamma\|e_n\|_\infty$.
\end{enumerate}
\end{theorem}

The proof is provided in Appendix~\ref{app:speedup}. To verify this
acceleration, we ran a numerical simulation on the same gridworld environment
used in Section~\ref{sec:offline-lblp}, with the detailed setting described in
Appendix~\ref{app:toy}. The results are reported in Appendix~\ref{app:kspeed}.

\subsection{Exactness of the Relaxed LBLP}
\label{sec:hinge_theory}

LBLP is a constrained program, and enforcing its constraints directly becomes
impractical once $Q$ and $V$ are parameterized by neural networks. Replacing the constraints with penalties on
their violation turns LBLP into an unconstrained problem, and with hinge
penalties this replacement keeps the solution of LBLP once the penalty weights
are large enough. We state the relaxation over the same variables as
Definition~\ref{def:lblp}, with the rollout value given by the expectation
$y_K$, thereby maintaining a direct correspondence with the original program.

\begin{definition}[Relaxed LBLP]
\label{def:relaxed_LBLP}
Let \(\mathcal M\), \(\beta\), and \(K\) be as in Definition~\ref{def:lblp}, with the objective and constraint functions defined there. For
penalty weights \(\lambda_{\mathrm B},\lambda_{\mathrm E},\lambda_{\mathrm K}>0\),
the relaxed LBLP minimizes
\begin{equation} \label{eq:relaxed-lblp}
\begin{aligned}
\mathcal L_{\mathrm{R}}(Q,V)
:=&~
F(Q,V)
+
\lambda_{\mathrm B}
\left\|
\left[g_{\mathrm B}(Q,V)\right]_+
\right\|_1
+
\lambda_{\mathrm E}
\left\|
\left[g_{\mathrm E}(Q,V)\right]_+
\right\|_1
\\
&+
\lambda_{\mathrm K}
\left(
\left\|
\left[g_{\mathrm{KQ}}(Q,V)\right]_+
\right\|_1
+
\left\|
\left[g_{\mathrm{KV}}(Q,V)\right]_+
\right\|_1
\right),
\end{aligned}
\end{equation}
\end{definition}

\begin{proposition}[Exactness of the hinge relaxation] \label{prop:exactness}
Let $(\lambda^*_{\mathrm B}, \lambda^*_{\mathrm E}) \ge 0$ be optimal dual
variables of the program on ${\mathcal M}_{\mathcal D}$ that keeps only the
constraints $g_{\mathrm B} \le 0$ and $g_{\mathrm E} \le 0$, and suppose
$\lambda_{\mathrm B} > \max_{(s,a)\in\mathcal X_{\mathcal D}} \lambda^*_{\mathrm B}(s,a)$
and
$\lambda_{\mathrm E} > \max_{(s,a)\in\mathcal X_{\mathcal D}} \lambda^*_{\mathrm E}(s,a)$.
Then $( Q^*_{\mathcal D}, V^*_{\mathcal D})$ is the unique
minimizer of~\eqref{eq:relaxed-lblp}, for every $\lambda_{\mathrm K}\ge0$ and
every rollout length $K$.
\end{proposition}

The proof is provided in Appendix~\ref{app:exactness}. The threshold is finite,
so the solution is recovered at a finite setting of the coefficients rather
than in a limit. The threshold involves $g_{\mathrm B}$ and $g_{\mathrm E}$ alone, so
exactness places no condition on the rollout coefficient
$\lambda_{\mathrm K}$. The acceleration of Section~\ref{subsec:acc} also appears
when \eqref{eq:relaxed-lblp} is minimized by gradient descent, as
Appendix~\ref{app:kspeed} shows. Proposition~\ref{prop:exactness} relates LBLP to its tabular hinge relaxation and does not enter the analysis of ALBUM.

%% file: contents/05_algor.tex
\section{Approximate Lifted Bellman Unconstrained Minimization}
\label{sec:album}

ALBUM implements the relaxed LBLP of Definition~\ref{def:relaxed_LBLP}
with neural networks. Its critic update detaches the $K$-step rollout targets
by \texttt{stop\_gradient}. The lifted variables are represented by separate
networks $Q_\theta$ and $V_\phi$. Given a transition sequence
$(s_0,a_0,r_0,\dots,r_{K-1},s_K)$ from the dataset with
$(s,a)=(s_0,a_0)$, let us define
$\hat y_k(s,a):=\sum_{t=0}^{k-1}\gamma^{t}r_t+\gamma^{k}V_\phi(s_k)$
for $k\in\{1,K\}$.
The ALBUM critic loss is
\begin{equation}
\label{eq:album-loss}
\boxed{
\begin{aligned}
\mathcal L(\theta,\phi)
=&~
\omega_Q\,\mathbb E_{\mathcal D}\!\left[Q_\theta(s,a)\right]
+
\omega_V\,\mathbb E_{\mathcal D}\!\left[V_\phi(s)\right]
\\
&+
\lambda_{\mathrm B}\,
\mathbb E_{\mathcal D}\!\left[
\left[\hat y_1(s,a)-Q_\theta(s,a)\right]_+
\right]
+
\lambda_{\mathrm E}\,
\mathbb E_{\mathcal D}\!\left[
\left[Q_\theta(s,a)-V_\phi(s)\right]_+
\right]
\\
&+
\lambda_{\mathrm K}
\left(
\mathbb E_{\mathcal D}\!\left[
\left[\operatorname{sg}\!\big[\hat y_K(s,a)\big]-Q_\theta(s,a)\right]_+
\right]
+
\mathbb E_{\mathcal D}\!\left[
\left[\operatorname{sg}\!\big[\hat y_K(s,a)\big]-V_\phi(s)\right]_+
\right]
\right)
\end{aligned}
}
\end{equation}
where $\operatorname{sg}[\cdot]$ denotes \texttt{stop\_gradient}, so the rollout
target is treated as a constant at each update. The two objective terms are
the empirical counterparts of $F$, with the sampling distribution of
$\mathcal D$ playing the role of the weights $c_Q$ and $c_V$, which are
positive on every pair of $\mathcal M_{\mathcal D}$. The sample-based
$\hat y_K$ replaces the expectation in Definition~\ref{def:lblp} with a
single trajectory of $\mathcal D$, and under deterministic dynamics the
one-step constraints still imply this sampled constraint
(Remark~\ref{rem:path-redundant}).

Unlike a target network in temporal difference regression, which stabilizes a
bootstrapped target by delaying its update, the detached target
$\operatorname{sg}[\hat y_K]$ enters only as a lower bound inside a hinge.
Treating $\hat y_K$ as a constant removes a gradient term that shared parameters
would otherwise introduce into the update of $V_\phi(s)$ and that is absent from
the coordinatewise formulation in \eqref{eq:relaxed-lblp}
(Appendix~\ref{ablation:stop_grad}).
Thus, the implemented critic update is distinct from the gradient of the
scalar objective in \eqref{eq:album-loss}. Under deterministic dynamics and additional conditions on the dataset, Appendix~\ref{app:stationary} further shows that $(Q^*_\mathcal{D}, V^*_\mathcal{D})$ is a stationary point of the detached update (Proposition~\ref{prop:stationary}).

Since our contribution concerns critic optimization, we extract the policy with
a standard objective from the offline RL literature so that the reported gains are attributable to the critic rather than to policy
extraction. We use the DDPG+BC
objective~\citep{fujimoto2021minimalist, tarasov2023revisiting, park2024value},
\begin{equation}
\label{eq:album-actor-loss}
\mathcal L_{\mathrm{actor}}(\psi)
=
-\,\mathbb E_{s\sim\mathcal D}\!\left[Q_\theta(s,\pi_\psi(s))\right]
+
\alpha_{\mathrm{BC}}\,
\mathbb E_{(s,a)\sim\mathcal D}\!\left[\left\|\pi_\psi(s)-a\right\|_2^2\right].
\end{equation}

As in other in-sample methods~\citep{kostrikov2021offline}, the critic is never
trained at actions outside the dataset, so its values there rest on function
approximation, and the behavioral cloning term keeps the actor's queries near
the data.

\subsection{A Necessary Condition for Bounded Updates}
\label{sec:boundedness}

\begin{proposition}[Necessary conditions for bounded updates]
\label{prop:bounded}
Let $Q$ and $V$ in \eqref{eq:album-loss} be free variables, and consider lowering every entry of $Q$ and $V$ by the same amount.
\begin{enumerate}
\item[(i)] The ALBUM update, which treats $\mathrm{sg}[\hat{y}_K]$ as a constant, does not lower $Q$ and $V$ without bound along this direction only if
\begin{equation}\label{eq:bounded-sg}
\lambda_{\mathrm B}(1-\gamma) + 2\lambda_{\mathrm K} \;\ge\; \omega_{\mathrm Q}+\omega_{\mathrm V}.
\end{equation}
\item[(ii)] The gradient update of \eqref{eq:album-loss} without stop gradient does not lower $Q$ and $V$ without bound along this direction only if
\begin{equation}\label{eq:bounded}
\lambda_{\mathrm B}(1-\gamma) + 2\lambda_{\mathrm K}(1-\gamma^{K}) \;\ge\; \omega_{\mathrm Q}+\omega_{\mathrm V}.
\end{equation}
\end{enumerate}
\end{proposition}

The proof is provided in Appendix~\ref{app:bounded}. ALBUM follows the update in (i), so \eqref{eq:bounded-sg} governs the downward drift of its critic, and \eqref{eq:bounded} applies to an update that ALBUM does not use. The detached rollout target removes $\gamma$ and $K$ from the condition. The same condition extends to neural parameterizations whose outputs can be lowered by the same constant at every input, and Appendix~\ref{app:boundness} shows that the critic leaves the feasible range when it is violated. The default coefficients in Table~\ref{tab:hyperparams} give $\lambda_{\mathrm B}(1-\gamma)+2\lambda_{\mathrm K}=2.0125$ against $\omega_{\mathrm Q}+\omega_{\mathrm V}=0.15$.

%% file: contents/06_exper.tex
\section{Experiments}
\label{sec:experiments}

We evaluate ALBUM on the state-based tasks of OGBench~\citep{ogbench} and
compare it with different policy classes. We report the success rate after
$1$M gradient steps over $8$ seeds, with $50$ evaluation episodes per seed.
All five critic coefficients in~\eqref{eq:album-loss} are fixed across all
domains at the values in Table~\ref{tab:hyperparams} and are not tuned per
environment. Since \eqref{eq:bounded-sg} involves neither the reward nor the
dataset, these values satisfy it in every domain. Although the coefficients
can affect performance, as shown in Appendix~\ref{app:ablation}, we select them
once on the ablation tasks. For the actor, we follow the tuning protocol of
\citet{park2025flow} and tune only the behavioral cloning coefficient
$\alpha_{\mathrm{BC}}$ in~\eqref{eq:album-actor-loss} on the default task of
each environment. This per-environment tuning budget is no larger than that of any baseline, as
ALBUM performs no per-environment critic tuning whereas
ReBRAC~\citep{tarasov2023revisiting} tunes its critic regularization coefficient
per environment~\citep{park2025flow}. The ablations rerun FQL and ReBRAC with
the hyperparameters in Table~6 of \citet{park2025flow}.

\paragraph{Comparison across policy classes.}
Table~\ref{tab:main} compares ALBUM with ten baselines, including Gaussian
policies (BC, IQL~\citep{kostrikov2021offline},
ReBRAC~\citep{tarasov2023revisiting}), diffusion policies
(IDQL~\citep{hansen2023idql}, SRPO~\citep{chen2024score},
CAC~\citep{ding2024consistency}), and flow policies (FAWAC, FBRAC, IFQL,
FQL~\citep{park2025flow}). Their scores are taken from \citet{park2025flow}, who
tuned every baseline under the fair tuning budget, which allows a direct
comparison with ALBUM tuned under the same protocol. ALBUM and FQL achieve the
two highest state-based averages in Table~\ref{tab:main}. ALBUM attains the
highest score on \texttt{antmaze-large-navigate},
\texttt{antmaze-giant-navigate}, \texttt{scene-play}, and
\texttt{puzzle-3x3-play}, by a large margin on the last two, and all four
domains require long-horizon value propagation. ALBUM is weakest on
\texttt{humanoidmaze-large-navigate}. It achieves these results with a Gaussian
policy and without a critic ensemble or a target network. Compared with ReBRAC,
which uses the same policy class, ALBUM improves the average by 12.4 points and
leads on eight of ten domains.



\providecommand{\std}[1]{{\scriptsize$\pm$#1}}
\begin{table}[t]
\centering\footnotesize\setlength{\tabcolsep}{2.5pt}
\resizebox{\textwidth}{!}{%
\begin{tabular}{lcccccccccc|c}
\toprule
 & \multicolumn{3}{c}{Gaussian Policies} & \multicolumn{3}{c}{Diffusion Policies} & \multicolumn{4}{c}{Flow Policies} & \multicolumn{1}{c}{Gaussian (Ours)} \\
\cmidrule(lr){2-4}\cmidrule(lr){5-7}\cmidrule(lr){8-11}\cmidrule(lr){12-12}
Task Category & BC & IQL & ReBRAC & IDQL & SRPO & CAC & FAWAC & FBRAC & IFQL & FQL & ALBUM (Ours) \\
\midrule
antmaze-large-navigate & 10.6 & 53.4 & \underline{80.8} & 20.8 & 10.6 & 32.8 & 6.4 & 60.2 & 28.0 & 78.6 & \textbf{89.8}\std{1.0} \\
antmaze-giant-navigate & 0.2 & 4.0 & \underline{26.2} & 0.0 & 0.0 & 0.0 & 0.0 & 3.8 & 2.6 & 8.6 & \textbf{28.1}\std{5.1}\\
humanoidmaze-medium-navigate & 2.0 & 32.8 & 21.8 & 0.8 & 1.4 & 52.8 & 19.4 & 38.4 & \textbf{60.4} & \underline{57.4} & 35.5\std{1.8} \\
humanoidmaze-large-navigate & 0.4 & 2.4 & 2.6 & 0.6 & 0.2 & 0.6 & 0.2 & 2.2 & \textbf{11.0} & \underline{4.2} & 1.6\std{0.5} \\
antsoccer-arena-navigate & 1.0 & 8.4 & 0.0 & 11.8 & 1.0 & 1.8 & 12.4 & 16.0 & 33.2 & \textbf{60.2} & \underline{36.1}\std{3.8} \\
cube-single-play & 5.4 & 83.0 & 90.6 & \underline{94.6} & 79.6 & 85.2 & 81.2 & 78.6 & 79.2 & \textbf{95.8} & 83.2\std{1.7} \\
cube-double-play & 1.6 & 6.4 & 12.2 & 14.6 & 1.4 & 5.8 & 5.2 & \underline{15.0} & 14.0 & \textbf{28.6} & 13.3\std{2.8} \\
scene-play & 4.6 & 27.6 & 40.6 & 46.2 & 20.0 & 39.8 & 29.8 & 44.8 & 30.4 & \underline{55.8} & \textbf{64.4}\std{5.5} \\
puzzle-3x3-play & 1.8 & 9.0 & 21.6 & 10.4 & 17.8 & 19.4 & 6.4 & 14.0 & 19.0 & \underline{29.6} & \textbf{61.2}\std{6.4} \\
puzzle-4x4-play & 0.2 & 7.4 & 14.0 & \textbf{29.2} & 10.6 & 14.8 & 0.4 & 13.2 & \underline{25.2} & 17.2 & 20.6\std{2.5} \\
\midrule
Average (state-based) & 2.8 & 23.4 & 31.0 & 22.9 & 14.3 & 25.3 & 16.1 & 28.6 & 30.3 & \textbf{43.6} & \textbf{43.4} \\
\bottomrule
\end{tabular}}
\caption{Success rate (\%) on OGBench state-based tasks. ALBUM: mean over $5$
tasks and $8$ seeds at $1$M steps, $\pm$ std across seeds. Baselines from
\citet{park2025flow}. Best in bold, second underlined.}
\label{tab:main}
\end{table}

\subsection{Ablations}
\label{sec:ablations}

In this section, we examine the design properties induced by the LBLP formulation and
test whether the empirical behavior is consistent with the theoretical
analysis. Unless otherwise specified, hyperparameters follow Table~\ref{tab:hyperparams}, with task, seed, and evaluation settings specified in Appendix~\ref{app:ablation}.

\paragraph{Why does ALBUM remain stable without target networks or critic ensembles?}
Averaged over the three tasks in Table~\ref{tab:targetensemble}, removing target networks and EMA updates degrades ReBRAC by 16\% and FQL by 5\%, using a single critic degrades them by 87\% and 34\%, and removing both degrades them by 90\% and 49\%. These results are consistent with the LP-derived critic objective, which contains neither moving-target regression nor action maximization. See Appendix~\ref{ablation:EMA,ensemble} for detailed comparisons.

\providecommand{\std}[1]{{\scriptsize$\pm$#1}}
\begin{table}[h]
\centering\footnotesize\setlength{\tabcolsep}{2.5pt}
\resizebox{\textwidth}{!}{%
\begin{tabular}{l rr rr rr}
\toprule
& \multicolumn{2}{c}{antmaze-large-navigate-task1} & \multicolumn{2}{c}{scene-play-task2} & \multicolumn{2}{c}{puzzle-3x3-play-task4} \\
\cmidrule(lr){2-3}\cmidrule(lr){4-5}\cmidrule(lr){6-7}
 & Success (\%) & $\Delta$ & Success (\%) & $\Delta$ & Success (\%) & $\Delta$ \\
\midrule
ALBUM(w/ ensemble \& w/o target) & $91.8$\std{4.3} & $-$ & $93.0$\std{3.0} & $-$ & $93.2$\std{8.4} & $-$ \\
\textbf{ALBUM(w/o ensemble \& target, Ours)} & $92.5$\std{3.1} & $+0.7$ & $93.2$\std{3.2} & $+0.2$ & $93.0$\std{7.9} & $-0.2$ \\
\midrule
FQL & $88.0$\std{3.3} & $-$ & $84.8$\std{11.5} & $-$ & $10.8$\std{2.0} & $-$ \\
FQL(w/o ensemble) & $54.0$\std{41.5} & $-34.0$ & $58.5$\std{14.7} & $-26.2$ & $8.2$\std{2.1} & $-2.5$ \\
FQL(w/o target) & $83.5$\std{7.5} & $-4.5$ & $77.0$\std{8.5} & $-7.8$ & $14.0$\std{6.6} & $+3.2$ \\
FQL(w/o ensemble \& target) & $45.8$\std{11.5} & $-42.2$ & $35.2$\std{6.6} & $-49.5$ & $12.0$\std{3.9} & $+1.2$ \\
\midrule
ReBRAC & $95.8$\std{3.4} & $-$ & $52.0$\std{15.5} & $-$ & $6.5$\std{2.4} & $-$ \\
ReBRAC(w/o ensemble) & $7.0$\std{18.5} & $-88.8$ & $7.8$\std{10.3} & $-44.2$ & $5.5$\std{5.0} & $-1.0$ \\
ReBRAC(w/o target) & $56.0$\std{38.2} & $-39.8$ & $66.8$\std{11.1} & $+14.8$ & $7.2$\std{3.0} & $+0.8$ \\
ReBRAC(w/o ensemble \& target) & $0.0$\std{0.0} & $-95.8$ & $10.8$\std{10.9} & $-41.2$ & $4.2$\std{4.4} & $-2.2$ \\
\bottomrule
\end{tabular}}
\caption{Success rate (\%) on OGBench state-based default tasks. Mean $8$ seeds at 1M steps, $\pm$ std across seeds, $\Delta$ the gap to the un-ablated method.}
\label{tab:targetensemble}
\end{table}

\paragraph{How does ALBUM compare with action-chunking methods?}
Table~\ref{tab:qc} compares ALBUM with QC-FQL~\citep{li2026reinforcement}, DEAS~\citep{kim2026deas}, DQC~\citep{li2026decoupled}, CGQ~\citep{song2026chunkguided}, and NFQL, an $n$-step variant of FQL evaluated in CGQ. ALBUM attains the highest score on \texttt{antmaze-large-navigate} and \texttt{antmaze-giant-navigate}, exceeding the best baseline by 10.8 and 16.4 points, respectively. Its average of 37.5 is higher than those of NFQL and QC-FQL and lower than those of FQL, DEAS, CGQ, and DQC. Full results are reported in Appendix~\ref{app:qc}.

\paragraph{How much does ALBUM reduce computational cost?}
We compare the training cost per gradient step of ALBUM with ReBRAC, IQL, FQL,
NFQL, QC-FQL, DQC, and CGQ, each in its default configuration
(Table~\ref{tab:cost-antmaze}). These methods hold four to eight value networks,
including critic ensembles and target copies, whereas ALBUM holds two. ALBUM
also uses 2.43M parameters and 78\,MiB of peak memory, the fewest among all
measured methods. Compared with FQL, it reduces the parameters by 50\%, the
peak memory by 39\%, and the time per step by 44\%. The reduction is larger
against action-chunking methods, reaching 76\%, 71\%, and 74\% against CGQ.
With the actor updated at every step for all methods, ALBUM is the fastest,
although ReBRAC with its default delayed actor update remains 15\% faster. See Appendix~\ref{app:cost}.

\paragraph{What happens when the coefficient condition is violated?}
Proposition~\ref{prop:bounded} gives a necessary condition for the update to
avoid an unbounded downward drift, and we enforce \eqref{eq:bounded-sg} when
choosing the penalty coefficients. To verify its necessity in practice,
we set $\lambda_{\mathrm K}=0$ on
\texttt{antmaze-large-navigate}, yielding
$\lambda_{\mathrm B}(1-\gamma)+2\lambda_{\mathrm K}=0.013<0.15
=\omega_{\mathrm Q}+\omega_{\mathrm V}$. The resulting $V_{\min}$ is
driven to the order of $-10^6$, far outside the physical range
$[-1/(1-\gamma),1/(1-\gamma)]=[-200,200]$, confirming the unbounded
regime predicted by Proposition~\ref{prop:bounded}. The task also has near-zero success and only $1\%$ of samples satisfy $g_{\mathrm B}\leq0$.
Full results are reported in Appendix~\ref{app:boundness}.

\paragraph{Does the $K$-step term behave as predicted?}
Theorem~\ref{thm:speedup} shows that adding the $K$-step term does not slow
the idealized iteration and that, where the $K$-step term is selected, the
error contracts at $\gamma^K$. Appendix~\ref{app:kspeed} confirms both in a
tabular setting. Under neural parameterization, a longer horizon accelerates
learning on \texttt{antmaze-large-navigate}, with $K=20$ reaching $50\%$
success at $300$k steps against $500$k for $K=5$. On
\texttt{puzzle-3x3-play} a longer horizon instead lowers the final success
rate, which the idealized iteration does not predict. Removing the $K$-step
term entirely fails on every task. Full results are reported in
Appendix~\ref{ablation:K}.

\paragraph{Does the coefficient condition alone recover performance?}
The condition can be satisfied without the $K$-step constraint by increasing the
one-step penalties. We set $\lambda_{\mathrm K} = 0$ and increase
$\lambda_{\mathrm B}=\lambda_{\mathrm E}$ well beyond the value
$(\omega_Q+\omega_V)/(1-\gamma)$ required by \eqref{eq:bounded-sg}. The values no longer diverge, but the best $\lambda_{\mathrm K}=0$ setting achieves only $4.2\%$, $2.0\%$, and $0.5\%$
success on \texttt{antmaze-large-navigate}, \texttt{scene-play}, and
\texttt{puzzle-3x3-play}, respectively, compared with $92.5\%$, $93.2\%$, and
$93.0\%$ with $\lambda_{\mathrm K}=1$. Thus, the $K$-step constraint contributes
beyond satisfying the condition. Value heatmaps on
\texttt{antmaze-large-navigate} (Figure~\ref{fig:propagate}) indicate faster
value propagation with the $K$-step constraint (Appendix~\ref{riselk}).


\paragraph{Does the ALBUM critic improvement transfer to a generative policy?}
The main experiments fix a Gaussian policy so that the critic is the only
change relative to the baselines of the same policy class in Table~\ref{tab:main}.
To test the critic under an expressive policy, we replace the IQL critic of
IFQL~\citep{park2025flow} with ALBUM without any tuning. The five-task average
rises from 30.4 to 68.8 on \texttt{scene-play} and from 19.0 to 72.3 on
\texttt{puzzle-3x3-play}, above both FQL and ALBUM in Table~\ref{tab:main},
and stays near its original value on \texttt{antmaze-large-navigate}. The gains
concentrate on the tasks IFQL fails to solve, four of which move from 0 to
between 54.0 and 97.0. (Appendix~\ref{app:album-ifql}).

%% file: contents/07_concl.tex
\section{Conclusion}

In this work, we show that in-sample Bellman optimality can be imposed on an offline critic through inequality constraints, rather than through regression. LBLP realizes this principle over the joint $(Q,V)$ space. With positive weight on both $Q$ and $V$, its unique minimizer is the in-sample optimal pair, and $K$-step constraints along dataset trajectories leave this minimizer unchanged. Under deterministic dynamics, this minimizer lies between the best dataset return and the optimal value. The tabular hinge relaxation of LBLP recovers the same solution above a finite penalty coefficient. ALBUM implements this relaxed formulation with neural networks, and we derive a necessary condition on its coefficients for bounded updates that extends to the neural parameterization. The ALBUM critic contains no squared regression, so it trains without target networks or EMA updates, and under deterministic dynamics its $K$-step penalties use multi-step dataset returns as lower bounds without off-policy correction or action chunking. With a single Q and V network and a Gaussian actor, ALBUM matches the average performance of FQL on OGBench, substantially outperforms ReBRAC, with the fewest parameters and the least peak GPU memory among all measured methods. Establishing convergence of the neural update and extending the analysis to stochastic transitions remain future work.

\section*{AI Use Statement}

In this work, we used generative AI tools to assist in formulating and refining some mathematical claims, to suggest proof ingredients and proof sketches for some results, to implement the proposed methods, and to check the proofs as a supplement to the authors' own verification. The theoretical framework, the formulation of LBLP, and the design of the experiments were developed by the authors. We have not used generative AI tools to design research methodology or experiments or to interpret experimental results, and generating synthetic data, cleaning or reformatting datasets, and qualitative or thematic data analysis are not applicable to this work. Additionally, we used generative AI tools for translation and to edit parts of the manuscript for readability and organization. We have reviewed all AI-assisted work. Every proof was checked by at least two authors. We take responsibility for the final content of this work, including all mathematical claims, proofs, implementations, and experimental results.

\section*{Reproducibility statement}

The reference ALBUM implementation is provided as supplementary material with pinned dependency versions and is written in JAX. Algorithm~\ref{alg:album} summarizes the training procedure, with the critic loss and actor objective given in~\eqref{eq:album-loss} and~\eqref{eq:album-actor-loss}, respectively. The hyperparameters shared across all environments are listed in Table~\ref{tab:hyperparams}, the environment-specific behavioral cloning coefficients are reported in Table~\ref{tab:alpha}, and the evaluation protocol, including the number of random seeds and evaluation episodes, is described in Section~\ref{sec:experiments}. The OGBench datasets are publicly available, and the baseline scores are taken from~\citet{park2025flow} and~\citet{song2026chunkguided}. The theoretical results are stated with their assumptions and accompanied by complete proofs in Appendix~\ref{app:proofs}. For the Gridworld experiments, we provide the environment configuration, dataset collection procedure, exact solution, and computational procedure to facilitate reproduction in Appendix~\ref{app:toy}.

%% file: contents/XX_appen.tex
\newpage

\appendix

\section{Limitations}
\label{app:limitations}

Lemma~\ref{lem:kstep-redundant}, Proposition~\ref{prop:lblp-exactness},
Proposition~\ref{prop:exactness}, and Proposition~\ref{prop:bounded} hold for
any finite discounted MDP, and the remaining theoretical results assume
deterministic transitions. The assumption enters at three points.
Proposition~\ref{prop:sandwich} identifies $\mathcal M_{\mathcal D}$ with
$\mathcal M$ on $\mathcal X_{\mathcal D}$, Lemma~\ref{lem:propagation} fixes
one dataset trajectory per pair, and Remark~\ref{rem:path-redundant} replaces
the expectation $y_K$ by a single trajectory. The last one carries the sharpest
consequence. Under stochastic transitions a sampled path is not a lower bound
on $Q^{*}_{\mathcal D}$, so $\hat y_K$ can exceed it and the $K$-step hinges no
longer provide lower bounds.

Under continuous states, each state of $\mathcal{D}$ appears once almost surely,
so $|\mathcal{A}_{\mathcal{D}}(s)| = 1$ for every $s \in \mathcal{S}_{\mathcal{D}}$.
The interval in Proposition~\ref{prop:sandwich} collapses because
$G_{\mathcal{D}} = Q^*_{\mathcal{D}}$, and the LBLP solution reduces to the
discounted return along each dataset trajectory. The same condition gives
$\delta_K = 0$ on $\mathcal{X}_{\mathcal{D}}$, so the stationary point of
Corollary~\ref{cor:single-action} is this trajectory-wise return.
Proposition~\ref{prop:lblp-exactness}, Proposition~\ref{prop:sandwich}, and
Corollary~\ref{cor:single-action} therefore characterize the in-sample
solution within $\mathcal{M}_{\mathcal{D}}$. The collapse follows from
$|\mathcal{A}_{\mathcal{D}}(s)| = 1$ alone and is independent of the LP
formulation, so the same collapse applies to methods that explicitly target
an in-sample optimum \citep{kostrikov2021offline, garg2023extreme,
xu2023offline}.

The results divide into two groups according to whether the neural
parameterization affects them. Proposition~\ref{prop:lblp-exactness},
Proposition~\ref{prop:exactness}, and Proposition~\ref{prop:sandwich}
characterize the in-sample solution underlying the critic loss in
\eqref{eq:album-loss}. These results describe the solution of the program on
$\mathcal{M}_{\mathcal{D}}$ rather than the values the neural critic attains. The mean of the learned critic exceeds the mean of
$V^*_{\mathcal{D}}$ in nine of the ten environments
(Appendix~\ref{app:vstar}). The remaining properties concern the structure of
the implemented loss and update. Every critic term in
\eqref{eq:album-loss} is evaluated at a pair in $\mathcal{D}$, so critic
training never evaluates $Q_\theta$ at an action absent from the data. The
actor objective in \eqref{eq:album-actor-loss} evaluates $Q_\theta$ at actions
outside $\mathcal{D}$, and the behavioral cloning term restricts these
evaluations to a neighborhood of the data. Under deterministic dynamics, the
one-step hinges imply the sampled $K$-step hinges
(Remark~\ref{rem:path-redundant}), so the $K$-step term changes neither the
constraint set nor the in-sample solution targeted by the loss, for any $K$.
Proposition~\ref{prop:bounded} concerns the update along a common shift of all
outputs, which any network with a bias in its final layer admits, so the
condition \eqref{eq:bounded-sg} governs the neural update directly, and
Appendix~\ref{app:boundness} confirms the predicted divergence when it is
violated. A full characterization of the learned values under neural
parameterization remains open.

Theorem~\ref{thm:speedup} concerns an idealized coordinatewise iteration rather
than the ALBUM update, and Appendix~\ref{app:stationary} establishes the
existence of a zero of the detached update rather than uniqueness or
convergence. The coefficient conditions separate divergence from boundedness
and do not separate the settings that perform well from those that do not, as
$\lambda_{\mathrm B}=10$ satisfies them while failing on every task
(Appendix~\ref{sec:abl-lambda}). Extending the analysis to stochastic
transitions and to an empirical MDP with branching action sets remains future
work.

\newpage
\section{Algorithm Pseudocode}
\label{app:pseudocode}

\begin{algorithm}[h]
\caption{Approximate Lifted Bellman Unconstrained Minimization (ALBUM)}
\label{alg:album}
\begin{algorithmic}[1]
\REQUIRE Offline dataset \(\mathcal D\), mini-batch size \(B\), rollout length
\(K\), discount factor \(\gamma\)
\REQUIRE Objective weights \(\omega_Q,\omega_V\), penalty weights
\(\lambda_{\mathrm B},\lambda_{\mathrm E},\lambda_{\mathrm K}\), BC weight
\(\alpha_{\mathrm{BC}}\)
\REQUIRE Learning rates \(\eta_{\mathrm{critic}},\eta_{\mathrm{actor}}\)
\STATE Initialize \(Q_\theta\), \(V_\phi\), and \(\pi_\psi\)
\WHILE{not converged}
    \STATE Sample \(B\) subsequences of length \(K\) from \(\mathcal D\)
    \STATE Compute \( \hat y_1\) and \(\hat y_K\).
    \STATE Compute \(\mathcal L(\theta,\phi)\) using Eq.~\ref{eq:album-loss}
    \STATE \(\theta\leftarrow\theta-\eta_{\mathrm{critic}}\nabla_\theta\mathcal L\),
           \(\phi\leftarrow\phi-\eta_{\mathrm{critic}}\nabla_\phi\mathcal L\)
    \STATE \(\psi\leftarrow\psi-\eta_{\mathrm{actor}}\nabla_\psi\mathcal L_{\mathrm{actor}}\)
           using Eq.~\ref{eq:album-actor-loss}
\ENDWHILE
\RETURN \(Q_\theta,V_\phi,\pi_\psi\)
\end{algorithmic}
\end{algorithm}

\section{Offline Dataset and Empirical MDP}
\label{app:empirical-mdp}
To discuss LBLP in the offline setting, we define the empirical MDP induced by the dataset, following~\citet{fujimoto2019off, xiao2023sample}, where we adopt our own notation to explain the construction more clearly. Let $\mathcal{M} = (\mathcal S,\mathcal A,P,R,\gamma)$ be the underlying MDP and $\mathcal D = \{(s_i,a_i,r_i,s_i')\}_{i=1}^{m}$ be a fixed offline dataset of $m$ transitions collected from $\mathcal M$; every object below is computed from $\mathcal D$ alone, with no further interaction with $\mathcal M$. Writing $n_{\mathcal D}(s,a) := \sum_{i=1}^{m} \mathbf 1\{(s_i,a_i)=(s,a)\}$ for the number of times a pair appears in the data, the pairs seen at least once form the empirical support
\[ \mathcal X_{\mathcal D} := \{(s,a): n_{\mathcal D}(s,a)>0\}, \]
with state marginal and action slices
\[ \mathcal S_{\mathcal D} = \{s:\exists a,\,(s,a) \in \mathcal X_{\mathcal D}\}, \qquad \mathcal A_{\mathcal D}(s) = \{a:(s,a) \in \mathcal X_{\mathcal D}\}, \]
so that $\mathcal X_{\mathcal D} = \{(s,a): s \in \mathcal S_{\mathcal D},\, a \in \mathcal A_{\mathcal D}(s)\}$.
All value functions below are vectors in $\mathbb R^{|\mathcal X_{\mathcal D}|}$,
i.e.\ they are defined only on pairs the dataset has actually seen. Normalizing the
counts gives the empirical sampling distribution
$ d_{\mathcal D}(s,a):=n_{\mathcal D}(s,a)/m$, a probability distribution
supported exactly on $\mathcal X_{\mathcal D}$, which records how strongly each pair
is weighted in the empirical objective; we write
$d_{\min}:=\min_{\mathcal X_{\mathcal D}} d_{\mathcal D}$,
$d_{\max}:=\max_{\mathcal X_{\mathcal D}} d_{\mathcal D}$, and
$ D_{\mathcal D}:=\operatorname{diag}( d_{\mathcal D}(s,a))_{(s,a)\in\mathcal X_{\mathcal D}}$.
Averaging the transitions that start at each $(s,a)\in\mathcal X_{\mathcal D}$ yields
the empirical reward and transition kernel
\[
 R_{\mathcal D}(s,a):=\frac{1}{n_{\mathcal D}(s,a)}\!\!\sum_{i:(s_i,a_i)=(s,a)}\!\!r_i,
\qquad
 P_{\mathcal D}(s'\mid s,a):=\frac{1}{n_{\mathcal D}(s,a)}\!\!\sum_{i:(s_i,a_i)=(s,a)}\!\!\mathbf 1\{s_i'=s'\}.
\]

The successors at which no action of $\mathcal{D}$ appears form $\mathcal{S}^0_\mathcal{D} := \{s'_i : i\in[m]\}\setminus\mathcal{S}_\mathcal{D}$, the states that the data reaches but never leaves. We treat such states as absorbing with zero reward, so that
$\mathcal A_{\mathcal D}(s) = \emptyset$ and
\begin{equation} \label{T}
    \max_{a' \in \mathcal A_{\mathcal D}(s)} Q(s,a') := 0 \quad \text{whenever } \mathcal A_{\mathcal D}(s) = \emptyset, \qquad \text{equivalently } V(s) = 0.
\end{equation}
The resulting empirical MDP is ${\mathcal M}_{\mathcal D} = (\mathcal S_{\mathcal D} \cup \mathcal S^{0}_{\mathcal D}, \{\mathcal A_{\mathcal D}(s)\}_s,  P_{\mathcal D},  R_{\mathcal D}, \gamma)$,
and its Bellman optimality operator
${\mathcal T}_{\mathcal D}: \mathbb R^{|\mathcal X_{\mathcal D}|} \to \mathbb R^{|\mathcal X_{\mathcal D}|}$
is
\[ \left({\mathcal T}_{\mathcal D} Q\right)(s,a) =  R_{\mathcal D}(s,a) + \gamma \mathbb E_{s' \sim  P_{\mathcal D}(\cdot \mid s,a)} \left[\max_{a' \in \mathcal A_{\mathcal D}(s')} Q(s',a')\right], \qquad (s,a) \in \mathcal X_{\mathcal D}, \]
with~\eqref{T} applied at terminal successors. This is the ordinary Bellman optimality operator subject to two restrictions,
in that the transition kernel and reward are the empirical ones and the inner
maximization ranges over the actions the dataset contains.

This definition is well posed because \({\mathcal M}_{\mathcal D}\) is closed
under its own transitions. For every \((s,a)\in\mathcal X_{\mathcal D}\) and
every \(s'\) with \(P_{\mathcal D}(s'\mid s,a)>0\), the successor \(s'\) lies
in \(\mathcal S_{\mathcal D}\) or in \(\mathcal S^{0}_{\mathcal D}\). Indeed, a
positive empirical transition probability means the transition
\((s,a)\to s'\) appears in \(\mathcal D\), so \(s'\) either carries an action
of its own or carries none at all. In the first case the inner maximum is
taken over the nonempty set \(\mathcal A_{\mathcal D}(s')\) and reads only
entries of \(Q\) indexed by \(\mathcal X_{\mathcal D}\), and in the second it
equals \(0\) by~\eqref{T}. Hence \({\mathcal T}_{\mathcal D}\) never
queries \(Q\) outside \(\mathcal X_{\mathcal D}\), it is a \(\gamma\)-contraction
in \(\|\cdot\|_\infty\), and it admits a unique fixed point \(Q^{*}_{\mathcal D}\). This fixed point is
not a new object, as prior work states the same in-sample Bellman optimality
equation \citep{xiao2023sample}.

\section{Proofs}
\label{app:proofs}
Lemma~\ref{lem:kstep-redundant}, Proposition~\ref{prop:lblp-exactness}, Proposition~\ref{prop:exactness},  and Proposition~\ref{prop:bounded} hold for any finite discounted MDP, and the remaining results assume deterministic dynamics.

\subsection{Proof of Lemma~\ref{lem:kstep-redundant}}
\label{app:kstep-redundant}
\begin{proof}
Fix $(s,a)\in\mathcal S\times\mathcal A$. Consider a $\beta$-rollout with
$(s_0,a_0)=(s,a)$, so that $s_{t+1}\sim P(\cdot\mid s_t,a_t)$ and
$a_{t+1}\sim\beta(\cdot\mid s_{t+1})$ for $t\ge0$. From $g_{\mathrm E}\le0$ we
have $V(s_{t+1})\ge Q(s_{t+1},a_{t+1})$ for every $a_{t+1}$, and averaging over
$a_{t+1}\sim\beta(\cdot\mid s_{t+1})$ preserves the inequality,
\begin{equation}
\label{eq:epi-averaged}
V(s_{t+1})
\;\ge\;
\mathbb E_{a_{t+1}\sim\beta(\cdot\mid s_{t+1})}
\!\left[Q(s_{t+1},a_{t+1})\right].
\end{equation}
Combining $g_{\mathrm B}\le0$ with \eqref{eq:epi-averaged}, we have
\begin{equation}
\label{eq:one-hop}
\begin{aligned}
Q(s_t,a_t)
&\;\ge\;
R(s_t,a_t)+\gamma\,\mathbb E_{\mathcal M}\!\left[V(s_{t+1})\mid s_t,a_t\right]
\\
&\;\ge\;
R(s_t,a_t)
+\gamma\,\mathbb E_{\mathcal M,\beta}\!\left[Q(s_{t+1},a_{t+1})\mid s_t,a_t\right].
\end{aligned}
\end{equation}
Applying \eqref{eq:one-hop} at $t=0$ and then taking the expectation of
\eqref{eq:one-hop} at $t=1,\dots,K-2$ under $\mathbb E_{\mathcal M,\beta}
[\,\cdot\mid s,a]$, and finally $g_{\mathrm B}\le0$ at $t=K-1$,
\begin{align}
Q(s,a)
&\;\ge\;
R(s,a)+\gamma\,\mathbb E_{\mathcal M,\beta}\!\left[Q(s_1,a_1)\mid s,a\right]
\label{eq:unroll-1}\\[2pt]
&\;\ge\;
\mathbb E_{\mathcal M,\beta}\!\left[
\sum_{t=0}^{1}\gamma^{t}R(s_t,a_t)
+\gamma^{2}Q(s_2,a_2)\;\Big|\;s,a\right]
\label{eq:unroll-2}\\[2pt]
&\;\;\;\vdots \nonumber\\[2pt]
&\;\ge\;
\mathbb E_{\mathcal M,\beta}\!\left[
\sum_{t=0}^{K-2}\gamma^{t}R(s_t,a_t)
+\gamma^{K-1}Q(s_{K-1},a_{K-1})\;\Big|\;s,a\right]
\label{eq:unroll-k1}\\[2pt]
&\;\ge\;
\mathbb E_{\mathcal M,\beta}\!\left[
\sum_{t=0}^{K-1}\gamma^{t}R(s_t,a_t)
+\gamma^{K}V(s_K)\;\Big|\;s,a\right]
\;=\;y_K(s,a),
\label{eq:unroll-final}
\end{align}
which implies $g_{\mathrm{KQ}}(Q,V)(s,a)\le0$. 

Applying $g_{\mathrm E}\le0$
at $(s,a)$ we have
\begin{equation}
V(s)\;\ge\;Q(s,a)\;\ge\;y_K(s,a),
\end{equation}
which implies $g_{\mathrm{KV}}(Q,V)(s,a)\le0$.
\end{proof}

\begin{remark}
\label{rem:path-redundant}
Suppose the dynamics are deterministic and let
$(s_0,a_0),\dots,(s_{K-1},a_{K-1}),s_K$ be a segment of a dataset trajectory, so
that every $(s_t,a_t)$ lies in $\mathcal X_{\mathcal D}$ and $s_{t+1}$ is the
successor of $(s_t,a_t)$. If $(Q,V)$ satisfies $g_{\mathrm B}\le0$ and
$g_{\mathrm E}\le0$ on $\mathcal M_{\mathcal D}$, then
\begin{equation*}
Q(s_t,a_t)\ge r_t+\gamma V(s_{t+1})\ge r_t+\gamma Q(s_{t+1},a_{t+1}),
\end{equation*}
and chaining these inequalities as in \eqref{eq:unroll-1}--\eqref{eq:unroll-final}
gives $\sum_{t<K}\gamma^t r_t+\gamma^K V(s_K)\le Q(s_0,a_0)\le V(s_0)$. Hence
the sampled rollout constraints are implied by the one-step constraints and
leave the solution unchanged.
\end{remark}

\subsection{Proof of Proposition~\ref{prop:lblp-exactness}}
\label{app:lblp-exactness}
\begin{proof}
We first check that $(Q^*,V^*)$ is feasible. Bellman optimality gives
$Q^* = R + \gamma PV^*$, so $g_{\mathrm{B}}(Q^*,V^*) = 0$, and
$V^*(s) = \max_{a} Q^*(s,a) \ge Q^*(s,a)$ gives $g_{\mathrm{E}}(Q^*,V^*) \le 0$.
The remaining constraints $g_{\mathrm{KQ}}(Q^*,V^*) \le 0$ and
$g_{\mathrm{KV}}(Q^*,V^*) \le 0$ hold automatically by
Lemma~\ref{lem:kstep-redundant}. Hence the $(Q^*, V^*)$ is feasible.

Now we check that $(Q^*,V^*)$ is the unique minimizer. Let $(Q,V)$ be an arbitrary
feasible point. From $g_B(Q, V) \le 0$ and $g_E(Q, V) \le 0$, we have
\[ V(s) \ge Q(s,a) \ge R(s,a) + \gamma (PV)(s,a),
\qquad \forall (s,a)\in\mathcal S\times\mathcal A. \]

Taking the maximum over $a$ on
the right, we have $V \ge \mathcal{B}V$ where the Bellman optimality operator on
state-value functions is defined as
$(\mathcal{B}V)(s) := \max_{a \in \mathcal{A}}[R(s,a) + \gamma (PV)(s,a)]$.
Iterating this inequality and using that $\mathcal{B}$ is monotone and a
$\gamma$-contraction, one gets $V\ge V^*$~\citep{puterman2014markov}. Moreover,
by using the monotonicity of $P$, we get
$Q \ge R +\gamma PV \ge R + \gamma PV^* = Q^*$. Therefore, 
\[ V(s) \ge V^*(s) \quad \forall s \in \mathcal{S},
\qquad\qquad
Q(s,a) \ge Q^*(s,a) \quad \forall (s,a) \in \mathcal{S}\times\mathcal{A}. \]

Since $c_Q, c_V > 0$,
\[ F(Q,V)-F(Q^*,V^*) =\langle c_Q, Q-Q^* \rangle + \langle c_V, V-V^* \rangle \ge 0, \]
which implies that equality holds if and only if $(Q,V)=(Q^*,V^*)$. Hence
$(Q^*,V^*)$ is the unique minimizer.
\end{proof}

\subsection{Proof of Proposition~\ref{prop:sandwich}}
\label{app:proof-sandwich}

Throughout, write \(s'\) for the successor of \((s,a)\in\mathcal X_{\mathcal D}\),
which is unique by determinism and agrees with the successor stored in
\(\mathcal D\). By Proposition~\ref{prop:lblp-exactness} the minimizer
\((Q^*_{\mathcal D},V^*_{\mathcal D})\) is the Bellman-optimal pair of
\({\mathcal M}_{\mathcal D}\), so it satisfies
\begin{equation}\label{eq:emp-bellman}
Q^*_{\mathcal D}(s,a)=\ R_{\mathcal D}(s,a)+\gamma V^*_{\mathcal D}(s'),
\qquad
V^*_{\mathcal D}(s)=\max_{a\in\mathcal A_{\mathcal D}(s)}Q^*_{\mathcal D}(s,a),
\end{equation}
for every \((s,a)\in\mathcal X_{\mathcal D}\), with the convention
\(V^*_{\mathcal D}(s)=0\) at states carrying no action of
\(\mathcal A_{\mathcal D}\). By hypothesis such states are terminal in
\(\mathcal M\), so \(V^{*}\) vanishes there.

\paragraph{Lower bound.}
Fix \((s,a)\in\mathcal X_{\mathcal D}\) and let
\((s_0,a_0),(s_1,a_1),\dots\) be a continuation of \((s,a)\), with
\(r_t=\ R_{\mathcal D}(s_t,a_t)\). Every pair of the continuation lies in
\(\mathcal X_{\mathcal D}\) and \(a_{t+1}\in\mathcal A_{\mathcal D}(s_{t+1})\),
so the second identity of \eqref{eq:emp-bellman} gives
\(V^*_{\mathcal D}(s_{t+1})\ge Q^*_{\mathcal D}(s_{t+1},a_{t+1})\). Combining this
with the first identity,
\[
Q^*_{\mathcal D}(s_t,a_t)\;\ge\;r_t+\gamma\,Q^*_{\mathcal D}(s_{t+1},a_{t+1}).
\]
Let $s_T$ be the final state of the continuation. By hypothesis $\mathcal{A}_\mathcal{D}(s_T)=\emptyset$, so $V^*_\mathcal{D}(s_T)=0$ and $Q^*_\mathcal{D}(s_{T-1},a_{T-1}) = r_{T-1}$. Unrolling the inequality for $t=0,\dots,T-2$ yields
\begin{equation*}
Q^*_\mathcal{D}(s,a) \ge \sum_{t<T}\gamma^t r_t = G(\tau),
\end{equation*}
and taking the maximum over $\tau\in\mathcal{H}_\mathcal{D}(s,a)$ gives $G_\mathcal{D}(s,a)\le Q^*_\mathcal{D}(s,a)$.

\paragraph{Upper bound.}
Let \({\mathcal T}_{\mathcal D}\) be the Bellman optimality operator of
\({\mathcal M}_{\mathcal D}\) and let \(Q^{*}\big|_{\mathcal X_{\mathcal D}}\)
denote the restriction of \(Q^{*}\) to \(\mathcal X_{\mathcal D}\). For
\((s,a)\in\mathcal X_{\mathcal D}\), determinism and the agreement of
\( R_{\mathcal D}\) and \( P_{\mathcal D}\) with \(R\) and \(P\)
on \(\mathcal X_{\mathcal D}\) give
\[
\big({\mathcal T}_{\mathcal D}Q^{*}\big|_{\mathcal X_{\mathcal D}}\big)(s,a)
= R(s,a)+\gamma\max_{a'\in\mathcal A_{\mathcal D}(s')}Q^{*}(s',a')
\;\le\;
R(s,a)+\gamma\max_{a'\in\mathcal A}Q^{*}(s',a')
= Q^{*}(s,a),
\]
where the inequality uses \(\mathcal A_{\mathcal D}(s')\subseteq\mathcal A\),
and holds at a terminal successor as well, since both sides are then
\(R(s,a)\). Thus \(Q^{*}\big|_{\mathcal X_{\mathcal D}}\) satisfies
\({\mathcal T}_{\mathcal D}Q^{*}\big|_{\mathcal X_{\mathcal D}}
\le Q^{*}\big|_{\mathcal X_{\mathcal D}}\). Since
\({\mathcal T}_{\mathcal D}\) is monotone, iterating preserves the
inequality, and since it is a \(\gamma\)-contraction with fixed point
\(Q^*_{\mathcal D}\), the iterates converge to \(Q^*_{\mathcal D}\). Hence
\(Q^*_{\mathcal D}\le Q^{*}\) on \(\mathcal X_{\mathcal D}\), which is the right
inequality.

\subsection{Variants of the sandwich property}
\label{app:sandwich-variants}

This section collects two variants of Proposition~\ref{prop:sandwich} that
were stated informally in Section~\ref{sec:offline-lblp}. The first replaces
the recorded return by the value of the behavior policy under a coverage
assumption, and the second transfers both bounds to the state-value function $V$.

\begin{corollary}[Sandwich property under full behavioral coverage]
\label{cor:sandwich-beta}
Assume in addition that
\(\operatorname{supp}\beta(\cdot\mid s)\subseteq\mathcal A_{\mathcal D}(s)\)
for every \(s\in\mathcal S_{\mathcal D}\), so that \(\beta\) is a policy of
\({\mathcal M}_{\mathcal D}\). Then for every
\((s,a)\in\mathcal X_{\mathcal D}\),
\[
Q^{\beta}(s,a)
\;\le\;
Q^{*}_{\mathcal D}(s,a)
\;\le\;
Q^{*}(s,a),
\]
where \(Q^{\beta}\) denotes the value of the behavior policy in \(\mathcal M\).
\end{corollary}

\begin{corollary}[Sandwich property for the state-value function]
\label{cor:sandwich-v}
Under the assumptions of Proposition~\ref{prop:sandwich}, for every
\(s\in\mathcal S_{\mathcal D}\) we have
\(\max_{a\in\mathcal A_{\mathcal D}(s)}G_{\mathcal D}(s,a)\le V^{*}_{\mathcal D}(s)
\le V^{*}(s)\), and under the additional assumption of
Corollary~\ref{cor:sandwich-beta} we have
\(V^{\beta}(s)\le V^{*}_{\mathcal D}(s)\le V^{*}(s)\).
\end{corollary}

\subsection{Propagation Identity}
\label{app:propagation}

The following lemma is used in the proof of Theorem~\ref{thm:speedup}.

\begin{lemma}[Propagation identity] \label{lem:propagation}
Suppose the transition dynamics of $\mathcal M$ are deterministic. Under the
empirical MDP ${\mathcal M}_{\mathcal D}$, fix, for each $(s,a) \in \mathcal
X_{\mathcal D}$, a recorded trajectory starting from $(s,a)$, and let
$s^{(k)}(s,a)$ be the state reached after $k$ transitions along this trajectory
while $G_k(s,a) = \sum_{t<k} \gamma^t r_t$ collects the discounted rewards along
the way. The statement below holds for any such choice. Let us define the
horizon-$k$ target
\begin{equation}
    L_k[V](s,a) := G_k(s,a) + \gamma^k V \big(s^{(k)}(s,a)\big),
\end{equation}
and the block coordinate updates 
\begin{equation} \label{eq:repair2}
    Q_{n+1} := \max_{k \in \{1, K\}} L_k[V_n], \qquad V_{n+1}(s) := \max_{a \in \mathcal A_{\mathcal D}(s)} Q_{n+1}(s,a).
\end{equation}
Moreover, define the suboptimality of the realized $k$-step path and the current value gap
\begin{equation} \label{eq:delta}
\delta_k(s,a) := Q^*_{\mathcal D}(s,a) - L_k[V^*_{\mathcal D}](s,a),
\qquad
e_n := V^*_{\mathcal D} - V_n .
\end{equation}
Then, $\delta_K \ge 0$ with $\delta_1 = 0$. Moreover, if $Q_0 \le  Q^*_{\mathcal D}$ and $V_0 \le  V^*_{\mathcal D}$, the iterates remain underestimates, $Q_n \le  Q^*_{\mathcal D}$ and $V_n \le  V^*_{\mathcal D}$ for all $n$, so that $e_n \ge 0$. In addition, for every $(s,a) \in \mathcal X_{\mathcal D}$,
\begin{equation}\label{eq:propagation-rate}
    \big( Q^*_{\mathcal D} - Q_{n+1}\big)(s,a) = \min_{k \in \{1, K\}} \Big[\delta_k(s,a) + \gamma^k e_n \big(s^{(k)}(s,a)\big)\Big].
\end{equation}
\end{lemma}

\begin{proof}
We first show $\delta_k \ge 0$ with equality at $k = 1$. Write $(s_0,a_0)=(s,a)$ and let $s_1,a_1,\dots$ be the rest of the recorded trajectory so that $s_k = s^{(k)}(s,a)$. By Proposition~\ref{prop:lblp-exactness} applied to the finite discounted MDP
$\mathcal M_{\mathcal D}$, $(Q^*_{\mathcal D}, V^*_{\mathcal D})$ is the
Bellman-optimal pair of $\mathcal M_{\mathcal D}$, so
\[  Q^*_{\mathcal D}(s_t,a_t) = r_t + \gamma  V^*_{\mathcal D}(s_{t+1}), \qquad  V^*_{\mathcal D}(s_t) = \max_{a\in\mathcal A_{\mathcal D}(s_t)}  Q^*_{\mathcal D}(s_t,a) \ge  Q^*_{\mathcal D}(s_t,a_t), \]
where the first because deterministic dynamics give $(s_t,a_t)$ a unique next state, and the second with equality exactly when the recorded action $a_t$ is optimal at $s_t$. Alternating the two,
\begin{align*}
     Q^*_{\mathcal D}(s_0,a_0)
    =& r_0 + \gamma  V^*_{\mathcal D}(s_1) \\
    \ge& r_0 + \gamma  Q^*_{\mathcal D}(s_1,a_1) \\
    =& r_0 + \gamma r_1 + \gamma^2  V^*_{\mathcal D}(s_2) \\
    \ge& \cdots \ge \sum_{t=0}^{k-1} \gamma^t r_t + \gamma^k  V^*_{\mathcal D}(s_k)
    =L_k[ V^*_{\mathcal D}](s,a),
\end{align*}
so $\delta_k \ge 0$ for $k > 1$. For $k=1$, the first line is an equality, and $\delta_1=0$.

For the second claim, expand $L_k[V_n]$ and substitute the definitions of $e_n$ and $\delta_k$:
\begin{align}
    L_k[V_n](s,a)
    =& G_k(s,a) + \gamma^k V_n\big(s^{(k)}(s,a)\big) \nonumber \\
    =& G_k(s,a) + \gamma^k  V^*_{\mathcal D}\big(s^{(k)}(s,a)\big) - \gamma^k e_n\big(s^{(k)}(s,a)\big) \nonumber \\
    =&  Q^*_{\mathcal D}(s,a) - \delta_k(s,a) - \gamma^k e_n\big(s^{(k)}(s,a)\big). \label{L_K_expansion}
\end{align}
Each horizon target therefore falls short of $ Q^*_{\mathcal D}(s,a)$ by
the amount $\delta_k(s,a)+\gamma^k e_n(s^{(k)}(s,a))$. We argue by induction that this shortfall is nonnegative. At $n=0$, the underestimate assumption $V_0\le  V^*_{\mathcal D}$ gives $e_0 \ge 0$. Assuming $e_n \ge 0$, $L_k[V_n] \le  Q^*_{\mathcal D}$ for every $k$ with $\delta_k \ge 0$. Taking the maximum over $k$ preserves the inequality, so $Q_{n+1} \le  Q^*_{\mathcal D}$. Moreover, taking the maximum over $a \in \mathcal A_{\mathcal D}(s)$ gives $V_{n+1} \le  V^*_{\mathcal D}$, i.e.\ $e_{n+1} \ge 0$. The underestimate property is thus preserved for all $n$.

For the final claim, we have
\begin{align*}
    \big( Q^*_{\mathcal D} - Q_{n+1}\big)(s,a)
    =&  Q^*_{\mathcal D}(s,a) - \max_{k \in \{1,K\}}L_k[V_n](s,a) \\
    =& \min_{k \in \{1,K\}} \Big[ Q^*_{\mathcal D}(s,a) - L_k[V_n](s,a)\Big] \\
    =& \min_{k \in \{1,K\}} \Big[\delta_k(s,a) + \gamma^k e_n\big(s^{(k)}(s,a)\big)\Big],
\end{align*}
which is \eqref{eq:propagation-rate}. The first equality is the definition of~\eqref{eq:repair2} and the last equality substitutes~\eqref{L_K_expansion}.
\end{proof}

Note that the underestimate initialization in Lemma~\ref{lem:propagation} is not restrictive in practice. When rewards lie in $[r_{\min},r_{\max}]$, the choice $V_0 \equiv r_{\min}/(1-\gamma)$ and $Q_0 \equiv r_{\min}/(1-\gamma)$ satisfies the hypothesis for any dataset, since no value in ${\mathcal M}_{\mathcal D}$ can fall below that bound.

\subsection{Proof of Theorem~\ref{thm:speedup}}
\label{app:speedup}

\begin{proof}
For the first claim, since the minimum in \eqref{eq:propagation-rate} cannot exceed any of its terms, it is bounded by the $k=1$ term, which equals $\gamma e_n(s^{(1)}(s,a))$ since $\delta_1=0$ by Lemma~\ref{lem:propagation}. Thus,
\[ \big( Q^*_{\mathcal D} - Q_{n+1}\big)(s,a) \le \gamma e_n\big(s^{(1)}(s,a)\big) \le \gamma \max_{s'} e_n(s') = \gamma\|e_n\|_\infty \qquad\text{for every }(s,a)\in\mathcal X_{\mathcal D}. \]
Taking maxima over $a\in\mathcal A_{\mathcal D}(s)$ and using $\max_a h_1(a) - \max_a h_2(a) \le \max_a (h_1(a) - h_2(a))$ for any two functions $h_1, h_2$,
\[ e_{n+1}(s) = \max_{a}  Q^*_{\mathcal D}(s,a) - \max_{a} Q_{n+1}(s,a) \le \max_{a} \big( Q^*_{\mathcal D} - Q_{n+1}\big)(s,a) \le \gamma\|e_n\|_\infty \]

for every $s$. Since $e_{n+1} \ge 0$ by Lemma~\ref{lem:propagation}, this gives
$\|e_{n+1}\|_\infty\le\gamma\|e_n\|_\infty$, hence
$\|e_n\|_\infty\le\gamma^{n}\|e_0\|_\infty\to0$, so $V_n\to V^*_{\mathcal D}$ and
$Q_{n+1}=\max_{k\in\{1,K\}}L_k[V_n]\to Q^*_{\mathcal D}$. Only the $k=1$ term
entered the argument, so the conclusion holds for any rollout length $K$.

For the second claim, including the $k=K$ term gives the minimum in \eqref{eq:propagation-rate} more terms to choose from, so at every coordinate the error can only get smaller, never larger. Suppose at some $(s,a)$,
\[ \delta_K(s,a) + \gamma^K e_n\big(s^{(K)}(s,a)\big) < \gamma e_n\big(s^{(1)}(s,a)\big). \]
Then, the $k=K$ term is strictly smaller than the $k=1$ term, so the error at $(s,a)$ is governed by $\delta_K(s,a) + \gamma^K e_n(s^{(K)}(s,a))$ rather than by the one-step term. If in addition every action in $\mathcal D$ along the $K$-step path is optimal in ${\mathcal M}_{\mathcal D}$, so that $\delta_K(s,a)=0$, the error at that coordinate contracts by $\gamma^K$
instead of $\gamma$.
\end{proof}

\subsection{Proof of Proposition~\ref{prop:exactness}}
\label{app:exactness}
\begin{proof}
Since $\mathcal L_{\mathrm{R}}$ replaces the constraints of LBLP by hinge penalties, the penalty weights must be large enough that violating is never profitable; otherwise the two problems would have different minimizers. Linear programming duality answers this: the optimal dual variables of the constrained program measure the gain from violating each constraint, and therefore define the optimal threshold. Let $\mathrm{LP}_1$ denote the relaxation of LBLP on ${\mathcal M}_{\mathcal D}$ obtained by keeping only the one-step and epigraph constraints $g_{\mathrm B}\le 0$ and $g_{\mathrm E}\le 0$. Since $g_{\mathrm{KQ}}$ and $g_{\mathrm{KV}}$ are redundant by Lemma~\ref{lem:kstep-redundant}, $\mathrm{LP}_1$ has the same feasible set as the LBLP on $\mathcal M_{\mathcal D}$, and as the two programs also share the objective $F$, they have the same minimizers. Proposition~\ref{prop:lblp-exactness}, applied to the finite discounted MDP $\mathcal M_{\mathcal D}$, identifies this common minimizer as the unique pair $(Q^*_{\mathcal D},V^*_{\mathcal D})$. Hence $\mathrm{LP}_1$ is feasible and bounded below by $p^* := F(Q^*_{\mathcal D},V^*_{\mathcal D})$, attained uniquely there. As $\mathrm{LP}_1$ is a feasible and bounded linear program, strong duality
holds, and its optimal dual variables for $g_{\mathrm B}$ and $g_{\mathrm E}$
are the $(\lambda^*_{\mathrm B},\lambda^*_{\mathrm E})$ of the statement,
indexed by $\mathcal X_{\mathcal D}$.

Next, we check that $( Q^*_{\mathcal D}, V^*_{\mathcal D})$ minimizes $\mathcal L_{\mathrm{R}}$. Consider~\eqref{eq:relaxed-lblp} as $\mathcal L_{\mathrm{R}} = \Phi_1 + \Phi_{\mathrm K}$, where
\[ \Phi_1(Q,V) := F(Q,V) + \lambda_{\mathrm B} \big\|[g_{\mathrm B}(Q,V)]_+\big\|_1 + \lambda_{\mathrm E} \big\|[g_{\mathrm E}(Q,V)]_+\big\|_1, \]
\[ \Phi_{\mathrm K}(Q,V) = \lambda_{\mathrm K}\Big(\big\|[g_{\mathrm{KQ}}(Q,V)]_+\big\|_1 + \big\|[g_{\mathrm{KV}}(Q,V)]_+\big\|_1\Big). \]
Writing $x=(Q,V)$ and $L(x,u,v) = F(x) + \langle u,g_{\mathrm B}(x)\rangle + \langle v,g_{\mathrm E}(x)\rangle$, strong duality gives $\min_x L(x,\lambda^*_{\mathrm B},\lambda^*_{\mathrm E}) = p^*$. Using the definition of $\|\cdot\|_1$, we get
\begin{align}
    \Phi_1(x)
    =& F(x) + \langle\lambda_{\mathrm B}\mathbf 1,[g_{\mathrm B}(x)]_+\rangle + \langle\lambda_{\mathrm E}\mathbf 1,[g_{\mathrm E}(x)]_+\rangle \nonumber \\
    =& F(x) + \langle\lambda^*_{\mathrm B},[g_{\mathrm B}(x)]_+\rangle + \langle\lambda^*_{\mathrm E},[g_{\mathrm E}(x)]_+\rangle \nonumber \\
    &+ \big\langle\lambda_{\mathrm B}\mathbf 1-\lambda^*_{\mathrm B},[g_{\mathrm B}(x)]_+\big\rangle + \big\langle\lambda_{\mathrm E}\mathbf 1-\lambda^*_{\mathrm E},[g_{\mathrm E}(x)]_+\big\rangle \nonumber \\
    \ge& L(x,\lambda^*_{\mathrm B},\lambda^*_{\mathrm E}) + \big\langle\lambda_{\mathrm B}\mathbf 1-\lambda^*_{\mathrm B},[g_{\mathrm B}(x)]_+\big\rangle + \big\langle\lambda_{\mathrm E}\mathbf 1-\lambda^*_{\mathrm E},[g_{\mathrm E}(x)]_+\big\rangle \nonumber \\
    \ge& p^* + \big\langle\lambda_{\mathrm B}\mathbf 1-\lambda^*_{\mathrm B},[g_{\mathrm B}(x)]_+\big\rangle + \big\langle\lambda_{\mathrm E}\mathbf 1-\lambda^*_{\mathrm E},[g_{\mathrm E}(x)]_+\big\rangle, \label{eq:key}
\end{align}
where the first inequality utilizes $g_{\mathrm B} \le [g_{\mathrm B}]_+$ and $g_{\mathrm E} \le [g_{\mathrm E}]_+$ together with $\lambda^*_{\mathrm B},\lambda^*_{\mathrm E} \ge 0$, and the last inequality comes from $\min_x L(x,\lambda^*_{\mathrm B},\lambda^*_{\mathrm E}) = p^*$. Since $( Q^*_{\mathcal D},  V^*_{\mathcal D})$ is feasible for $\mathrm{LP}_1$, both hinges in $\Phi_1$ vanish there. The rollout term $\Phi_{\mathrm K}$ vanishes as well, since by Lemma~\ref{lem:kstep-redundant}, the one-step and epigraph constraints imply $g_{\mathrm{KQ}}(x)\le 0$ and $g_{\mathrm{KV}}(x)\le 0$. Therefore, $\mathcal L_{\mathrm{R}}( Q^*_{\mathcal D}, V^*_{\mathcal D}) = p^*$ and \eqref{eq:key} shows it is a minimizer.

We now check that no other point can be a minimizer. If $x$ violates a one-step or epigraph constraint at some $(s,a)$, the corresponding hinge is strictly positive there and its slack $\lambda_{\mathrm B} - \lambda^*_{\mathrm B}(s,a)$ or $\lambda_{\mathrm E} - \lambda^*_{\mathrm E}(s,a)$ is strictly positive by the threshold assumption. Furthermore, since $\lambda_{\mathrm K}\ge 0$, we have $\Phi_{\mathrm K}\ge 0$ everywhere. Therefore, \eqref{eq:key} gives $\mathcal L_{\mathrm{R}}(x) > p^*$: no infeasible point can then be a minimizer since $( Q^*_{\mathcal D}, V^*_{\mathcal D})$ already attains $p^*$. Every minimizer is therefore feasible for $\mathrm{LP}_1$, where all hinges of $\Phi_1$ vanish and, by Lemma~\ref{lem:kstep-redundant}, so does $\Phi_{\mathrm K}$, leaving $\mathcal L_{\mathrm{R}}=F$. Minimizing $\mathcal L_{\mathrm{R}}$ thus reduces to $\mathrm{LP}_1$ itself, whose unique solution is $( Q^*_{\mathcal D}, V^*_{\mathcal D})$ by Proposition~\ref{prop:lblp-exactness} together with the feasible-set identity
of Lemma~\ref{lem:kstep-redundant}. Finally, the rollout term entered the argument only through its nonnegativity and its vanishing on the feasible set, never through its magnitude, so the conclusion holds for every $\lambda_{\mathrm K}\ge 0$ and every rollout length $K$.
\end{proof}

\subsection{Proof of Proposition~\ref{prop:bounded}}
\label{app:bounded}

\begin{proof}
Fix any $(Q,V)$ and write $(Q_c, V_c) := (Q - c\mathbf 1,\, V -
c\mathbf 1)$ for $c > 0$. We evaluate \eqref{eq:album-loss} along this
family.

The two objective terms are linear, so they contribute
\begin{equation}
\omega_{\mathrm Q}\,\mathbb E_{\mathcal D}\!\left[Q_c(s,a)\right]
+ \omega_{\mathrm V}\,\mathbb E_{\mathcal D}\!\left[V_c(s)\right]
= \omega_{\mathrm Q}\,\mathbb E_{\mathcal D}\!\left[Q(s,a)\right]
+ \omega_{\mathrm V}\,\mathbb E_{\mathcal D}\!\left[V(s)\right]
- c\,(\omega_{\mathrm Q}+\omega_{\mathrm V}).
\end{equation}

For the residuals, the epigraph term is invariant, since
\begin{equation}
Q_c(s,a) - V_c(s) = \big(Q(s,a) - c\big) - \big(V(s) - c\big)
= Q(s,a) - V(s).
\end{equation}
The one-step target shifts by $\gamma c$, because
$\hat y_1(s,a) = r_0 + \gamma V(s_1)$ is affine in $V$ with coefficient
$\gamma$, so
\begin{equation}
\hat y_1(s,a)\big|_{V_c} - Q_c(s,a)
= \hat y_1(s,a) - Q(s,a) + c\,(1-\gamma).
\end{equation}
The rollout target $\hat y_K(s,a) = \sum_{t<K}\gamma^{t}r_t +
\gamma^{K}V(s_K)$ is affine in $V$ with coefficient $\gamma^{K}$, so both
rollout residuals shift by $c\,(1-\gamma^{K})$,
\begin{equation}
\begin{aligned}
\hat y_K(s,a)\big|_{V_c} - Q_c(s,a)
&= \hat y_K(s,a) - Q(s,a) + c\,(1-\gamma^{K}), \\
\hat y_K(s,a)\big|_{V_c} - V_c(s)
&= \hat y_K(s,a) - V(s) + c\,(1-\gamma^{K}).
\end{aligned}
\end{equation}

Each residual above is an affine function of $c$ with a positive slope,
so there is a finite $c_0$ beyond which all of them are positive at every
sample, and each hinge equals its argument. For $c > c_0$ the objective
is therefore affine in $c$ with
\begin{equation} \label{eq:recession-slope}
\frac{d}{dc}\,L(Q_c,V_c)
= -(\omega_{\mathrm Q}+\omega_{\mathrm V})
+ \lambda_{\mathrm B}(1-\gamma)
+ 2\lambda_{\mathrm K}(1-\gamma^{K}),
\end{equation}
where the factor 2 collects the two rollout penalties, which shift at the same rate. Without stop gradient, the update is the gradient of \eqref{eq:album-loss}, and its component along this direction is the negative of \eqref{eq:recession-slope}. If \eqref{eq:recession-slope} is negative, the gradient update lowers $Q$ and $V$ without bound along this direction, which yields \eqref{eq:bounded}.

When \texttt{stop\_gradient} is applied to $\hat y_K$, its forward value is
unchanged. Hence, the rollout residuals still increase at rate
$1-\gamma^{K}$ along this family, and every hinge is active for $c>c_0$ as
before. The derivative used by the update, however, differs. For $c>c_0$, each
rollout penalty is affine in $c$, and without the stop gradient both arguments
depend on $c$,
\begin{equation}
\frac{d}{dc}\,\lambda_{\mathrm K}
\big(\hat y_K(s,a)\big|_{V_c}-Q_c(s,a)\big)
=
\lambda_{\mathrm K}(1-\gamma^{K}),
\qquad
\frac{d}{dc}\,\lambda_{\mathrm K}
\big(\hat y_K(s,a)\big|_{V_c}-V_c(s)\big)
=
\lambda_{\mathrm K}(1-\gamma^{K}),
\end{equation}
giving the rollout contribution in \eqref{eq:recession-slope}. With the stop
gradient, $\hat y_K$ is treated as constant by the update, so its dependence on
$V_c$ is removed from the derivative,
\begin{equation}
\frac{d}{dc}\,\lambda_{\mathrm K}
\big(\mathrm{sg}[\hat y_K(s,a)]-Q_c(s,a)\big)
=
\lambda_{\mathrm K},
\qquad
\frac{d}{dc}\,\lambda_{\mathrm K}
\big(\mathrm{sg}[\hat y_K(s,a)]-V_c(s)\big)
=
\lambda_{\mathrm K}.
\end{equation}
The objective and epigraph terms are unaffected by the stop gradient, while the
one-step penalty contributes $\lambda_{\mathrm B}(1-\gamma)$ as before. Thus,
the component of the update along this direction is
\begin{equation}
-(\omega_{\mathrm Q}+\omega_{\mathrm V})
+\lambda_{\mathrm B}(1-\gamma)
+2\lambda_{\mathrm K}.
\end{equation}
If this quantity is negative, the update induces an unbounded downward drift
along this direction, yielding \eqref{eq:bounded-sg}.
\end{proof}

\section{The \texttt{MudWorld} Environment}
\label{app:toy}

\begin{figure}[h]
\centering
\includegraphics[width=0.7\textwidth]{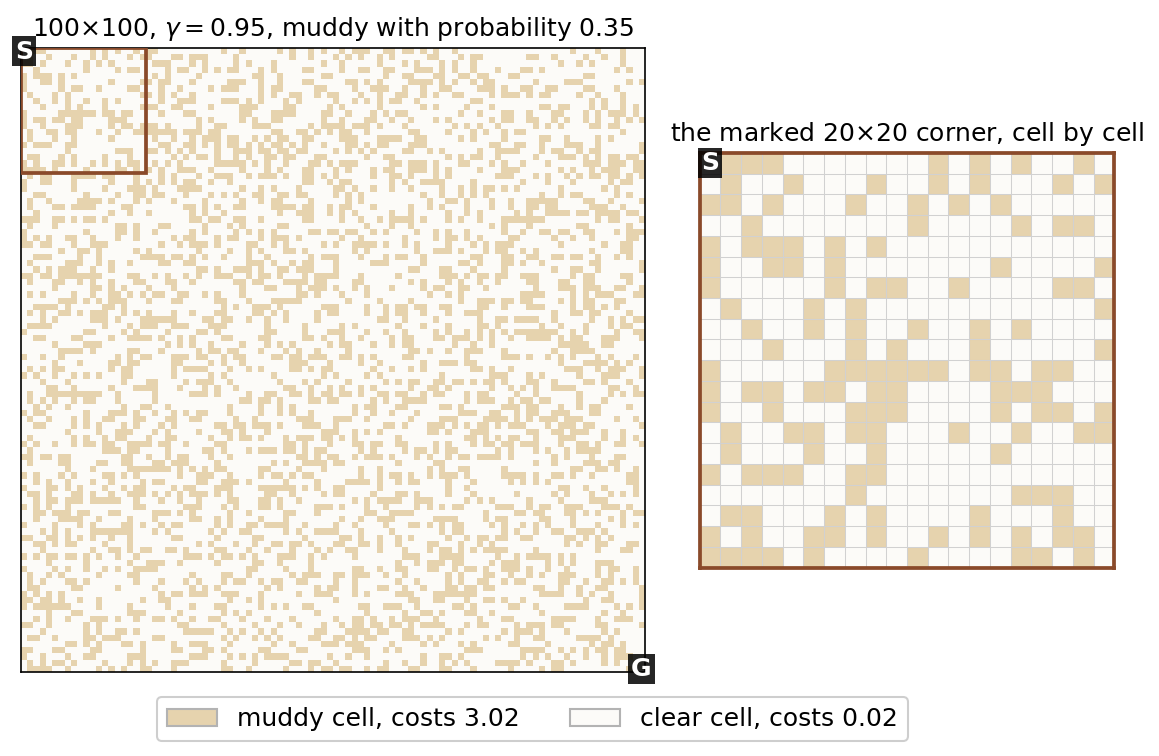}
\caption{The \texttt{MudWorld} environment shared by both numerical experiments.
Brown cells are muddy. The right image magnifies the $20\times20$ corner marked on
the left.}
\label{fig:mudworld}
\end{figure}

To verify the two theoretical statements in Sections~\ref{sec:offline-lblp} and~\ref{subsec:acc}, we construct a custom gridworld, \texttt{MudWorld}, for numerical simulation. The corresponding experiments are presented in Appendices~\ref{app:sandwich} and~\ref{app:kspeed}, respectively.
 Reaching a brown cell costs \texttt{-3.02}, while
reaching a white cell costs only \texttt{-0.02}. We name the environment
\texttt{MudWorld} after these costly cells, which the agent must decide whether
to cross or go around. The agent starts at \texttt{S} and the goal is to reach
\texttt{G}. The grid is $100\times100$, giving $|\mathcal S|=10{,}000$ states,
and $\mathcal A=\{\texttt{Up},\texttt{Down},\texttt{Left},\texttt{Right}\}$ with
$|\mathcal A|=4$, so $|\mathcal X|=40{,}000$. \texttt{S} is the top-left cell
$(0,0)$ and \texttt{G} the bottom-right cell $(99,99)$, a Manhattan distance of
$198$ apart. The dynamics are deterministic. Each action moves the agent one
cell in its direction, and an action that would leave the grid leaves the agent
where it is, so every action is available in every state. Each cell carries a cost that is charged on
entry. Muddy cells are drawn i.i.d.\ with probability $0.35$ and cost $3.02$,
and the remaining cells cost $0.02$.
The reward at a transition into $s'$ is $1-\mathrm{cost}(s')$ if $s'$ is the
goal and $-\mathrm{cost}(s')$ otherwise, so rewards lie in $[-3.02, 0.98]$. We
use $\gamma=0.95$. The reference solution $Q^*$ is obtained by value iteration
on $\mathcal M$.

\section{Numerical Verification of the Sandwich Property}
\label{app:sandwich}

\subsection{Dataset}
Since the setting is offline, a dataset is collected in advance, and it contains
only \texttt{Down} and \texttt{Right} actions, so no \texttt{Up} or \texttt{Left}
action appears anywhere in the data. The behavior is therefore suboptimal, and
the optimal policy of $\mathcal M$ cannot be recovered from the dataset support.
To collect the offline dataset, we run a behavior policy $\beta$ that selects
uniformly between the two actions, excluding the one that would leave the grid,
so every episode reaches the goal in exactly $198$ steps. Collecting $2{,}000$
episodes gives $396{,}000$ transitions, covering
$|\mathcal X_{\mathcal D}|=9{,}693$ pairs over
$|\mathcal S_{\mathcal D}|=5{,}133$ states. 

\subsection{Solving LBLP}
We set $K=10$, and take the objective weights from the
empirical sampling distribution, $c_Q=\omega_Q\,\hat d(s,a)$ and
$c_V=\omega_V\,\hat d_{\mathcal S}(s)$ with $\omega_Q=0.1$ and
$\omega_V=0.05$, where $\hat d$ is the visitation frequency of $\mathcal D$ and
$\hat d_{\mathcal S}$ its state marginal. Both are strictly positive on
$\mathcal X_{\mathcal D}$ and $\mathcal S_{\mathcal D}$, so
Definition~\ref{def:lblp} applies and the minimizer is unique. The program has
$14{,}826$ decision variables, $9{,}693$ for $Q$ on $\mathcal X_{\mathcal D}$ and
$5{,}133$ for $V$ on $\mathcal S_{\mathcal D}$, and $38{,}772$ constraints,
$9{,}693$ from each of the four families, with $225{,}776$ nonzeros in the
constraint matrix. Assembling $y_K$ as an affine function of $V$ by the $K$-step
recursion takes about $1.5$\,s; we then solve the program with
\texttt{scipy.optimize.linprog} using the \texttt{highs} method.

The two reference quantities ($Q^*$ and $G_{\mathcal D}$) are obtained
independently of the program. $Q^{*}$ is the value iteration solution of
Appendix~\ref{app:toy}. We compute $G_{\mathcal D}$ in two passes over
$\mathcal D$, first accumulating the discounted return of every suffix backward
along each trajectory, then taking, for each pair, the maximum over the
positions at which it appears.

\subsection{Results}
We repeat the simulation over eight random seeds, each generating a different
mud layout and a different dataset. As shown in Figure~\ref{fig:seedgrid}, the
LBLP minimizer lies between $G_{\mathcal D}$ and $Q^{*}$ at every pair in all
eight seeds.

\begin{figure}[h]
\centering
\includegraphics[width=\textwidth]{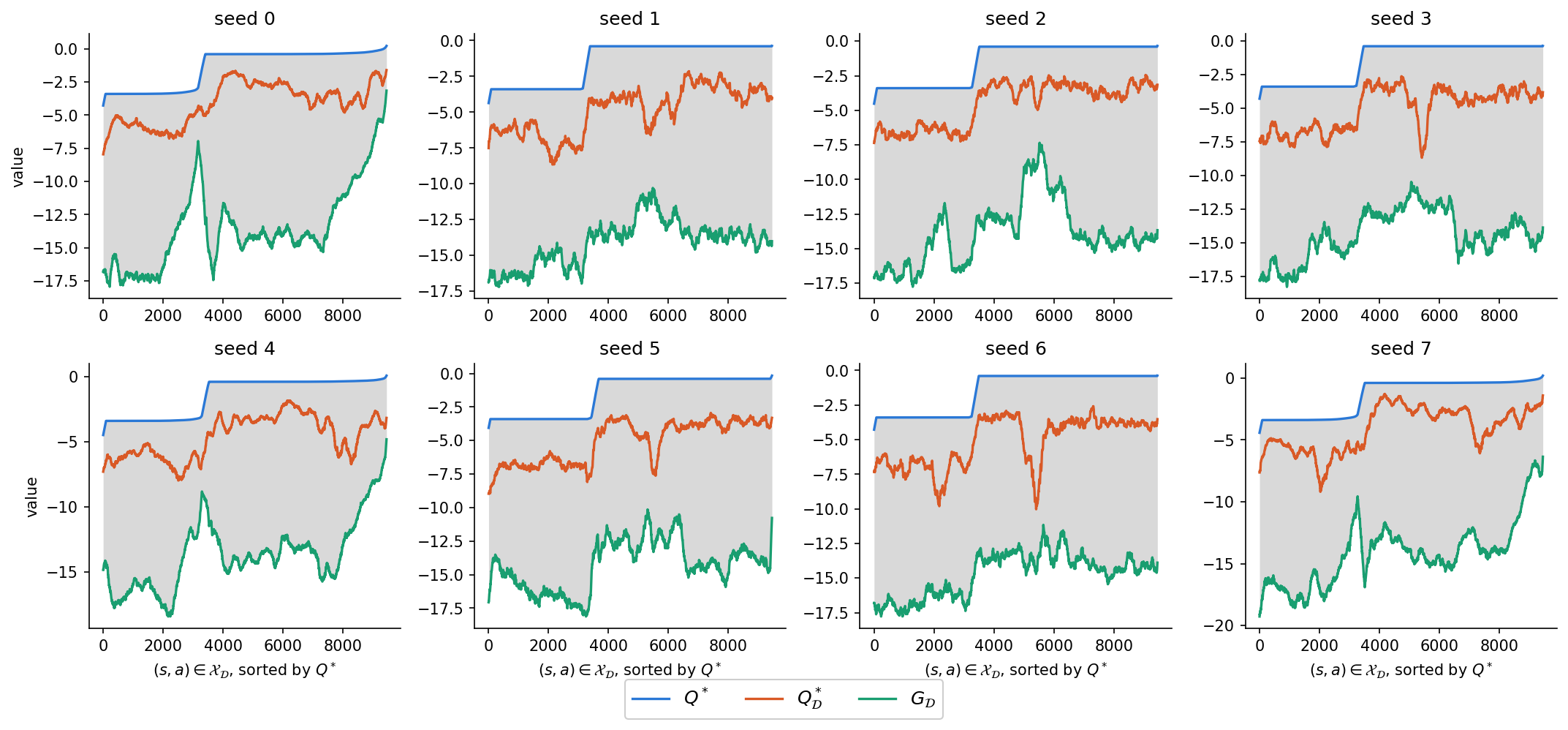}
\caption{Proposition~\ref{prop:sandwich} across eight seeds. Each subfigure plots
$Q^{*}$, $Q^{*}_{\mathcal D}$, and $G_{\mathcal D}$ at every pair of
$\mathcal X_{\mathcal D}$, sorted by $Q^{*}$.}
\label{fig:seedgrid}
\end{figure}

\section{Numerical Verification of the $K$-Rollout Acceleration}
\label{app:kspeed}

\subsection{Dataset}

In this section we reuse the environment of Appendix~\ref{app:toy} with a
different behavior policy. According to Lemma~\ref{lem:propagation}, the error at a pair
after one update is $\min_{k \in \{1,K\}}\big[\delta_k + \gamma^k
e_n(s^{(k)})\big]$, and $\delta_1 = 0$, so the $k=1$ term contributes $\gamma
e_n(s^{(1)})$ and the $k=K$ term contributes $\delta_K + \gamma^K e_n(s^{(K)})$.
While $e_n$ is large, $\gamma^K e_n$ falls below $\gamma e_n$ by enough to
absorb $\delta_K$ and the $K$-step term is selected, so the error contracts at
$\gamma^K$. Once $e_n$ has decayed to the point where $\gamma^K e_n$ is small
against $\delta_K$, the one-step term takes over and the rate returns to
$\gamma$. The residual $\delta_K$ thus acts as a floor below which the
acceleration no longer operates.

To verify such existence of $\delta_K$ and $\gamma^K$, we design and collect a second dataset with an $\varepsilon$-greedy policy that
takes $\arg\max_a Q^{*}(s,a)$ with probability $1-\varepsilon$ and a uniform
action otherwise, with $\varepsilon = 0.02$, from a start cell drawn uniformly
over the $9{,}999$ non-goal cells. Three hundred episodes capped at $400$ steps
give $41{,}178$ transitions over $|\mathcal X_{\mathcal D}| = 3{,}182$ pairs and
$|\mathcal S_{\mathcal D}| = 2{,}703$ states, with visit counts ranging from $1$
to $374$.

\subsection{Computation}

In Figure~\ref{fig:kspeed}(a), we run~\eqref{eq:repair} directly,
$Q_{n+1} = \max_{k\in{1,K}} L_k[V_n]$ and
$V_{n+1}(s) = \max_{a\in\mathcal A_{\mathcal D}(s)} Q_{n+1}(s,a)$, starting from
the underestimate $V_0 \equiv r_{\min}/(1-\gamma)$ given in
Lemma~\ref{lem:propagation}. The reference $(Q^{*}_{\mathcal D},V^{*}_{\mathcal D})$
is obtained by value iteration with $\mathcal T_{\mathcal D}$. As shown in
Figure~\ref{fig:kspeed}(a), during the early phase of the iteration, every
error curve with $K>1$ exhibits a fast decay rate close to $\gamma^K$.
However, the decay rate changes in the later phase and becomes aligned with
the slope of the $K=1$ curve, corresponding to a decay rate of $\gamma$.
This provides partial evidence that $\delta_K$ indeed exists and contributes
to the convergence behavior in the late stage of the iteration. (Lemma~\ref{lem:propagation})

Figure~\ref{fig:kspeed}(b) compares the idealized iteration (\eqref{eq:repair})
with gradient descent on the relaxed LBLP objective (\eqref{eq:relaxed-lblp}),
using $\lambda_B = \lambda_E = \lambda_K = 2.5$ and a learning rate of $0.03$.
Since the two methods use different computational procedures and are measured
in different metrics, we first measure each method using its own no-rollout
version as the baseline, and then compare the resulting speedup factors.
Both methods start from $Q \equiv V \equiv r_{\min}/(1-\gamma)$ and are evaluated
at the same target error.

Together, Figures~\ref{fig:kspeed}(a) and~\ref{fig:kspeed}(b) verify that
the acceleration indeed exists. Figure~\ref{fig:kspeed}(b) further suggests
that, under favorable conditions, this acceleration can even outperform the
idealized iteration. Both methods achieve a speedup compared with the
version without rollout terms.

\begin{figure}[h]
\centering
\includegraphics[width=\textwidth]{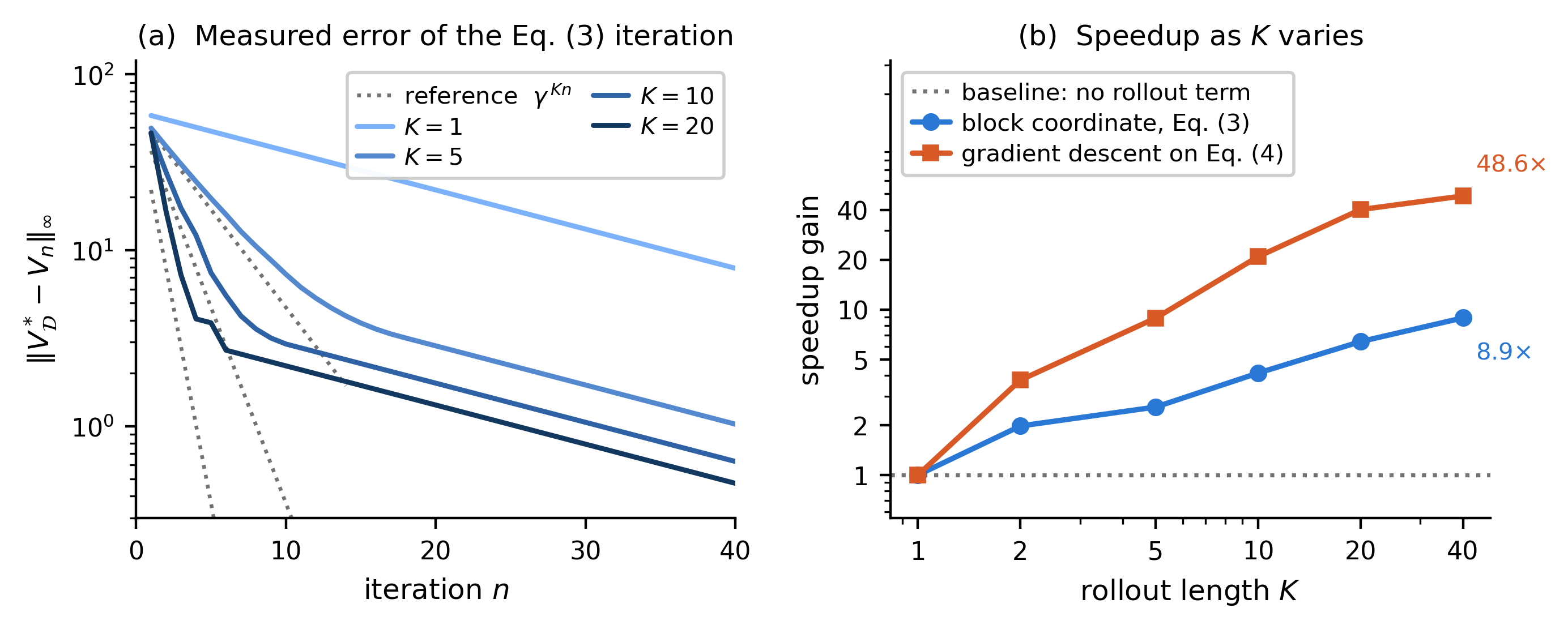}
\caption{$K$-rollout acceleration on \texttt{MudWorld}.
\textbf{(a)} Each curve leaves the start on its own $\gamma^{Kn}$ slope, and at the $\delta_K$
floor the curves rejoin the $\gamma^{n}$ slope.
\textbf{(b)} The direct gradient descent on the relaxed LBLP objective
exhibits an acceleration and can outperform the idealized iteration
under favorable conditions, while both methods achieve a speedup over the
version without rollout terms.
}
\label{fig:kspeed}
\end{figure}

\newpage
\section{Design of \texttt{stop\_gradient} on $\hat y_K$ under Neural Parameterization}
\label{ablation:stop_grad}

As mentioned in Section~\ref{sec:album}, we apply \texttt{stop\_gradient} to the
rollout term $\hat y_K$. The device might be familiar from temporal difference
methods, where the target is the label of a regression and holding it fixed
keeps the regression from chasing its own output. However, our use of it has a different
origin, and follows from the structure of the tabular objective of Definition~\ref{def:relaxed_LBLP}, which the
neural parameterization breaks in one specific way. In
\eqref{eq:relaxed-lblp} the value function is a vector with one coordinate per
state, and distinct states occupy orthogonal coordinates. Under a network
$V_\phi$, the values at distinct states share parameters, so a gradient
component aimed at one state also moves the others. The following remark makes
the consequence precise.

\begin{remark}[One descent step on the rollout penalty]
\label{rem:stopgrad-step}
Fix a recorded pair $(s,a)$ with $s \neq s_K$, suppose the hinge
$\lambda_\mathrm K\big(y_K(s,a) - V(s)\big)_+$ is active, and take one descent
step with rate $\eta$ on this term alone. In \eqref{eq:relaxed-lblp}
the value function is a vector with one coordinate per state, and the
term has the two derivatives
\begin{equation}
\frac{\partial}{\partial V(s)}\,
\lambda_K\big(y_K(s,a) - V(s)\big)_+ = -\lambda_K,
\qquad
\frac{\partial}{\partial V(s_K)}\,
\lambda_K\big(y_K(s,a) - V(s)\big)_+ = \lambda_K\gamma^{K},
\label{eq:tab-grad}
\end{equation}
so the step moves the coordinate $V(s)$ by
\begin{equation}
\Delta V(s) = \eta\lambda_K .
\label{eq:tab-step}
\end{equation}
The second derivative in \eqref{eq:tab-grad} moves $V(s_K)$ as well, by
$-\eta\lambda_K\gamma^{K}$, but that coordinate does not enter
\eqref{eq:tab-step}. In \eqref{eq:album-loss} the decision variable is
$\phi$ and one network carries $V$ at every state, so the two
derivatives of \eqref{eq:tab-grad} collapse into the single gradient
\begin{equation}
\nabla_\phi\,\lambda_K\big(\hat y_K(s,a) - V_\phi(s)\big)_+
= \lambda_K\Big(\gamma^{K}\,\nabla_\phi V_\phi(s_K)
- \nabla_\phi V_\phi(s)\Big),
\label{eq:nn-grad}
\end{equation}
the step is $\Delta\phi = -\eta\lambda_K\big(\gamma^{K}\nabla_\phi
V_\phi(s_K) - \nabla_\phi V_\phi(s)\big)$, and the value at $s$ moves by
\begin{equation}
\Delta V_\phi(s)
= \eta\lambda_K\Big(\big\|\nabla_\phi V_\phi(s)\big\|^{2}
- \gamma^{K}\,\nabla_\phi V_\phi(s)^\top \nabla_\phi V_\phi(s_K)\Big)
+ O(\eta^{2}).
\label{eq:nn-step}
\end{equation}
The second term of \eqref{eq:nn-step} depends on $s_K$, whereas
\eqref{eq:tab-step} has no counterpart to it. This term arises because
the shared parameters of a neural network allow the gradient of
$V_\phi(s_K)$ to contribute to the update at $s$. Consequently, without
\texttt{stop\_gradient}, the gradient of $\hat y_K$ introduces an
additional term that is absent from the tabular update and changes the
intended gradient direction. Applying \texttt{stop\_gradient} to
$\hat y_K$ removes this term, so that \eqref{eq:nn-step} recovers the
gradient structure of \eqref{eq:tab-step}.
\end{remark}

\newpage
\section{Stationary Point Analysis}
\label{app:stationary}

The four hinge terms in \eqref{eq:album-loss} are nondifferentiable at
pairs where their constraint residuals equal zero. Hence, each admits a set of
subgradients rather than a unique gradient. Let $u,v,w,x$ denote selections
from the subdifferentials of
$\lambda_{\mathrm B}\mathbb E_{\mathcal D}[(g_{\mathrm B})_+]$,
$\lambda_{\mathrm E}\mathbb E_{\mathcal D}[(g_{\mathrm E})_+]$,
$\lambda_{\mathrm K}\mathbb E_{\mathcal D}[(g_{\mathrm{KQ}})_+]$, and
$\lambda_{\mathrm K}\mathbb E_{\mathcal D}[(g_{\mathrm{KV}})_+]$,
respectively. For a constraint residual $g$ and coefficient $\lambda$, an
admissible subgradient selection takes the value
$\lambda d_{\mathcal D}(s,a)$ when $g(s,a)>0$, the value $0$ when
$g(s,a)<0$, and any value in
$[0,\lambda d_{\mathcal D}(s,a)]$ when $g(s,a)=0$.

\begin{definition}[Detached update field]
\label{def:detached-field}
For a fixed $\bar y$ on $\mathcal X_{\mathcal D}$, let us define
$\mathcal L_R(Q,V;\bar y)$ as \eqref{eq:album-loss} with
$\operatorname{sg}[\hat y_K]$ replaced by $\bar y$ and with $Q$ and $V$ treated
as free variables. This function is convex in $(Q,V)$, and let us define the
detached update field as
\begin{equation}
\mathcal G(Q,V) := \partial_{(Q,V)}\,\mathcal L_R(Q,V;\bar y)\big|_{\bar y=\hat y_K(V)},
\end{equation}
where $\partial$ denotes the convex subdifferential. The update of ALBUM
descends along $\mathcal G$, which is in general not the subdifferential of any
function of $(Q,V)$.
\end{definition}

At $(Q^{*}_{\mathcal D},V^{*}_{\mathcal D})$, the admissible selections have
the following ranges:
\begin{itemize}
\item $u(s,a) \in [0,\lambda_{\mathrm B}d_{\mathcal D}(s,a)]$ for every
$(s,a)\in\mathcal X_{\mathcal D}$;
\item $v(s,a) \in [0,\lambda_{\mathrm E}d_{\mathcal D}(s,a)]$ when
$Q^{*}_{\mathcal D}(s,a)=V^{*}_{\mathcal D}(s)$, and $v(s,a)=0$ otherwise;
\item $w(s,a) \in [0,\lambda_{\mathrm K}d_{\mathcal D}(s,a)]$ when
$g_{\mathrm{KQ}}(s,a)=0$, and $w(s,a)=0$ otherwise;
\item $x(s,a) \in [0,\lambda_{\mathrm K}d_{\mathcal D}(s,a)]$ when
$g_{\mathrm{KV}}(s,a)=0$, and $x(s,a)=0$ otherwise.
\end{itemize}

The conditions on $w$ and $x$ follow from the two $K$-step residuals at
$(Q^{*}_{\mathcal D},V^{*}_{\mathcal D})$. The $K$-step target along the dataset
trajectory of $(s,a)$ never exceeds the in-sample optimum, so
\[
g_{\mathrm{KQ}}(s,a)=\hat y_K(s,a)-Q^{*}_{\mathcal D}(s,a)\leq 0,
\qquad
g_{\mathrm{KV}}(s,a)=g_{\mathrm{KQ}}(s,a)-\bigl(V^{*}_{\mathcal D}(s)-Q^{*}_{\mathcal D}(s,a)\bigr)\leq 0,
\]
where $g_{\mathrm{KQ}}$ vanishes precisely when the dataset continuation of
$(s,a)$ is optimal in $\mathcal M_{\mathcal D}$. Hence $g_{\mathrm{KV}}$ vanishes
only on those pairs at which $a$ also attains $V^{*}_{\mathcal D}(s)$.

\begin{proposition}[Stationarity of the LBLP solution under the detached update]
\label{prop:stationary}
Suppose there exist admissible selections $u,v,w,x$ at
$(Q^{*}_{\mathcal D},V^{*}_{\mathcal D})$ satisfying
\begin{equation} \label{eq:stat-q}
\omega_{\mathrm Q}d_{\mathcal D}(s,a)
-u(s,a)+v(s,a)-w(s,a)=0
\qquad
\text{for every }(s,a)\in\mathcal X_{\mathcal D},
\end{equation}
and
\begin{equation} \label{eq:stat-v}
\omega_{\mathrm V}d_{\mathcal S}(s)
+\gamma\!\!\sum_{(t,b)\in\mathcal I(s)}\!\!u(t,b)
-\!\!\sum_{a\in\mathcal A_{\mathcal D}(s)}
\bigl(v(s,a)+x(s,a)\bigr)=0
\qquad
\text{for every }s\in\mathcal S_{\mathcal D},
\end{equation}
where
\[
\mathcal I(s)
:=
\{(t,b)\in\mathcal X_{\mathcal D}:t'=s\}
\]
collects the pairs whose one-step target contains $V(s)$.
Then
\[
0\in\mathcal G(Q^{*}_{\mathcal D},V^{*}_{\mathcal D}).
\]
\end{proposition}

\begin{proof}
With $\bar y$ held fixed, Definition~\ref{def:detached-field} gives
\begin{equation}
\begin{aligned}
\mathcal L_R(Q,V;\bar y)
=&~\omega_{\mathrm Q}\,\mathbb E_{\mathcal D}\!\left[Q(s,a)\right]
+\omega_{\mathrm V}\,\mathbb E_{\mathcal D}\!\left[V(s)\right] \\
&+\lambda_{\mathrm B}\,\mathbb E_{\mathcal D}\!\left[\left(r+\gamma V(s')-Q(s,a)\right)_+\right]
+\lambda_{\mathrm E}\,\mathbb E_{\mathcal D}\!\left[\left(Q(s,a)-V(s)\right)_+\right] \\
&+\lambda_{\mathrm K}\left(
\mathbb E_{\mathcal D}\!\left[\left(\bar y(s,a)-Q(s,a)\right)_+\right]
+\mathbb E_{\mathcal D}\!\left[\left(\bar y(s,a)-V(s)\right)_+\right]\right).
\end{aligned}
\end{equation}
Every term is a finite convex function and each hinge composes
$[\,\cdot\,]_+$ with an affine map of $(Q,V)$, so the sum rule and the chain
rule for affine composition
apply~\citep[Theorems 23.8 and 23.9]{rockafellar1997convex}. Selecting one
subgradient from each hinge determines one element of $\mathcal G$, and we take
the admissible selections $u,v,w,x$ of the statement, evaluated at
$(Q^{*}_{\mathcal D},V^{*}_{\mathcal D})$ with $\bar y=\hat y_K$, which already
carry the weight $d_{\mathcal D}$.

Differentiating the objective terms gives
\begin{equation}
\frac{\partial}{\partial Q(s,a)}\,\omega_{\mathrm Q}\mathbb E_{\mathcal D}[Q]
=\omega_{\mathrm Q}d_{\mathcal D}(s,a),
\qquad
\frac{\partial}{\partial V(s)}\,\omega_{\mathrm V}\mathbb E_{\mathcal D}[V]
=\omega_{\mathrm V}d_{\mathcal S}(s).
\end{equation}
For the hinges, the chain rule sends each selection to the transpose of its
affine map, so at $Q(s,a)$ we obtain the contributions
\begin{equation}
-u(s,a),\qquad +v(s,a),\qquad -w(s,a),\qquad 0,
\end{equation}
from $g_{\mathrm B}$, $g_{\mathrm E}$, $g_{\mathrm{KQ}}$ and $g_{\mathrm{KV}}$
respectively, and at $V(s)$ the contributions
\begin{equation}
\gamma\!\!\sum_{(t,b)\in\mathcal I(s)}\!\!u(t,b),
\qquad
-\!\!\sum_{a\in\mathcal A_{\mathcal D}(s)}\!\!v(s,a),
\qquad
0,
\qquad
-\!\!\sum_{a\in\mathcal A_{\mathcal D}(s)}\!\!x(s,a),
\end{equation}
where the first sum collects the pairs whose one-step argument
$r+\gamma V(t')$ contains $V(s)$, and the two zeros hold because $\bar y$ does
not depend on $(Q,V)$.

Summing the contributions yields
\begin{equation}
\omega_{\mathrm Q}d_{\mathcal D}(s,a)-u(s,a)+v(s,a)-w(s,a)=0
\end{equation}
by \eqref{eq:stat-q}, and
\begin{equation}
\omega_{\mathrm V}d_{\mathcal S}(s)
+\gamma\!\!\sum_{(t,b)\in\mathcal I(s)}\!\!u(t,b)
-\!\!\sum_{a\in\mathcal A_{\mathcal D}(s)}\bigl(v(s,a)+x(s,a)\bigr)=0
\end{equation}
by \eqref{eq:stat-v}. The selected element of $\mathcal G$ therefore has zero
component at every $Q(s,a)$ and every $V(s)$, which yields
$0\in\mathcal G(Q^{*}_{\mathcal D},V^{*}_{\mathcal D})$.
\end{proof}

The existence of such selections is not automatic. The objective terms
contribute $\omega_{\mathrm Q}d_{\mathcal D}(s,a)$ and
$\omega_{\mathrm V}d_{\mathcal S}(s)$ at each coordinate, while the selections
enter \eqref{eq:stat-q} and \eqref{eq:stat-v} with the opposite sign and are
each bounded by their respective penalty coefficients. Which selections carry
these contributions separates two regimes, and each yields its own bound on the
coefficients.

\paragraph{The one-step selection regime.}
Setting $w=x=0$ leaves $u$ and $v$ to carry both conditions. The affine map of
$g_{\mathrm B}$ places $u(s,a)$ at $Q(s,a)$ in \eqref{eq:stat-q} and at $V(s')$
in \eqref{eq:stat-v} with weight $\gamma$, so the same selection enters two
conditions and accumulates along a trajectory of $\mathcal M_{\mathcal D}$. The
bound $\lambda_{\mathrm B}$ places on it therefore grows with the effective
horizon $1/(1-\gamma)$.

\begin{remark}[The one-step selection alone]
\label{rem:onestep-regime}
Suppose $|\mathcal A_{\mathcal D}(s)|=1$ for every $s\in\mathcal S_{\mathcal D}$,
so that $d_{\mathcal D}(s,a)=d_{\mathcal S}(s)$, and consider a dataset
trajectory $(s_0,a_0),\dots,(s_{T-1},a_{T-1})$ in which $\mathcal I(s_0)$ is
empty and $\mathcal I(s_t)=\{(s_{t-1},a_{t-1})\}$ for $1\le t\le T-1$, so that
each pair appears once and $d_{\mathcal D}$ is uniform. At
$(Q^{*}_{\mathcal D},V^{*}_{\mathcal D})$ the residual $g_{\mathrm B}$ vanishes
by Bellman optimality, and the single action gives
$Q^{*}_{\mathcal D}(s,a)=V^{*}_{\mathcal D}(s)$, so both $u$ and $v$ are free.
Fix $w=x=0$ and write $u_t:=u(s_t,a_t)$ and $v_t:=v(s_t,a_t)$. At $(s_t,a_t)$,
\eqref{eq:stat-q} gives
\[
v_t = u_t - \omega_{\mathrm Q}\,d_{\mathcal D},
\]
and at $s_t$, the single action makes \eqref{eq:stat-v} reduces to
$\omega_{\mathrm V} d_{\mathcal D} + \gamma u_{t-1} - v_t = 0$, so
\[
u_t = \gamma u_{t-1}
+ (\omega_{\mathrm Q}+\omega_{\mathrm V})\,d_{\mathcal D},
\qquad u_{-1}=0 .
\]
Unrolling from $t=0$,
\[
u_t = (\omega_{\mathrm Q}+\omega_{\mathrm V})\,d_{\mathcal D}
\sum_{j=0}^{t}\gamma^{\,j}
= (\omega_{\mathrm Q}+\omega_{\mathrm V})\,d_{\mathcal D}\,
\frac{1-\gamma^{\,t+1}}{1-\gamma},
\]
which increases in $t$ and stays below
$(\omega_{\mathrm Q}+\omega_{\mathrm V})d_{\mathcal D}/(1-\gamma)$. The
admissibility bounds $u_t\le\lambda_{\mathrm B}d_{\mathcal D}$ and
$v_t\le\lambda_{\mathrm E}d_{\mathcal D}$ therefore hold once
\[
\lambda_{\mathrm B}\ \ge\
\frac{\omega_{\mathrm Q}+\omega_{\mathrm V}}{1-\gamma},
\qquad
\lambda_{\mathrm E}\ \ge\
\frac{\omega_{\mathrm Q}+\omega_{\mathrm V}}{1-\gamma}-\omega_{\mathrm Q},
\]
so $0\in\mathcal G(Q^{*}_{\mathcal D},V^{*}_{\mathcal D})$ holds without the
rollout selections at coefficients of this order. The same bound on
$\lambda_{\mathrm B}$ appears in \eqref{eq:bounded-sg} at
$\lambda_{\mathrm K}=0$.
\end{remark}

Without uniformity the ratio $d_{\max}/d_{\min}$ enters this bound, so the
stationarity conditions require a larger $\lambda_{\mathrm B}$ than
\eqref{eq:bounded-sg} does at $\lambda_{\mathrm K}=0$. The aggregation of
Remark~\ref{rem:aggregate} replaces the coordinatewise bounds by their totals,
which carry no dependence on $d_{\mathcal D}$.

\paragraph{The rollout selections regime.}
Taking $w$ and $x$ nonzero instead allows them to carry the objective terms. The
selection $w(s,a)$ enters \eqref{eq:stat-q} at $(s,a)$ and $x(s,a)$ enters
\eqref{eq:stat-v} at $s$, and neither appears elsewhere, since $\bar y$ is held
fixed in Definition~\ref{def:detached-field} and contributes nothing at
$V(s_K)$. Unlike $u$, which the affine map of $g_{\mathrm B}$ places at both
$Q(s,a)$ and $V(s')$, they offset the objective term at that coordinate alone,
so the resulting bound on $\lambda_{\mathrm K}$ involves neither $\gamma$ nor
$K$.

\begin{proposition}[Stationarity of the LBLP solution under the detached update]
\label{prop:below-threshold}
Suppose the transition dynamics of $\mathcal{M}$ are deterministic, and let
$\delta_K$ be as in Lemma~\ref{lem:propagation}. Let us define
$\mathcal{A}^{*}_{\mathcal{D}}(s) := \{a \in \mathcal{A}_{\mathcal{D}}(s) :
Q^{*}_{\mathcal{D}}(s,a) = V^{*}_{\mathcal{D}}(s)\}$. Assume
$\delta_K(s,a) = 0$ for every $(s,a) \in \mathcal{X}_{\mathcal{D}}$ and
\begin{equation}
\lambda_{\mathrm{K}} \ge \omega_{\mathrm{Q}},
\qquad
\lambda_{\mathrm{K}} \sum_{a \in \mathcal{A}^{*}_{\mathcal{D}}(s)}
d_{\mathcal{D}}(s,a) \ \ge\ \omega_{\mathrm{V}}\, d_{\mathcal{S}}(s)
\quad \text{for every } s \in \mathcal{S}_{\mathcal{D}}.
\label{eq:lk-sufficient}
\end{equation}
Then $0 \in \mathcal{G}(Q^{*}_{\mathcal{D}}, V^{*}_{\mathcal{D}})$ for every
$\lambda_{\mathrm{B}} \ge 0$ and every $\lambda_{\mathrm{E}} \ge 0$.
\end{proposition}

\begin{proof}
The assumption $\delta_K = 0$ gives $g_{\mathrm{KQ}}(s,a) = 0$ on
$\mathcal{X}_{\mathcal{D}}$, and $g_{\mathrm{KV}}(s,a) = 0$ exactly on the pairs
with $a \in \mathcal{A}^{*}_{\mathcal{D}}(s)$. The selections
\begin{equation*}
u = 0, \qquad v = 0, \qquad w(s,a) = \omega_{\mathrm Q} d_{\mathcal D}(s,a),
\end{equation*}
\begin{equation*}
x(s,a) = \omega_{\mathrm V}\, d_{\mathcal S}(s)\,
\frac{d_{\mathcal D}(s,a)}
{\sum_{a' \in \mathcal A^{*}_{\mathcal D}(s)} d_{\mathcal D}(s,a')}\,
\mathbf 1\{a \in \mathcal A^{*}_{\mathcal D}(s)\}.
\end{equation*}
are admissible under \eqref{eq:lk-sufficient}, since
$w(s,a) \le \lambda_{\mathrm{K}} d_{\mathcal{D}}(s,a)$ and
$x(s,a) \le \lambda_{\mathrm{K}} d_{\mathcal{D}}(s,a)$, and they vanish wherever
the corresponding residual is strictly negative. Substituting them into
\eqref{eq:stat-q} and \eqref{eq:stat-v} yields
$\omega_{\mathrm{Q}} d_{\mathcal{D}} - w = 0$ and
$\omega_{\mathrm{V}} d_{\mathcal{S}}(s)
- \sum_{a \in \mathcal{A}_{\mathcal{D}}(s)} x(s,a) = 0$,
so Proposition~\ref{prop:stationary} applies.
\end{proof}

\begin{corollary}[Single-action empirical MDP]
\label{cor:single-action}
If $|\mathcal{A}_{\mathcal{D}}(s)| = 1$ for every $s \in \mathcal{S}_{\mathcal{D}}$,
then $\delta_K = 0$ on $\mathcal{X}_{\mathcal{D}}$ and $d_{\mathcal{D}}(s,a) =
d_{\mathcal{S}}(s)$, so \eqref{eq:lk-sufficient} reduces to
$\lambda_{\mathrm{K}} \ge \max(\omega_{\mathrm{Q}}, \omega_{\mathrm{V}})$. The
coefficients of Table~\ref{tab:hyperparams} satisfy this condition.
\end{corollary}

\begin{remark}[The boundedness conditions along the common shift direction]
\label{rem:aggregate}
Summing \eqref{eq:stat-q} over $(s,a)\in\mathcal X_{\mathcal D}$ and
\eqref{eq:stat-v} over $s\in\mathcal S_{\mathcal D}$, the selection $v$ cancels,
since it appears with $+v(s,a)$ in the first and with $-v(s,a)$ in the second,
and the selection $u$ appears once with $-u(s,a)$ and once with
$+\gamma u(s,a)$. With
$\sum_{\mathcal X_{\mathcal D}} d_{\mathcal D}=
\sum_{\mathcal S_{\mathcal D}} d_{\mathcal S}=1$,
\[
\omega_{\mathrm Q}+\omega_{\mathrm V}
=(1-\gamma)\!\!\sum_{\mathcal X_{\mathcal D}}\!\! u
+\!\!\sum_{\mathcal X_{\mathcal D}}\!\! w
+\!\!\sum_{\mathcal X_{\mathcal D}}\!\! x .
\]
Admissibility gives $u(s,a)\le\lambda_{\mathrm B}d_{\mathcal D}(s,a)$ and
$w(s,a),x(s,a)\le\lambda_{\mathrm K}d_{\mathcal D}(s,a)$, so the three sums are
bounded by $\lambda_{\mathrm B}$, $\lambda_{\mathrm K}$ and
$\lambda_{\mathrm K}$ respectively, and
\[
\omega_{\mathrm Q}+\omega_{\mathrm V}
\ \le\ \lambda_{\mathrm B}(1-\gamma)+2\lambda_{\mathrm K},
\]
which is \eqref{eq:bounded-sg}. Setting $\lambda_{\mathrm K}=0$ leaves
$\lambda_{\mathrm B}\ge(\omega_{\mathrm Q}+\omega_{\mathrm V})/(1-\gamma)$, the
bound of Remark~\ref{rem:onestep-regime}. Without \texttt{stop\_gradient} each
rollout selection reappears at $V(s_K)$ with weight $\gamma^K$, and the same
computation yields \eqref{eq:bounded}. Condition \eqref{eq:bounded-sg} is the
component of \eqref{eq:stat-q} and \eqref{eq:stat-v} along the common shift
direction.
\end{remark}

\newpage
\section{OGBench Experiment Details}
\label{app:config}

\begin{table}[h]
\centering

\begin{tabular}{ll}
\toprule
Hyperparameter & Value \\
\midrule
Actor learning rate & $3 \times 10^{-4}$ \\
Critic learning rate & $1 \times 10^{-4}$ \\
Optimizer & Adam~\citep{kingma2014adam} \\
Gradient steps & $1{,}000{,}000$ \\
Minibatch size & $256$ \\
MLP dimensions & $[512, 512, 512, 512]$ (actor and critic) \\
Nonlinearity & GELU~\citep{hendrycks2016gaussian} \\
Layer normalization & critic only \\
Critic ensemble size & $1$ \\
Policy & Gaussian, state-independent $\sigma$ \\
Discount factor $\gamma$ & $0.995$ \\
Sequence length $K$ & $10$ \\
$\omega_{Q}$ & $0.1$ \\
$\omega_{V}$ & $0.05$ \\
$\lambda_{B}$ & $2.5$ \\
$\lambda_{E}$ & $2.5$ \\
$\lambda_{K}$ & $1.0$ \\
BC coefficient $\alpha$ & task-specific (see below) \\
\bottomrule
\end{tabular}
\caption{Hyperparameters for ALBUM}
\label{tab:hyperparams}
\end{table}

\begin{table}[h]
\centering

\begin{tabular}{lc}
\toprule
Environment & $\alpha$ \\
\midrule
antmaze-large-navigate       & $0.01$ \\
antmaze-giant-navigate       & $0.003$ \\
humanoidmaze-medium-navigate & $0.003$ \\
humanoidmaze-large-navigate  & $0.003$ \\
antsoccer-arena-navigate     & $0.01$ \\
cube-single-play             & $0.3$ \\
cube-double-play             & $0.1$ \\
scene-play                   & $0.1$ \\
puzzle-3x3-play              & $0.1$ \\
puzzle-4x4-play              & $0.1$ \\
\bottomrule
\end{tabular}
\caption{BC coefficient $\alpha$ per environment.}
\label{tab:alpha}
\end{table}

\newpage

\section{Additional Results on OGBench}
\providecommand{\std}{}
\renewcommand{\std}[1]{\,$\pm$\scalebox{0.7}{$#1$}}
\begin{table}[h]

\centering
\scriptsize
\setlength{\tabcolsep}{2.5pt}
\renewcommand{\arraystretch}{0.9}
\resizebox{\textwidth}{!}{%
\begin{tabular}{ll|cccccccccc|c}
\toprule
 & & \multicolumn{3}{c}{Gaussian Policies} & \multicolumn{3}{c}{Diffusion Policies} & \multicolumn{4}{c}{Flow Policies} & \multicolumn{1}{c}{Gaussian (Ours)} \\
\cmidrule(lr){3-5}\cmidrule(lr){6-8}\cmidrule(lr){9-12}\cmidrule(lr){13-13}
Environment & Task & BC & IQL & ReBRAC & IDQL & SRPO & CAC & FAWAC & FBRAC & IFQL & FQL & ALBUM (Ours) \\
\midrule
antmaze-large-navigate & 1$^{*}$ & 0 & 48 & 91 & 0 & 0 & 42 & 1 & 70 & 24 & 80 & \textbf{92.5}\std{3.1} \\
 & 2 & 6 & 42 & \textbf{88} & 14 & 4 & 1 & 0 & 35 & 8 & 57 & 80.2\std{6.0} \\
 & 3 & 29 & 72 & 51 & 26 & 3 & 49 & 12 & 83 & 52 & \textbf{93} & 91.8\std{2.7} \\
 & 4 & 8 & 51 & 84 & 62 & 45 & 17 & 10 & 37 & 18 & 80 & \textbf{92.5}\std{4.0} \\
 & 5 & 10 & 54 & 90 & 2 & 1 & 55 & 9 & 76 & 38 & 83 & \textbf{92.0}\std{4.7} \\
\midrule
antmaze-giant-navigate & 1$^{*}$ & 0 & 0 & \textbf{27} & 0 & 0 & 0 & 0 & 0 & 0 & 4 & 3.5\std{3.6} \\
 & 2 & 0 & 1 & 16 & 0 & 0 & 0 & 0 & 4 & 0 & 9 & \textbf{76.2}\std{7.2} \\
 & 3 & 0 & 0 & \textbf{34} & 0 & 0 & 0 & 0 & 0 & 0 & 0 & 0.0\std{0.0} \\
 & 4 & 0 & 0 & 5 & 0 & 0 & 0 & 0 & 9 & 0 & 14 & \textbf{41.2}\std{21.2} \\
 & 5 & 1 & 19 & \textbf{49} & 0 & 0 & 0 & 0 & 6 & 13 & 16 & 19.5\std{25.6} \\
\midrule
humanoidmaze-medium-navigate & 1$^{*}$ & 1 & 32 & 16 & 1 & 0 & 38 & 6 & 25 & \textbf{69} & 19 & 2.0\std{2.8} \\
 & 2 & 1 & 41 & 18 & 1 & 1 & 47 & 40 & 76 & 85 & \textbf{94} & 79.0\std{6.9} \\
 & 3 & 6 & 25 & 36 & 0 & 2 & \textbf{83} & 19 & 27 & 49 & 74 & 4.8\std{3.7} \\
 & 4 & 0 & 0 & \textbf{15} & 1 & 1 & 5 & 1 & 1 & 1 & 3 & 2.2\std{2.5} \\
 & 5 & 2 & 66 & 24 & 1 & 3 & 91 & 31 & 63 & \textbf{98} & 97 & 89.5\std{3.1} \\
\midrule
humanoidmaze-large-navigate & 1$^{*}$ & 0 & 3 & 2 & 0 & 0 & 1 & 0 & 0 & 6 & \textbf{7} & 0.0\std{0.0} \\
 & 2 & \textbf{0} & \textbf{0} & \textbf{0} & \textbf{0} & \textbf{0} & \textbf{0} & \textbf{0} & \textbf{0} & \textbf{0} & \textbf{0} & \textbf{0.0}\std{0.0} \\
 & 3 & 1 & 7 & 8 & 3 & 1 & 2 & 1 & 10 & \textbf{48} & 11 & 6.2\std{4.8} \\
 & 4 & 1 & 1 & 1 & 0 & 0 & 0 & 0 & 0 & 1 & \textbf{2} & 0.5\std{1.3} \\
 & 5 & 0 & 1 & \textbf{2} & 0 & 0 & 0 & 0 & 1 & 0 & 1 & 1.0\std{1.0} \\
\midrule
antsoccer-arena-navigate & 1 & 2 & 14 & 0 & 44 & 2 & 1 & 22 & 17 & 61 & \textbf{77} & 66.0\std{7.0} \\
 & 2 & 2 & 17 & 0 & 15 & 3 & 0 & 8 & 8 & 75 & \textbf{88} & 58.0\std{7.8} \\
 & 3 & 0 & 6 & 0 & 0 & 0 & 8 & 11 & 16 & 14 & \textbf{61} & 4.8\std{6.9} \\
 & 4$^{*}$ & 1 & 3 & 0 & 0 & 0 & 0 & 12 & 24 & 16 & \textbf{39} & 30.2\std{9.1} \\
 & 5 & 0 & 2 & 0 & 0 & 0 & 0 & 9 & 15 & 0 & \textbf{36} & 21.5\std{5.6} \\
\midrule
cube-single-play & 1 & 10 & 88 & 89 & 95 & 89 & 77 & 81 & 73 & 79 & \textbf{97} & 84.5\std{10.6} \\
 & 2$^{*}$ & 3 & 85 & 92 & 96 & 82 & 80 & 81 & 83 & 73 & \textbf{97} & 89.8\std{6.1} \\
 & 3 & 9 & 91 & 93 & \textbf{99} & 96 & 98 & 87 & 82 & 88 & 98 & 90.8\std{4.1} \\
 & 4 & 2 & 73 & 92 & 93 & 70 & 91 & 79 & 79 & 79 & \textbf{94} & 84.2\std{6.5} \\
 & 5 & 3 & 78 & 87 & 90 & 61 & 80 & 78 & 76 & 77 & \textbf{93} & 67.0\std{6.0} \\
\midrule
cube-double-play & 1 & 8 & 27 & 45 & 39 & 7 & 21 & 21 & 47 & 35 & \textbf{61} & 47.2\std{8.8} \\
 & 2$^{*}$ & 0 & 1 & 7 & 16 & 0 & 2 & 2 & 22 & 9 & \textbf{36} & 5.2\std{4.7} \\
 & 3 & 0 & 0 & 4 & 17 & 0 & 3 & 1 & 4 & 8 & \textbf{22} & 4.5\std{3.1} \\
 & 4 & 0 & 0 & 1 & 0 & 0 & 0 & 0 & 0 & 1 & \textbf{5} & 2.0\std{1.4} \\
 & 5 & 0 & 4 & 4 & 1 & 0 & 3 & 2 & 2 & 17 & \textbf{19} & 7.8\std{7.0} \\
\midrule
scene-play & 1 & 19 & 94 & 95 & \textbf{100} & 94 & \textbf{100} & 87 & 96 & 98 & \textbf{100} & 91.8\std{4.9} \\
 & 2$^{*}$ & 1 & 12 & 50 & 33 & 2 & 50 & 18 & 46 & 0 & 76 & \textbf{93.2}\std{3.2} \\
 & 3 & 1 & 32 & 55 & 94 & 4 & 49 & 38 & 78 & 54 & \textbf{98} & 70.8\std{11.7} \\
 & 4 & 2 & 0 & 3 & 4 & 0 & 0 & 6 & 4 & 0 & 5 & \textbf{66.2}\std{9.5} \\
 & 5 & \textbf{0} & \textbf{0} & \textbf{0} & \textbf{0} & \textbf{0} & \textbf{0} & \textbf{0} & \textbf{0} & \textbf{0} & \textbf{0} & \textbf{0.0}\std{0.0} \\
\midrule
puzzle-3x3-play & 1 & 5 & 33 & \textbf{97} & 52 & 89 & \textbf{97} & 25 & 63 & 94 & 90 & 96.2\std{4.1} \\
 & 2 & 1 & 4 & 1 & 0 & 0 & 0 & 4 & 2 & 1 & 16 & \textbf{77.5}\std{32.6} \\
 & 3 & 1 & 3 & 3 & 0 & 0 & 0 & 1 & 1 & 0 & 10 & \textbf{39.0}\std{16.7} \\
 & 4$^{*}$ & 1 & 2 & 2 & 0 & 0 & 0 & 1 & 2 & 0 & 16 & \textbf{93.0}\std{7.9} \\
 & 5 & 1 & 3 & 5 & 0 & 0 & 0 & 1 & 2 & 0 & \textbf{16} & 0.8\std{1.4} \\
\midrule
puzzle-4x4-play & 1 & 1 & 12 & 26 & 48 & 24 & 44 & 1 & 32 & \textbf{49} & 34 & 40.2\std{6.4} \\
 & 2 & 0 & 7 & 12 & 14 & 0 & 0 & 0 & 5 & 4 & \textbf{16} & 11.2\std{3.5} \\
 & 3 & 0 & 9 & 15 & 34 & 21 & 29 & 1 & 20 & \textbf{50} & 18 & 33.2\std{15.0} \\
 & 4$^{*}$ & 0 & 5 & 10 & \textbf{26} & 7 & 1 & 0 & 5 & 21 & 11 & 11.8\std{4.7} \\
 & 5 & 0 & 4 & 7 & \textbf{24} & 1 & 0 & 0 & 4 & 2 & 7 & 6.2\std{3.1} \\
\bottomrule
\end{tabular}}
\caption{Per-task results against all baselines. Ours is the mean over 8 seeds at the final evaluation, $\pm$ the standard deviation across seeds; baseline means are taken from the FQL paper. The best entry in each row is bold. $^{*}$ indicates the default task.}
\label{tab:per-task-full}
\end{table}

\newpage
\section{Ablation studies and Further Analysis}
\label{app:ablation}

This appendix gives the full results for the questions raised in Section~\ref{sec:ablations}.
\begin{itemize}
    \item Why does ALBUM remain stable without target networks or critic ensembles?(Appendix~\ref{ablation:EMA,ensemble})

   \item How does ALBUM compare with action-chunking methods? (Appendix~\ref{app:qc})
       
    \item How much does ALBUM reduce computational cost? (Appendix~\ref{app:cost})

    \item Why should ALBUM satisfy the boundedness condition? (Appendix~\ref{app:boundness})
    
    \item How do the rollout horizon $K$ and $\lambda_{\mathrm K}$ affect learning? (Appendix~\ref{ablation:K})

    \item Does raising the one-step penalties recover performance without the $K$-step constraint? (Appendix~\ref{riselk})

    \item Does the ALBUM critic improvement transfer to a generative policy? (Appendix~\ref{app:album-ifql})
    
\end{itemize}
It also answers the following questions, which the main text has no space for.
\begin{itemize}

    \item How sensitive is ALBUM to the penalty coefficients? (Appendix~\ref{sec:abl-lambda})
    
    \item Is ALBUM's lower performance on the harder navigation domains a matter of convergence speed? (Appendix~\ref{abl:3m})
    
    \item Is the benefit of the $K$-step penalties specific to the LBLP formulation? (Appendix~\ref{app:hinge-iql})

    \item In continuous state spaces the program's solution reduces to the discounted return of each dataset trajectory. What is the gap between the neural critic and this quantity? (Appendix~\ref{app:vstar})

\end{itemize}

\newpage
\subsection{Target Networks and Critic Ensembles}
\label{ablation:EMA,ensemble}

\begin{figure}[h]
    \centering
    \includegraphics[width=1\linewidth]{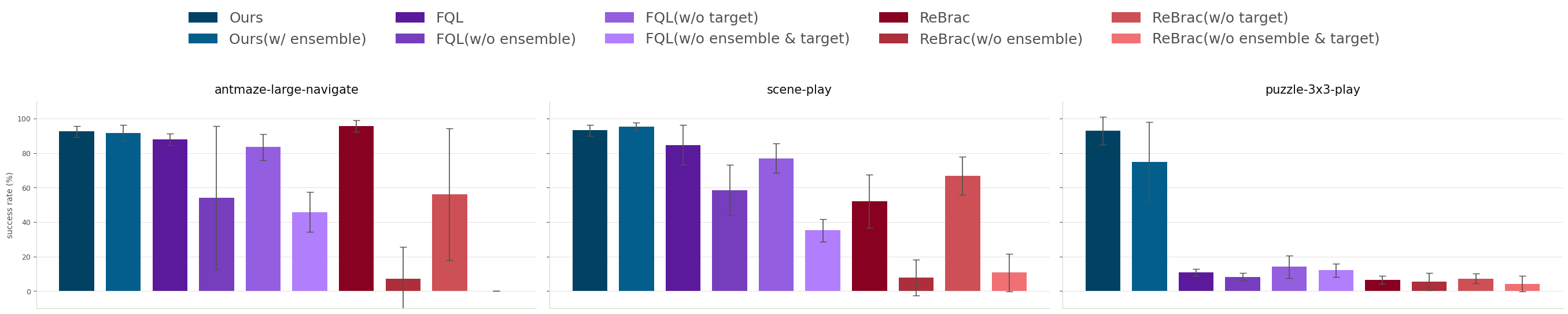}
    \caption{Critic ensemble and target network. ALBUM has no target network and appears once; Both baselines degrade when their ensemble and target are removed, ReBRAC most.}
    \label{fig:ensemble-target}
\end{figure}

We compare against FQL and ReBRAC, the strongest flow-based and Gaussian
methods in Table~\ref{tab:main}. Both rely on a target network and a two-critic ensemble, which have been standard in offline actor-critic methods since TD3 \citep{fujimoto2018addressing} and TD3+BC \citep{fujimoto2021minimalist}. ReBRAC provides the closer comparison, since
ALBUM also extracts a Gaussian policy, so differences in policy
expressiveness do not account for the performance gap. We use the default task of each environment, which is task 1 of \texttt{antmaze-large-navigate}, task 2 of \texttt{scene-play}, and task 4 of \texttt{puzzle-3x3-play}. Figure~\ref{fig:ensemble-target} removes the critic ensemble, the target network, and both. ALBUM uses a single critic and no target network, the configuration with both devices removed. In contrast, both baselines degrade on \texttt{antmaze-large-navigate} and \texttt{scene-play} when the ensemble is removed and when both devices are removed, most notably ReBRAC. Removing the target network alone degrades ReBRAC most on \texttt{antmaze-large-navigate}.

This behavior follows from where the two devices come from. Target networks
with EMA updates trace back to DQN and DDPG~\citep{mnih2015human, lillicrap2015continuous}, where the semi-gradient update of
Q-learning regresses $Q_\theta$ onto a target that itself depends on
$Q_\theta$, and the slowly moving copy is what keeps that regression from
chasing its own output. The ALBUM objective contains no squared regression
and no such target, since $Q_\theta$ and $V_\phi$ are decision variables of a
single objective, so the mechanism the device was introduced to control is
absent. Critic ensembles address a different failure. The maximization
inside the Q-learning update selects whichever action carries the largest
error, which biases the backup upward, and double Q-learning and its
ensemble descendants avoid this by decorrelating the estimates that the
maximum is taken over~\citep{hasselt2010double, van2016deep, fujimoto2018addressing}. In LBLP the maximization is a family of constraints rather than a step of the update rule, so the critic update never selects an action through a maximum over estimated values.

\newpage
\subsection{Comparison with Action-Chunking Methods}
\label{app:qc}

\paragraph{Comparison with horizon-reduction methods.}
We also compare ALBUM with methods that shorten the horizon through multi-step
returns or action chunks, such as QC-FQL~\citep{li2026reinforcement},
DEAS~\citep{kim2026deas}, DQC~\citep{li2026decoupled},
CGQ~\citep{song2026chunkguided}, and NFQL, an $n$-step variant of FQL that CGQ
evaluates as a baseline. Table~\ref{tab:qc} reports their scores from CGQ,
restricted to the five task categories whose dataset, reward, and number of
gradient steps match those of \citet{park2025flow} and ours, so the subset is determined by protocol agreement rather than by task performance. ALBUM attains the highest
score on \texttt{antmaze-large-navigate} and \texttt{antmaze-giant-navigate},
exceeding DQC, the strongest action-chunking method on both domains, by 18.8
and 16.4 points, respectively. Its average of 37.5 is higher than those of NFQL and QC-FQL and lower than those of FQL, DEAS, CGQ, and DQC, and the methods with action-chunked critics lead on \texttt{cube-double-play} and \texttt{puzzle-4x4-play}. These methods rely on
chunked critics, and those marked with $\dagger$ in Table~\ref{tab:qc} also
execute action sequences, whereas ALBUM uses a single critic and a Gaussian
policy that executes one action at a time. It also requires far less
computation than QC-FQL, DQC, and CGQ, as we quantify in Appendix~\ref{app:cost}.

\providecommand{\std}{}
\renewcommand{\std}[1]{\,$\pm$\scalebox{0.7}{$#1$}}
\begin{table}[h]
\centering
\footnotesize
\setlength{\tabcolsep}{2.5pt}
\resizebox{\textwidth}{!}{%
\begin{tabular}{lcccccc|c}
\toprule
& \multicolumn{2}{c}{Flow, no Q-chunking}
& \multicolumn{4}{c|}{Flow, Q-chunking}
& \multicolumn{1}{c}{Gaussian, no Q-chunking} \\
\cmidrule(lr){2-3}\cmidrule(lr){4-7}\cmidrule(lr){8-8}
Task Category
& FQL & NFQL
& QC-FQL$^{\dagger}$ & DEAS$^{\dagger}$ & DQC & CGQ
& ALBUM (Ours) \\
\midrule
antmaze-large-navigate
& \underline{79} & 47 & 20 & 67 & 71 & 67
& \textbf{89.8} \\
antmaze-giant-navigate
& 9 & 2 & 0 & 8 & \underline{10} & 4
& \textbf{26.4} \\
humanoidmaze-medium-navigate
& \underline{58} & 23 & 4 & 37 & \textbf{93} & 34
& 34.0 \\
\midrule
cube-double-play
& 29 & 11 & 39 & \underline{48} & 31 & \textbf{69}
& 14.6 \\
puzzle-4x4-play
& 17 & 23 & 26 & \textbf{39} & 8 & \underline{28}
& 22.7 \\
\midrule
Average
& 38.4 & 21.2 & 17.8 & 39.8 & \textbf{42.6} & \underline{40.4}
& 37.5 \\
\bottomrule
\end{tabular}}
\caption{Comparison with multi-step and Q-chunking methods. Baseline numbers obtained from CGQ. $^{\dagger}$Methods executing action sequences $\in \mathcal A^h$. Baselines report the mean of the last three evaluations (150 episodes each) over 4 seeds, and ALBUM the mean of the last three evaluations (50 episodes each) over 8 seeds. The ALBUM runs are the same as those in Table 1, which reports the evaluation at 1M steps only, so the two tables differ only in the aggregation of evaluation points. Best in bold, second underlined.}
\label{tab:qc}
\end{table}

\newpage
\subsection{Computational Cost Comparison}
\label{app:cost}
In this section, we measure the computational cost of each method in terms of
the number of network parameters, GFLOPs per step, wall-clock time per step,
the number of value networks, and peak GPU memory usage. We use
\texttt{antmaze-large-navigate-singletask-task1-v0} and run three seeds on a
single NVIDIA GeForce RTX 3080. Every method uses its default configuration,
including its critic ensemble and target copies, and the action-chunking
methods use the chunk lengths reported for \texttt{antmaze-large} in Table~8 of
\citet{song2026chunkguided}, which are $5$ for QC-FQL, $25$ for DQC, and $10$
for CGQ. Since ReBRAC skips its actor update on alternate steps, we also report
ReBRAC$^{*}$, which updates the actor at every step, so that every method
updates its actor at the same rate. All methods use a minibatch of 256 transitions and four hidden layers of width
512. Parameter counts and GFLOPs per step are read from the compiled program,
and peak GPU memory is read from the device allocator. Reported times are the
mean over three seeds, and the standard deviation stays below 3\% in every
case.

ALBUM trains with 2.43M parameters and holds two value networks, whereas the
other methods use 4.86M to 10.05M parameters and hold four to eight value
networks. Its peak GPU memory of 78\,MiB is the lowest among all methods, a
reduction of 39\% against FQL and of 71\% against CGQ. When every method updates
its actor at each step, ALBUM uses 28\% fewer FLOPs and 1\% less time per step
than ReBRAC$^{*}$, 7\% fewer FLOPs and 8\% less time than IQL, and 42\% fewer
FLOPs and 44\% less time than FQL. Against the action-chunking methods, it uses
40\% fewer FLOPs and 39\% less time than QC-FQL, 23\% fewer FLOPs and 47\% less
time than DQC, and 72\% fewer FLOPs and 74\% less time than CGQ. ReBRAC with its
default delayed actor update remains 15\% faster than ALBUM per step while using
twice the parameters. Figure~\ref{fig:pareto} places each method by its average success rate and its cost per step, and ALBUM lies on the Pareto frontier under both task sets.

\begin{table}[h]
\centering
\small
\begin{tabular}{lrrrrr}
\toprule
Method & Params & GFLOP/step & ms/step & Q/V nets & Peak GPU (MiB) \\
\midrule
ReBRAC & 4.86M & \textbf{3.67} & \textbf{0.99} & 4 & 132 \\
\textbf{ALBUM (ours)} & \textbf{2.43M} & 3.82 & 1.17 & \textbf{2} & \textbf{78} \\
ReBRAC* & 4.86M & 5.33 & 1.19 & 4 & 130 \\
IQL & 4.86M & 4.12 & 1.28 & 5 & 134 \\
QC-FQL & 5.00M & 6.41 & 1.93 & 4 & 130 \\
NFQL & 4.87M & 6.22 & 2.06 & 4 & 130 \\
FQL & 4.87M & 6.64 & 2.11 & 4 & 128 \\
DQC & 6.69M & 4.96 & 2.22 & 7 & 165 \\
CGQ & 10.05M & 13.79 & 4.59 & 8 & 265 \\
\bottomrule
\end{tabular}
\caption{Training cost per gradient step on antmaze-large-navigate. Q/V nets counts every value and critic network the agent holds, ensemble members and target copies included. Best in each column is bold.}
\label{tab:cost-antmaze}
\end{table}

\begin{figure}[h]
\centering
\includegraphics[width=0.8\textwidth]{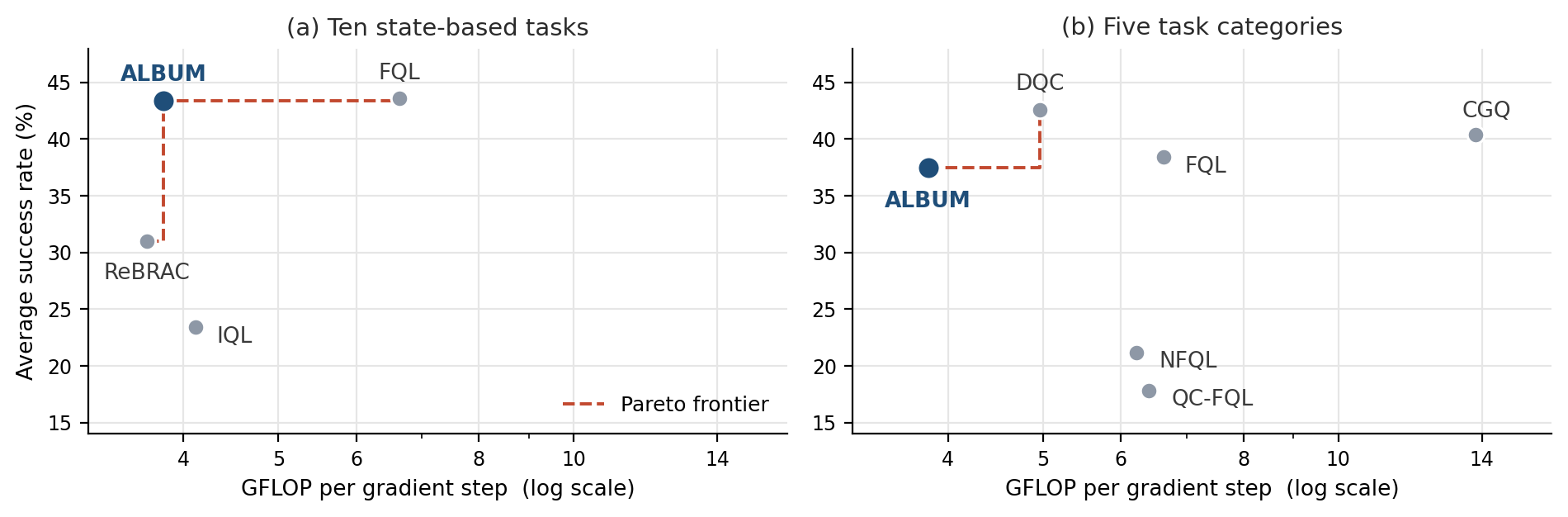}
\caption{Average success rate against training cost per gradient step. Left: the ten state-based tasks of Table~\ref{tab:main}, for the methods measured in Table~\ref{tab:cost-antmaze}. Right: the five task categories of Table~\ref{tab:qc}. Cost is measured on antmaze-large-navigate with each method in its default configuration, and the dashed line marks the Pareto frontier. ALBUM lies on the frontier in both, matching the average of FQL at 42\% fewer FLOPs per step.}
\label{fig:pareto}
\end{figure}

\begin{figure}[h]
\centering
\includegraphics[width=\textwidth]{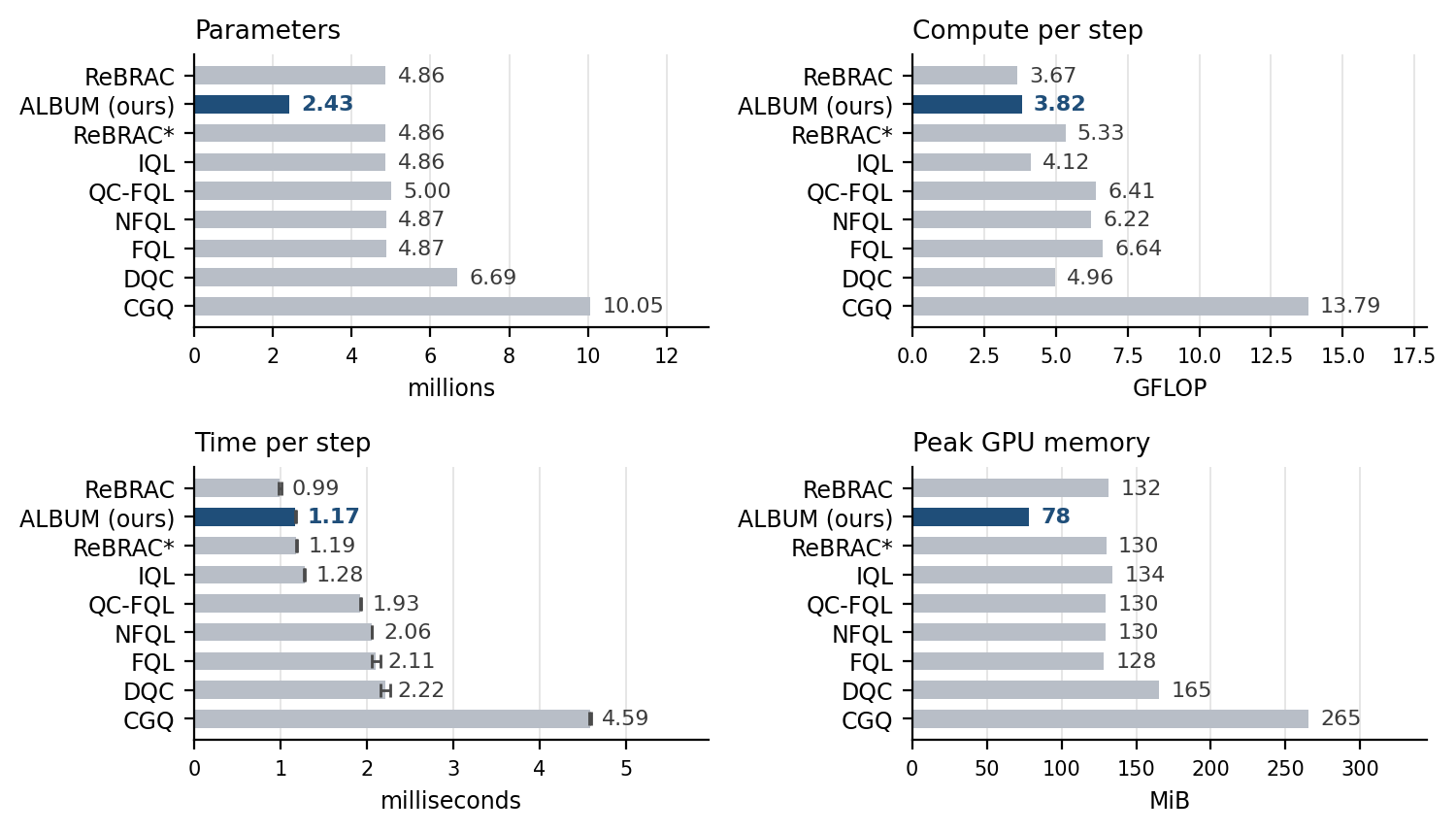}
\caption{Computational cost comparison regarding the number of network parameters, GFLOP per step,
wall-clock time per step, the number of critic networks, and peak GPU
memory usage.}
\label{fig:cost}
\end{figure}

\newpage
\subsection{Violating the Necessary Condition}
\label{app:boundness}

Proposition~\ref{prop:bounded} states a necessary condition on the coefficients for the update to stay bounded along the common shift direction, and it carries over to parameterizations whose outputs admit such a shift, including the networks of our experiments. The default coefficients of Table~\ref{tab:hyperparams} are chosen to satisfy it, with $\lambda_{\mathrm{B}}(1-\gamma) + 2\lambda_{\mathrm{K}} = 2.0125$ against $\omega_{\mathrm{Q}} + \omega_{\mathrm{V}} = 0.15$. We evaluate the condition with $\lambda_{\mathrm K}=0$,
which violates it at the default coefficients. The resulting critic statistics
and policy performance are reported in Table~\ref{tab:lk01}.

\begin{table}[h]
\centering\small\setlength{\tabcolsep}{4.5pt}
\begin{tabular}{l rr rr rr}
\toprule
& \multicolumn{2}{c}{\texttt{antmaze-large}} & \multicolumn{2}{c}{\texttt{scene}} & \multicolumn{2}{c}{\texttt{puzzle-3x3}} \\
\cmidrule(lr){2-3}\cmidrule(lr){4-5}\cmidrule(lr){6-7}
$\lambda_{\mathrm K}$ & 0 & 1 & 0 & 1 & 0 & 1 \\
\midrule
$\lambda_{\mathrm B}(1-\gamma)+2\lambda_{\mathrm K}$
& \textcolor{red}{$0.013$} & $2.01$
& \textcolor{red}{$0.013$} & $2.01$
& \textcolor{red}{$0.013$} & $2.01$ \\
\midrule
$V_{\max}$ & $-2.2{\times}10^{6}$ & $0.6$ & $-4.0{\times}10^{5}$ & $-15$ & $-5.2{\times}10^{5}$ & $-40$ \\
$V_{\mathrm{mean}}$ & $-5.9{\times}10^{6}$ & $-136$ & $-1.3{\times}10^{6}$ & $-285$ & $-8.6{\times}10^{5}$ & $-369$ \\
$V_{\min}$ & $-5.9{\times}10^{6}$ & $-199$ & $-1.3{\times}10^{6}$ & $-478$ & $-8.6{\times}10^{5}$ & $-521$ \\
$V_{\max}{-}V_{\mathrm{mean}}$ & $3.7{\times}10^{6}$ & $137$ & $9.3{\times}10^{5}$ & $271$ & $3.4{\times}10^{5}$ & $328$ \\
$g_{\mathrm B}\le0$ (\%) & 1 & 65 & 0 & 64 & 1 & 68 \\
Success (\%) & 0.5 & 92.5 & 1.0 & 93.2 & 0.5 & 93.0 \\
\bottomrule
\end{tabular}
\caption{Critic statistics at $1$M steps with
$\lambda_{\mathrm B}{=}\lambda_{\mathrm E}{=}2.5$ and $K{=}10$ for
$\lambda_{\mathrm K}\in\{0,1\}$. The first row reports the left side of
\eqref{eq:bounded-sg}, whose right side is
$\omega_{\mathrm Q}{+}\omega_{\mathrm V}=0.15$ throughout. Entries in red
violate the condition. Setting $\lambda_{\mathrm K}{=}0$ produces extreme
downward drift in $V$, accompanied by severe violations of the Bellman
constraints and near-zero success, whereas $\lambda_{\mathrm K}{=}1$
satisfies the condition and avoids this behavior.}
\label{tab:lk01}
\end{table}

With $\lambda_{\mathrm K}=0$, the critic diverges on all three tasks. On
\texttt{antmaze-large-navigate}, for example, the mean value reaches
$-5.9\times10^{6}$, more than four orders of magnitude below the physical
lower bound $-1/(1-\gamma)=-200$ induced by the per-step reward of $-1$.
The corresponding $g_{\mathrm B}\leq0$ rate is only $1\%$. This behavior
matches the descent direction identified in the proof of
Proposition~\ref{prop:bounded}.


\newpage
\subsection{Rollout Horizon $K$ and Coefficient $\lambda_\mathrm K$}
\label{ablation:K}
\begin{figure}[h]
\centering
\includegraphics[width=0.95\linewidth]{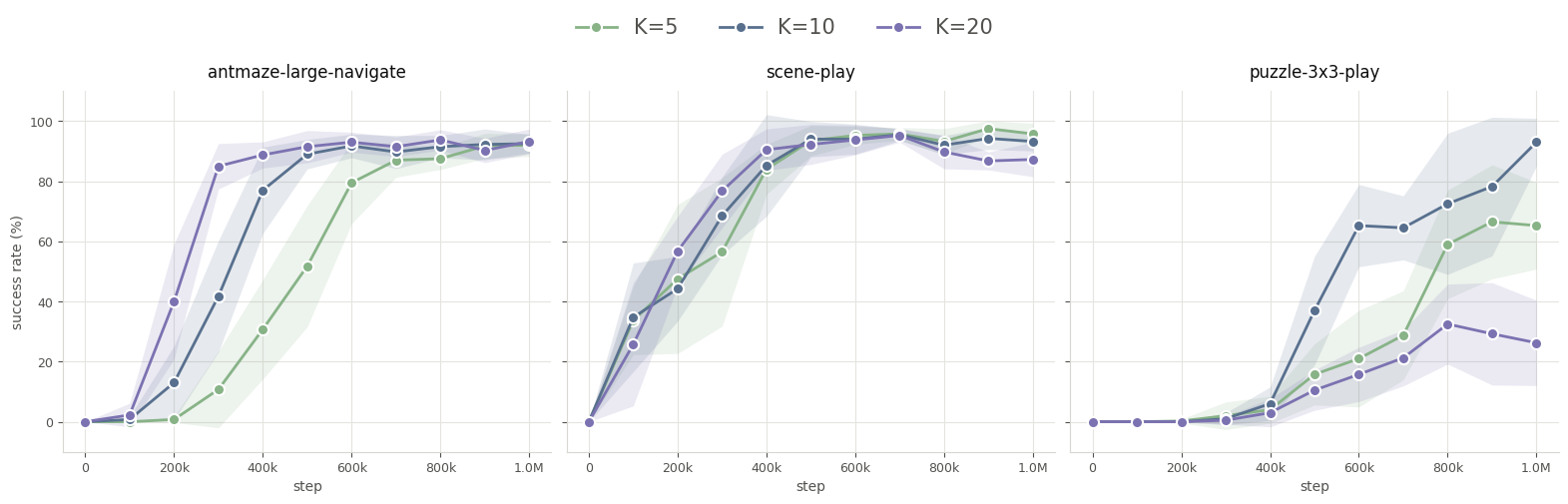}
\caption{Rollout horizon $K$ at the default $\lambda_{\mathrm K}$. Longer
horizons help on goal-directed data and hurt on play data.}
\label{fig:K}
\end{figure}

\begin{table}[h]
\centering\small\setlength{\tabcolsep}{6pt}
\begin{tabular}{l r r r}
\toprule
 & \texttt{antmaze-large-navigate} & \texttt{scene-play} & \texttt{puzzle-3x3-play} \\
\midrule
K=5 & $88.2$\std{4.4} & $\mathbf{95.8}$\std{3.5} & $65.2$\std{14.5} \\
K=10 & $92.5$\std{3.1} & $93.2$\std{3.2} & $\mathbf{93.0}$\std{7.9} \\
K=20 & $\mathbf{93.5}$\std{4.2} & $87.2$\std{5.7} & $26.2$\std{14.2} \\
\bottomrule
\end{tabular}
\caption{Success rate (\%) at the final evaluation, $8$ seeds $\times$ $50$ episodes; $\pm$ std across seeds.}
\label{tab:K}
\end{table}
\begin{figure}[H]
\centering
\includegraphics[width=0.95\linewidth]{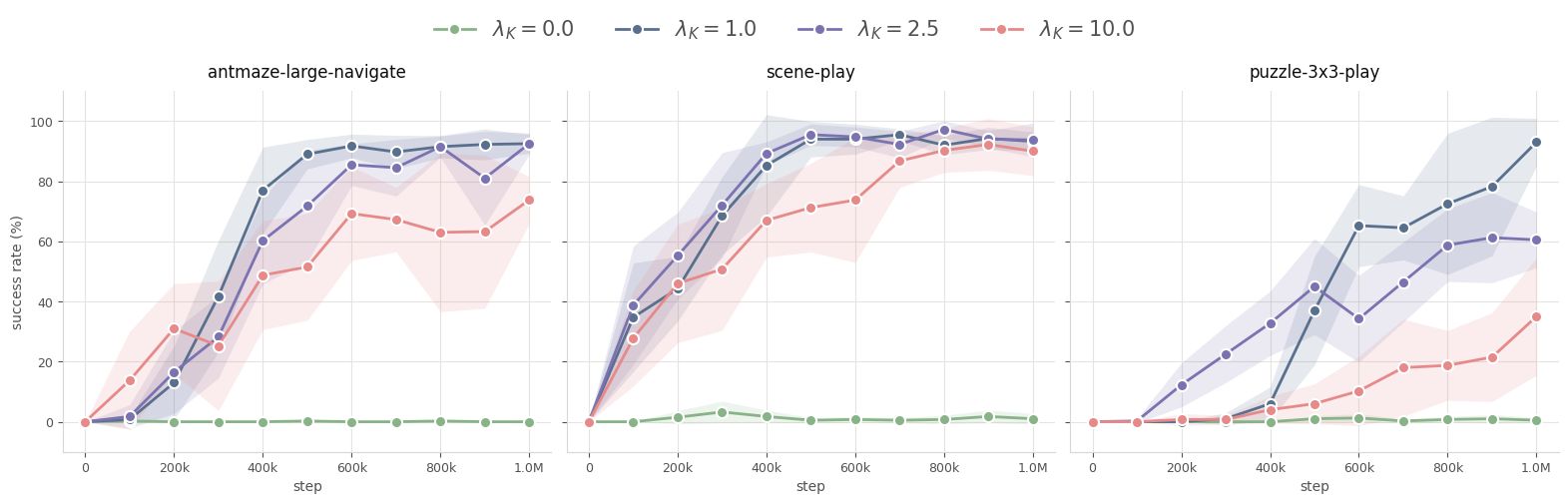}
\caption{Rollout coefficient $\lambda_{\mathrm K}$ at $K=10$. $\lambda_{\mathrm K}=1$ is the only setting that exceeds 90\% success on all three tasks, and $\lambda_{\mathrm K}=0$ fails on all of them.}
\label{fig:lambdaK}
\end{figure}
\begin{table}[h]
\centering\small\setlength{\tabcolsep}{6pt}
\begin{tabular}{l r r r}
\toprule
 & \texttt{antmaze-large-navigate} & \texttt{scene-play} & \texttt{puzzle-3x3-play} \\
\midrule
$\lambda_K=0.0$ & $0.5$\std{0.9} & $1.0$\std{1.7} & $0.5$\std{0.9} \\
$\lambda_K=1.0$ & $\mathbf{92.5}$\std{3.1} & $93.2$\std{3.2} & $\mathbf{93.0}$\std{7.9} \\
$\lambda_K=2.5$ & $79.5$\std{26.2} & $\mathbf{93.8}$\std{5.6} & $60.5$\std{9.3} \\
$\lambda_K=10.0$ & $73.0$\std{18.6} & $90.0$\std{8.1} & $35.0$\std{19.5} \\
\bottomrule
\end{tabular}
\caption{Success rate (\%) at the final evaluation, $8$ seeds $\times$ $50$ episodes; $\pm$ std across seeds.}
\label{tab:LK}
\end{table}

Figure~\ref{fig:K} and Table~\ref{tab:K} vary the rollout horizon $K$. On
\texttt{antmaze-large-navigate} a longer horizon accelerates learning, with
$K=20$ reaching $50\%$ success at $300$k steps against $500$k for $K=5$,
and the final success rate is unchanged. This matches
Theorem~\ref{thm:speedup}, under which a longer horizon never slows the
idealized iteration and contracts at $\gamma^K$ where the $K$-step term is
selected. On \texttt{puzzle-3x3-play} the final success rate drops to
$26.2\%$ at $K=20$ against $93.0\%$ at $K=10$, which the idealized iteration
does not predict. The source lies outside that analysis, in the sampled
constraint and the function approximation.

Figure~\ref{fig:lambdaK} and Table~\ref{tab:LK} vary $\lambda_{\mathrm K}$.
Setting $\lambda_{\mathrm K}=0$ fails on all three tasks, consistent with
Proposition~\ref{prop:bounded} and Appendix~\ref{riselk}.
$\lambda_{\mathrm K}=1$ is the best setting on
\texttt{antmaze-large-navigate} and \texttt{puzzle-3x3-play}, and
\texttt{scene-play} attains $93.8$ at $\lambda_{\mathrm K}=2.5$ against
$93.2$ at $1$. Table~\ref{tab:hyperparams} adopts $K=10$ and
$\lambda_{\mathrm K}=1$.

\newpage
\subsection{Additional Tuning under $\lambda_K = 0$}
\label{riselk}

\begin{table}[h]
\centering\small\setlength{\tabcolsep}{6pt}
\resizebox{\textwidth}{!}{%
\begin{tabular}{l rr rrr r}
\toprule
Environment & $\lambda_{\mathrm B}{=}\lambda_{\mathrm E}$ & $\lambda_{\mathrm K}$ &
$V_{\max}$ & $V_{\mathrm{mean}}$ & $V_{\min}$ & Success (\%) \\
\midrule
\texttt{antmaze-large-navigate}
 & $2.5$  & $0$ & $-2.2{\times}10^{6}$ & $-5.9{\times}10^{6}$ & $-5.9{\times}10^{6}$ & $0.5$\std{0.9} \\
 & $25$   & $0$ & $-207.7$ & $-973.8$ & $-982.0$ & $0.0$\std{0.0} \\
 & $250$  & $0$ & $1.7$    & $-163.3$ & $-168.4$ & $0.2$\std{0.7} \\
 & $2500$ & $0$ & $-2.1$   & $-29.1$  & $-36.0$  & $4.2$\std{2.7} \\
 & $2.5$  & $1$ & $0.6$    & $-136.2$ & $-199.4$ & $\mathbf{92.5}$\std{3.1} \\
\midrule
\texttt{scene-play}
 & $2.5$  & $0$ & $-4.0{\times}10^{5}$ & $-1.3{\times}10^{6}$ & $-1.3{\times}10^{6}$ & $1.0$\std{1.9} \\
 & $25$   & $0$ & $-947.4$ & $-3222.3$ & $-3236.7$ & $1.2$\std{1.5} \\
 & $250$  & $0$ & $-84.2$  & $-374.4$  & $-393.5$  & $1.5$\std{1.8} \\
 & $2500$ & $0$ & $-0.5$   & $-43.1$   & $-63.7$   & $2.0$\std{2.4} \\
 & $2.5$  & $1$ & $-14.7$  & $-285.4$  & $-477.5$  & $\mathbf{93.2}$\std{3.4} \\
\midrule
\texttt{puzzle-3x3-play}
 & $2.5$  & $0$ & $-5.2{\times}10^{5}$ & $-8.6{\times}10^{5}$ & $-8.6{\times}10^{5}$ & $0.5$\std{0.9} \\
 & $25$   & $0$ & $-3998.5$ & $-7678.6$ & $-7726.9$ & $0.2$\std{0.7} \\
 & $250$  & $0$ & $-344.7$  & $-573.7$  & $-587.7$  & $0.5$\std{0.9} \\
 & $2500$ & $0$ & $-22.4$   & $-60.7$   & $-75.8$   & $0.0$\std{0.0} \\
 & $2.5$  & $1$ & $-40.2$   & $-368.6$  & $-520.5$  & $\mathbf{93.0}$\std{8.5} \\
\bottomrule
\end{tabular}}
\caption{Raising $\lambda_{\mathrm B}$ past the value
$(\omega_Q+\omega_V)/(1-\gamma) = 30$ required by
Proposition~\ref{prop:bounded} at $\lambda_{\mathrm K} = 0$, with the rollout
term removed, $1$M steps, $8$ seeds. The last row of each group is the default
setting. The values become bounded once the condition holds, and no setting
without the rollout term exceeds $4.2\%$ on any of the three tasks.}
\label{tab:lambda-threshold}
\end{table}

\begin{table}[h]
\centering\small\setlength{\tabcolsep}{6pt}
\begin{tabular}{rr rrr rr r}
\toprule
$\lambda_{\mathrm B}{=}\lambda_{\mathrm E}$ & $\lambda_{\mathrm K}$ &
$V_{\max}$ & $V_{\mathrm{mean}}$ & $V_{\min}$ &
$V_{\max}{-}V_{\mathrm{mean}}$ & $V_{\max}{-}V_{\min}$ &
Success (\%) \\
\midrule
$2.5$    & $0$ & $6.1{\times}10^{3}$ & $-4.4{\times}10^{6}$ & $-4.5{\times}10^{6}$ & $4.4{\times}10^{6}$ & $4.6{\times}10^{6}$ & $8.5$\std{5.6} \\
$1000$   & $0$ & $0.3$ & $-16.9$ & $-24.0$ & $17.2$ & $24.3$ & $66.0$\std{15.2} \\
$10000$  & $0$ & $0.4$ & $-6.1$  & $-8.8$  & $6.4$  & $9.1$  & $49.5$\std{14.1} \\
\midrule
$2.5$    & $1$ & $1.4$ & $-63.1$ & $-106.5$ & $64.5$ & $107.9$ & $\mathbf{89.8}$\std{6.5} \\
\bottomrule
\end{tabular}
\caption{\texttt{cube-single-play}, task 2, $8$ seeds. Raising
$\lambda_{\mathrm B}$ bounds the values and lifts success to $66.0\%$, but
compresses the range they span and never reaches the default; pushing it
further to $10^{4}$ compresses the range again and costs performance.}
\label{tab:lambda-threshold-cube}
\end{table}

\begin{figure}[H]
\centering
\includegraphics[width=\textwidth]{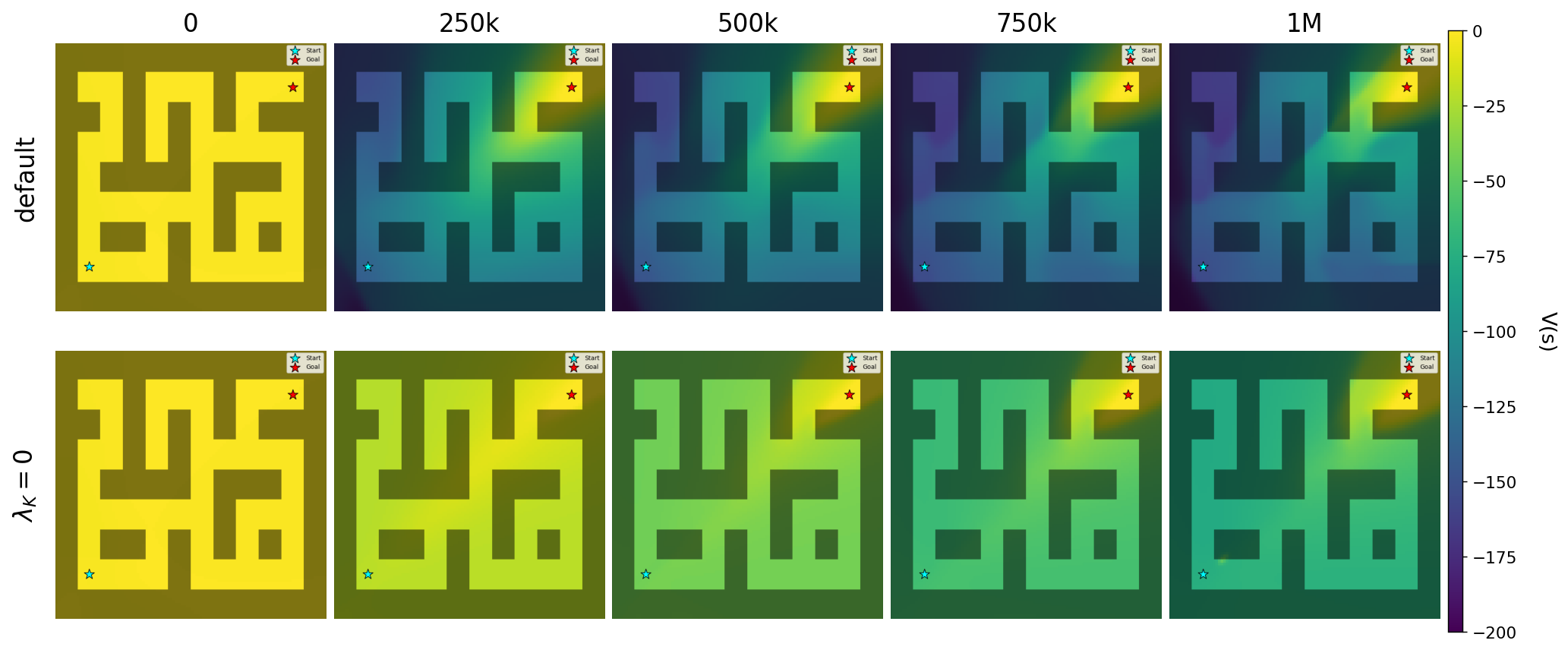}
\caption{Value maps of $V_\phi$ on \texttt{antmaze-large-navigate} at
[$250$k, $500$k, $750$k, and $1$M] training steps, for the default setting
($\lambda_{\mathrm K}=1$, top) and for $\lambda_{\mathrm K}=0$ with
$\lambda_{\mathrm B}{=}\lambda_{\mathrm E}=1000$ (bottom). Each cell shows
$V_\phi$ evaluated at the cell position with all other state coordinates set to
zero, and every map uses this same input and the same color scale
$[\,-1/(1-\gamma),0\,]$. The marker indicates the start and goal.}
\label{fig:propagate}
\end{figure}

At $\lambda_{\mathrm K} = 0$, \eqref{eq:bounded-sg} requires $\lambda_{\mathrm B} \ge (\omega_Q + \omega_V)/(1-\gamma) = 30$. Table~\ref{tab:lambda-threshold} places this threshold between $\lambda_{\mathrm B} = 25$ and $\lambda_{\mathrm B} = 250$. At $\lambda_{\mathrm B} = 25$, $V_{\min}$ lies below the range of $V^*_{\mathcal{D}}$ in Table~\ref{tab:critic-all-envs} on all three tasks, and from $\lambda_{\mathrm B} = 250$ upward it lies inside this range.

Boundedness does not restore performance. No setting with $\lambda_{\mathrm K} = 0$ exceeds $4.2\%$ on the three tasks, and \texttt{cube-single-play} reaches $66.0\%$ at $\lambda_{\mathrm B} = 1000$ against $89.8\%$ at the default (Table~\ref{tab:lambda-threshold-cube}). The range $V_{\max} - V_{\min}$ contracts as $\lambda_{\mathrm B}$ grows, from $170.1$ at $\lambda_{\mathrm B} = 250$ to $33.9$ at $\lambda_{\mathrm B} = 2500$ on \texttt{antmaze-large-navigate} and from $24.3$ to $9.1$ on \texttt{cube-single-play}, against $200.0$ and $107.9$ at the default.

Figure~\ref{fig:propagate} compares $V_\phi$ on \texttt{antmaze-large-navigate} under the default setting ($\lambda_{\mathrm K}=1$) and under $\lambda_{\mathrm K}=0$ with $\lambda_{\mathrm B}{=}\lambda_{\mathrm E}=1000$. Under the default setting, values spread from the goal across the maze from $250$k steps onward. In the bounded setting without the rollout term, $V_\phi$ remains nearly uniform across the maze through $1$M steps.

\subsection{ALBUM Critic with Generative Policies}
\label{app:album-ifql}

In this section, we examine whether the critic contribution observed in the main experiments persists when the policy is replaced by a generative policy. The main experiments use DDPG+BC to isolate the effect of the critic under a common Gaussian policy, leaving open whether it extends to generative policy classes. To assess this, we replace the IQL~\citep{kostrikov2021offline} critic in IFQL~\citep{park2025flow} with the ALBUM critic while keeping its generative policy and action extraction procedure unchanged.

IFQL trains its IQL critic independently of the policy, and IQL approximates the in-sample maximum through expectile regression, so the ALBUM critic, which also targets the in-sample maximum, replaces it directly. This comparison evaluates ALBUM and IQL under the same generative policy and therefore isolates the effect of replacing the critic.

\providecommand{\std}{}
\renewcommand{\std}[1]{\,$\pm$\scalebox{0.7}{$#1$}}
\providecommand{\gain}{}
\renewcommand{\gain}[1]{\textcolor{green!55!black}{$+#1$}}
\providecommand{\loss}{}
\renewcommand{\loss}[1]{\textcolor{red!65!black}{$-#1$}}

\begin{table}[h]
\centering
\small
\setlength{\tabcolsep}{5pt}
\begin{tabular}{llccc|c}
\toprule
Environment & Task & IFQL & IFQL-ALBUM & $\Delta$ & ALBUM \\
\midrule
antmaze-large-navigate & 1$^{*}$ & 24 & 25.5\std{14.7} & \gain{1.5}  & $\mathbf{92.5}$\std{3.1} \\
                       & 2       &  8 & 45.0\std{5.4}  & \gain{37.0} & $\mathbf{80.2}$\std{6.0} \\
                       & 3       & 52 & 51.0\std{11.4} & \loss{1.0}  & $\mathbf{91.8}$\std{2.7} \\
                       & 4       & 18 & 16.5\std{11.5} & \loss{1.5}  & $\mathbf{92.5}$\std{4.0} \\
                       & 5       & 38 & 26.5\std{21.2} & \loss{11.5} & $\mathbf{92.0}$\std{4.7} \\
\cmidrule(lr){2-6}
 & Average & $28.0$ & $32.9$ & \gain{4.9} & $\mathbf{89.8}$ \\
\midrule
scene-play & 1       & 98 & $\mathbf{98.5}$\std{1.7}  & \gain{0.5}  & 91.8\std{4.9} \\
           & 2$^{*}$ &  0 & $\mathbf{97.0}$\std{4.1}  & \gain{97.0} & 93.2\std{3.2} \\
           & 3       & 54 & $\mathbf{94.5}$\std{6.5}  & \gain{40.5} & 70.8\std{11.7} \\
           & 4       &  0 & 54.0\std{26.8} & \gain{54.0} & $\mathbf{66.2}$\std{9.5} \\
           & 5       &  0 & 0.0\std{0.0}   & $0.0$       & 0.0\std{0.0} \\
\cmidrule(lr){2-6}
 & Average & $30.4$ & $\mathbf{68.8}$ & \gain{38.4} & $64.4$ \\
\midrule
puzzle-3x3-play & 1       & 94 & $\mathbf{98.5}$\std{1.7}  & \gain{4.5}  & 96.2\std{4.1} \\
                & 2       &  1 & $\mathbf{97.0}$\std{2.2}  & \gain{96.0} & 77.5\std{32.6} \\
                & 3       &  0 & $\mathbf{63.5}$\std{19.3} & \gain{63.5} & 39.0\std{16.7} \\
                & 4$^{*}$ &  0 & $\mathbf{96.0}$\std{3.2}  & \gain{96.0} & 93.0\std{7.9} \\
                & 5       &  0 &  $\mathbf{6.5}$\std{3.0}  & \gain{6.5}  & 0.8\std{1.4} \\
\cmidrule(lr){2-6}
 & Average & $19.0$ & $\mathbf{72.3}$ & \gain{53.3} & $61.2$ \\
\bottomrule
\end{tabular}
\caption{Success rate (\%) of IFQL-ALBUM, which replaces the IQL critic of IFQL with the ALBUM critic. IFQL results are taken from \citet{park2025flow}. IFQL-ALBUM reports the mean over 4 seeds and ALBUM the mean over 8 seeds, $\pm$ the standard deviation across seeds. $\Delta$ denotes the change from IFQL in points. $^{*}$ denotes the default task.}
\label{tab:ifql-album-per-task}
\end{table}

The ALBUM critic raises the five-task average of IFQL from 30.4 to 68.8 on \texttt{scene-play} and from 19.0 to 72.3 on \texttt{puzzle-3x3-play}, and changes it from 28.0 to 32.9 on \texttt{antmaze-large-navigate}. The gains concentrate on the tasks that IFQL fails to solve. Of the six tasks at which IFQL reports 0, four reach between 54.0 and 97.0.

\newpage
\subsection{Sensitivity of Penalty Coefficients}
\label{sec:abl-lambda}

Figure~\ref{fig:LambdaEB} and Table~\ref{tab:ablation} vary
$\lambda_{\mathrm B}{=}\lambda_{\mathrm E}$ over $\{1,2.5,10,25\}$ at the
default $\lambda_{\mathrm K}$. The value $2.5$ is the only setting that
succeeds on all three tasks. At $1$ the two play tasks fall below it with
wider spreads across seeds, while \texttt{antmaze-large-navigate} is
unaffected. At $10$ and above every setting fails on the play tasks, and
\texttt{antmaze-large-navigate} follows at $25$.
Table~\ref{tab:hyperparams} adopts
$\lambda_{\mathrm B}{=}\lambda_{\mathrm E}{=}2.5$ across every domain.

\begin{figure}[h]
\centering
\includegraphics[width=\linewidth]{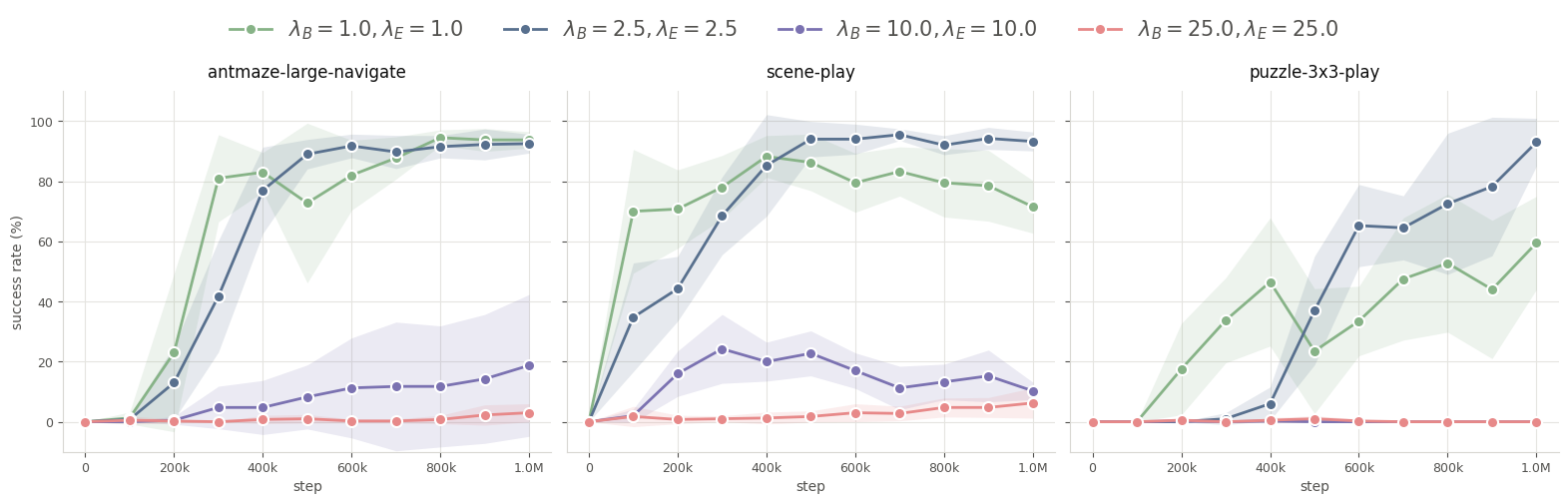}
\caption{Penalty coefficients $\lambda_{\mathrm B}{=}\lambda_{\mathrm E}$ at
the default $\lambda_{\mathrm K}$.}
\label{fig:LambdaEB}
\end{figure}
\begin{table}[H]
\centering\small\setlength{\tabcolsep}{6pt}
\begin{tabular}{l cccc}
\toprule
Task & $\lambda_{\mathrm B}{=}\lambda_{\mathrm E}{=}1$ & $2.5$ & $10$ & $25$ \\
\midrule
\texttt{antmaze-large-navigate} & $\mathbf{93.8}$\std{2.7} & $92.5$\std{3.1} & $13.0$\std{26.2} & $0.2$\std{0.7} \\
\texttt{scene-play} & $71.5$\std{8.8} & $\mathbf{93.2}$\std{3.2} & $10.2$\std{2.9} & $6.2$\std{5.0} \\
\texttt{puzzle-3x3-play} & $59.5$\std{15.5} & $\mathbf{93.0}$\std{7.9} & $0.0$\std{0.0} & $0.0$\std{0.0} \\
\bottomrule
\end{tabular}
\caption{Success rate (\%) at the final evaluation, $8$ seeds $\times$ $50$
episodes, $\pm$ std across seeds.}
\label{tab:ablation}
\end{table}

\newpage
\subsection{3M Training on Long horizon and Complex Tasks}
\label{abl:3m}
\begin{table}[h]
\centering\small\setlength{\tabcolsep}{6pt}
\begin{tabular}{l rrr rrr}
\toprule
& \multicolumn{3}{c}{\texttt{antmaze-giant-navigate}} & \multicolumn{3}{c}{\texttt{humanoidmaze-medium-navigate}} \\
\cmidrule(lr){2-4}\cmidrule(lr){5-7}
Task & $1$M & $2$M & $3$M & $1$M & $2$M & $3$M \\
\midrule
$1$ & $3.5$\std{3.8} & $20.5$\std{17.8} & $22.0$\std{19.2} & $2.0$\std{3.0} & $65.2$\std{11.1} & $82.2$\std{8.7} \\
$2$ & $76.2$\std{7.7} & $81.0$\std{5.1} & $80.5$\std{3.7} & $79.0$\std{7.3} & $90.5$\std{5.8} & $91.5$\std{4.5} \\
$3$ & $0.0$\std{0.0} & $0.0$\std{0.0} & $0.0$\std{0.0} & $4.8$\std{4.0} & $76.8$\std{10.1} & $92.0$\std{6.8} \\
$4$ & $41.2$\std{22.6} & $53.5$\std{23.0} & $47.5$\std{22.5} & $2.2$\std{2.7} & $57.0$\std{5.7} & $69.2$\std{8.7} \\
$5$ & $19.5$\std{27.4} & $30.8$\std{34.3} & $34.0$\std{36.2} & $89.5$\std{3.3} & $91.5$\std{5.6} & $93.5$\std{4.9} \\
\midrule
Average & $28.1$\std{5.1} & $37.1$\std{10.0} & $36.8$\std{11.4} & $35.5$\std{1.8} & $76.2$\std{3.7} & $85.7$\std{3.8} \\
\bottomrule
\end{tabular}
\caption{Success rate (\%) of ALBUM with an extended training budget, $8$ seeds
$\times$ $50$ evaluation episodes. Per-task rows report the spread across seeds;
the average row reports the spread of the per-seed task averages.}
\label{tab:3M}
\end{table}

Table~\ref{tab:3M} extends training from $1$M to $3$M steps on \texttt{antmaze-giant-navigate} and \texttt{humanoidmaze-medium-navigate}. On \texttt{humanoidmaze-medium-navigate} the average rises from $35.5\%$ at $1$M steps to $85.7\%$ at $3$M steps, and tasks 1, 3, and 4 rise from below $5\%$ to between $69.2\%$ and $92.0\%$, so the $1$M score on this environment reflects incomplete training. On \texttt{antmaze-giant-navigate} the average rises from $28.1\%$ to $37.1\%$ at $2$M steps without further increase, and task 3 remains at $0.0\%$. The baselines of Table~\ref{tab:main} are evaluated at $1$M steps only, so these results characterize the convergence of ALBUM and are not a comparison with the baselines.

These results also do not bear on Theorem~\ref{thm:speedup}. The theorem compares the iteration in~\eqref{eq:repair} with $K>1$ against the same iteration with $K=1$, and Appendix~\ref{ablation:K} reports the corresponding comparison within ALBUM. A comparison between ALBUM and TD-based methods is outside the scope of the theorem, since their convergence rates depend on optimization hyperparameters not included in the analysis. In particular, the critic learning rate is $10^{-4}$ for ALBUM and $3\times 10^{-4}$ for the baselines in Table~\ref{tab:main}.


\subsection{Hinge-IQL}
\label{app:hinge-iql}

\providecommand{\std}{}
\renewcommand{\std}[1]{\,$\pm$\scalebox{0.7}{$#1$}}
\begin{table}[h]
\centering
\small
\setlength{\tabcolsep}{6pt}
\begin{tabular}{ll|cc}
\toprule
Environment & Task & IQL & hinge-IQL \\
\midrule
antmaze-large-navigate & 1$^{*}$ & \textbf{48} & 12.5\std{14.9} \\
scene-play & 2$^{*}$ & \textbf{12} & \textbf{12.0}\std{7.1} \\
puzzle-3x3-play & 4$^{*}$ & \textbf{2} & 0.0\std{0.0} \\
\bottomrule
\end{tabular}
\caption{Adding ALBUM's $K$-step hinge to IQL. hinge-IQL is IQL with $\lambda_{\mathrm K}[\mathrm{sg}[\hat y_K(s,a)] - Q(s,a)]_+$ added to the critic loss and $\lambda_{\mathrm K}[\mathrm{sg}[\hat y_K(s,a)] - V(s)]_+$ added to the value loss at $\lambda_{\mathrm K} = 1$, where $\hat y_K$ is computed with IQL's value network and detached as in ALBUM, keeping IQL's AWR actor at $\alpha = 10$. Ours is the mean over 4 seeds at the final evaluation, $\pm$ the standard deviation across seeds, and the IQL column is taken from the FQL paper. The best entry in each row is bold. $^{*}$ indicates the default task.}
\label{tab:hinge-iql}
\end{table}

In this section we isolate the contribution of the $K$-step rollout hinge by
asking whether it improves a critic that is not derived from LBLP. We add the
two penalties to the critic and value losses of IQL, keeping the
hyperparameters reported in~\citet{park2025flow}, and run the default tasks of
\texttt{antmaze-large-navigate}, \texttt{scene-play} and
\texttt{puzzle-3x3-play}. As Table~\ref{tab:hinge-iql} shows, the penalty does
not improve IQL on any of the three tasks, and it degrades performance
substantially on \texttt{antmaze-large-navigate}. The benefit of the $K$-step
penalties therefore does not transfer on its own. The critic in IQL is
regressed onto a bootstrapped target, so a lower bound added to that loss has
no descent term to work against, whereas in ALBUM the same bound enters a
program whose objective descends on $Q$ and $V$.

\newpage

\subsection{Neural Critic versus In-Sample Optimum}
\label{app:vstar}

\begin{table}[h]
\centering\tiny\setlength{\tabcolsep}{4.5pt}
\resizebox{\textwidth}{!}{%
\begin{tabular}{l l rrr rr r r}
\toprule
Environment & & $V_{\max}$ & $V_{\mathrm{mean}}$ & $V_{\min}$ &
$V_{\max}{-}V_{\mathrm{mean}}$ & $V_{\max}{-}V_{\min}$ &
$g_{\mathrm B}\le0$ (\%) & Success (\%) \\
\midrule
\multirow{2}{*}{\texttt{antmaze-large-navigate}}
 & $V^{*}_{\mathcal D}$ & $0.0$ & $-156.8$ & $-198.7$ & $156.8$ & $198.7$ & --- & $11.8$ \\
 & learned              & $0.6$ & $-136.2$ & $-199.4$ & $136.8$ & $200.0$ & $65$ & $92.5$ \\
\midrule
\multirow{2}{*}{\texttt{antmaze-giant-navigate}}
 & $V^{*}_{\mathcal D}$ & $0.0$   & $-178.1$ & $-200.0$ & $178.1$ & $200.0$ & --- & $10.6$ \\
 & learned              & $-86.0$ & $-172.4$ & $-199.3$ & $86.4$  & $113.4$ & $81$ & $3.5$ \\
\midrule
\multirow{2}{*}{\texttt{humanoidmaze-medium-navigate}}
 & $V^{*}_{\mathcal D}$ & $0.0$  & $-177.1$ & $-200.0$ & $177.1$ & $200.0$ & --- & $15.1$ \\
 & learned              & $-0.3$ & $-116.5$ & $-157.1$ & $116.2$ & $156.8$ & $67$ & $2.0$ \\
\midrule
\multirow{2}{*}{\texttt{humanoidmaze-large-navigate}}
 & $V^{*}_{\mathcal D}$ & $0.0$    & $-178.6$ & $-200.0$ & $178.6$ & $200.0$ & --- & $7.7$ \\
 & learned              & $-199.5$ & $-199.5$ & $-199.5$ & $0.0$   & $0.0$   & $50$ & $0.0$ \\
\midrule
\multirow{2}{*}{\texttt{antsoccer-arena-navigate}}
 & $V^{*}_{\mathcal D}$ & $0.0$  & $-156.7$ & $-198.7$ & $156.7$ & $198.7$ & --- & $10.9$ \\
 & learned              & $-0.9$ & $-103.6$ & $-198.5$ & $102.7$ & $197.6$ & $74$ & $66.0$ \\
\midrule
\multirow{2}{*}{\texttt{cube-single-play}}
 & $V^{*}_{\mathcal D}$ & $0.0$ & $-147.7$ & $-198.7$ & $147.7$ & $198.7$ & --- & $35.6$ \\
 & learned              & $2.3$ & $-72.0$  & $-117.5$ & $74.3$  & $119.8$ & $72$ & $84.5$ \\
\midrule
\multirow{2}{*}{\texttt{cube-double-play}}
 & $V^{*}_{\mathcal D}$ & $0.0$   & $-311.8$ & $-397.3$ & $311.8$ & $397.3$ & --- & $1.5$ \\
 & learned              & $-85.3$ & $-309.3$ & $-377.5$ & $224.0$ & $292.2$ & $54$ & $5.2$ \\
\midrule
\multirow{2}{*}{\texttt{scene-play}}
 & $V^{*}_{\mathcal D}$ & $0.0$   & $-500.0$ & $-972.3$ & $500.0$ & $972.3$ & --- & $4.0$ \\
 & learned              & $-14.7$ & $-285.4$ & $-477.5$ & $270.6$ & $462.7$ & $64$ & $93.2$ \\
\midrule
\multirow{2}{*}{\texttt{puzzle-3x3-play}}
 & $V^{*}_{\mathcal D}$ & $0.0$   & $-714.0$ & $-1309.0$ & $714.0$ & $1309.0$ & --- & $5.4$ \\
 & learned              & $-40.2$ & $-368.6$ & $-520.5$  & $328.4$ & $480.3$  & $68$ & $93.0$ \\
\midrule
\multirow{2}{*}{\texttt{puzzle-4x4-play}}
 & $V^{*}_{\mathcal D}$ & $-3.0$   & $-1282.1$ & $-2241.7$ & $1279.1$ & $2238.7$ & --- & $0.0$ \\
 & learned              & $-696.1$ & $-1031.7$ & $-1234.3$ & $335.5$  & $538.2$  & $71$ & $11.8$ \\
\bottomrule
\end{tabular}
}
\caption{The learned value function against $V^*_\mathcal{D}$ on one task per environment, which is the default task except for task 1 of \texttt{antsoccer-arena-navigate} and \texttt{cube-single-play}, with the default configuration ($\lambda_{\mathrm B}=\lambda_{\mathrm E}=2.5$, $\lambda_{\mathrm K}=1$, $K=10$), 1M steps, and 8 seeds. For $V^*_\mathcal{D}$ the Success column reports the fraction of dataset trajectories that reach the goal. Reward scales differ across environments, so values compare within a row pair but not across them.}
\label{tab:critic-all-envs}
\end{table}

Under deterministic dynamics and a continuous state space, each state of $\mathcal{D}$ carries a single action, so $V^*_{\mathcal{D}}(s) = Q^*_{\mathcal{D}}(s,a) = G_{\mathcal{D}}(s,a)$, the discounted return along the dataset trajectory through $s$. Table~\ref{tab:critic-all-envs} compares the statistics of $V^*_{\mathcal{D}}$ over $\mathcal{D}$ with those of the learned $V_\phi$. The mean of $V_\phi$ exceeds the mean of $V^*_{\mathcal{D}}$ in nine of ten environments. The exception is \texttt{humanoidmaze-large-navigate}, with $V_\phi$ constant at $-199.5$, close to $-1/(1-\gamma) = -200$. The detached $K$-step hinges in \eqref{eq:album-loss} bound $V_\phi(s)$ from below, and $V_\phi$ shares parameters across states of different trajectories, so a value above the return of the trajectory through $s$ is consistent with lower bounds transferred from neighboring trajectories with higher returns. Rewards are nonpositive, so this transfer does not account for the entries $V_{\max} = 0.6$ on \texttt{antmaze-large-navigate} and $V_{\max} = 2.3$ on \texttt{cube-single-play}, which are attributable to function approximation error. The constraint $g_{\mathrm{B}} \le 0$ holds on $50\%$ to $81\%$ of the samples, so the learned pair is not feasible for LBLP on $\mathcal{M}_{\mathcal{D}}$.